\pdfoutput=1
\documentclass[11pt]{article}

\usepackage[a4paper,margin=22mm]{geometry}
\usepackage[T1]{fontenc}
\usepackage{amsmath,amssymb,amsthm}
\usepackage{newtxtext}
\usepackage{newtxmath}
\usepackage{parskip}

\usepackage{natbib}
\setcitestyle{authoryear,round,citesep={;},aysep={,},yysep={;}}

\usepackage{amsmath,amssymb,amsthm}
\usepackage{aliascnt}
\usepackage{array}
\usepackage{booktabs}
\usepackage{longtable}
\usepackage{tabularx}
\usepackage{enumitem}
\usepackage{flafter}
\usepackage{placeins}
\usepackage[final]{microtype}
\usepackage{graphicx}
\usepackage{capt-of}
\usepackage{wrapfig}
\usepackage{pgfplots}
\usepackage{pgfplotstable}
\usepackage{hyperref}
\usepackage{url}
\usepackage{xurl}
\hypersetup{
    colorlinks=true,
    linkcolor=blue,
    filecolor=blue,
    urlcolor=blue,
    citecolor=blue,
    anchorcolor=blue
}
\usepackage[nameinlink,capitalize,noabbrev]{cleveref}

\pgfplotsset{compat=1.18}
\usepgfplotslibrary{groupplots}
\usetikzlibrary{arrows.meta,positioning,calc}
\graphicspath{{figures/}}

\newlength{\paperpairsep}
\pgfplotsset{
  paper axis/.style={
    scale only axis,
    tick align=outside,
    tick pos=left,
    axis line style={draw=black!75, line width=0.6pt},
    tick style={draw=black!75, line width=0.6pt},
    tick label style={font=\footnotesize},
    label style={font=\footnotesize},
    title style={font=\small, yshift=1.5pt},
    grid=major,
    grid style={draw=black!12, line width=0.4pt},
    every axis plot/.append style={line width=1.1pt, mark size=2pt,
      mark options={solid, line width=0.7pt}},
    error bars/error bar style={line width=0.6pt, solid},
    error bars/error mark options={rotate=90, mark size=1.5pt,
      line width=0.6pt},
    legend style={font=\footnotesize, draw=none, fill=none,
      cells={anchor=west},
      /tikz/every even column/.append style={column sep=0.85em}},
    legend image code/.code={%
      \draw[mark repeat=2, mark phase=2, ##1]
        plot coordinates {(0cm,0cm) (0.26cm,0cm) (0.52cm,0cm)};},
  },
  paper pair/.style={
    paper axis,
    width=0.368\linewidth,
    height=0.27\linewidth,
  },
  paper wide/.style={
    paper axis,
    width=0.70\linewidth,
    height=0.30\linewidth,
  },
}

\newsavebox{\domainpanelbox}

\newcommand{\plotlegend}[1]{%
  \par\noindent\makebox[\linewidth][c]{\ref*{#1}}}

\definecolor{cbblue}{HTML}{0072B2}
\definecolor{cborange}{HTML}{D55E00}
\definecolor{cbgreen}{HTML}{009E73}
\definecolor{cbpurple}{HTML}{CC79A7}

\newcolumntype{R}{>{\raggedleft\arraybackslash}X}

\theoremstyle{plain}
\newtheorem{theorem}{Theorem}[section]
\newaliascnt{proposition}{theorem}
\newtheorem{proposition}[proposition]{Proposition}
\aliascntresetthe{proposition}
\newaliascnt{corollary}{theorem}
\newtheorem{corollary}[corollary]{Corollary}
\aliascntresetthe{corollary}
\newaliascnt{lemma}{theorem}
\newtheorem{lemma}[lemma]{Lemma}
\aliascntresetthe{lemma}
\theoremstyle{definition}
\newaliascnt{definition}{theorem}
\newtheorem{definition}[definition]{Definition}
\aliascntresetthe{definition}
\theoremstyle{plain}

\newcommand{\Prob}{\mathbb P}
\newcommand{\E}{\mathbb E}
\DeclareMathOperator{\Var}{Var}
\DeclareMathOperator{\Cov}{Cov}

\DeclareMathOperator{\Bern}{Bern}

\newcommand{\Eval}{\mathsf E}
\newcommand{\one}{\mathbf 1}

\DeclareMathOperator{\TV}{TV}

\newcommand{\dd}{\mathop{}\mathrm{d}}
\newcommand{\bigO}{\mathcal O}

\newcommand{\DL}{D_{\mathrm L}}
\newcommand{\DB}{D_{\mathrm B}}

\newcommand{\pM}{p_{\mathrm M}}
\newcommand{\pMmax}{p_{\mathrm M,\max}}
\newcommand{\Nround}{N_{\mathrm{round}}}

\newcommand{\Ncand}{N_{\mathrm{cand}}}

\newcommand{\KLlog}{\mathcal K^{\mathrm L}}

\DeclareMathOperator{\sigm}{sigm}
\newcommand{\tsep}{t_{\mathrm{sep}}}
\newcommand{\tsepM}{t_{\mathrm{sep,M}}}
\newcommand{\tsepB}{t_{\mathrm{sep,B}}}

\newcommand{\SATcount}{\#\mathrm{SAT}}
\DeclareMathOperator{\KL}{KL}
\DeclareMathOperator{\logit}{logit}
\newcommand{\binaryKL}{d_{\mathrm B}}

\crefname{theorem}{Theorem}{Theorems}
\crefname{proposition}{Proposition}{Propositions}
\crefname{corollary}{Corollary}{Corollaries}
\crefname{lemma}{Lemma}{Lemmas}
\crefname{definition}{Definition}{Definitions}
\crefname{section}{Section}{Sections}

\input{figures/tables}

\hypersetup{
  pdftitle={The cost of useful natural gradient updates},
  pdfauthor={Subhransu S. Bhattacharjee, Dylan Campbell, Rahul Shome}
}

\title{The cost of useful natural gradient updates}
\author{%
  Subhransu S. Bhattacharjee \qquad Dylan Campbell \qquad Rahul Shome\\[1ex]
  School of Computing, The Australian National University\\
  \normalsize Corresponding author: \texttt{Subhransu.Bhattacharjee@anu.edu.au}%
}
\date{}

\begin{document}
\maketitle
\begin{abstract}
What information is needed to turn a natural-gradient direction into a useful
finite update? Under a population Kullback--Leibler (KL) budget, we call a step
\emph{useful} if it is feasible and loses at most a fraction $\varepsilon$ of
the best feasible gain along the direction. We construct a four-state
exponential family whose laws share their initial gradient, scalar Fisher
information and natural gradient, yet two laws have disjoint useful-step sets.
With these quantities supplied exactly and the law otherwise known only
through draws, the family's worst-case sample complexity is $\Theta(\log(1/\delta)/(p\varepsilon^2))$ for small
$\varepsilon$, where $p$ scales rare-state probabilities and $\delta$ is the
failure probability. The budget is fixed and the optimal gain stays bounded
away from zero, so the step length, not the direction, carries this cost. For succinctly described event-tilt models, returning a useful
step is NP-hard even with the exact natural gradient and efficient exact
sampling. Recovering the unit natural gradient to constant error is also
NP-hard even in a two-parameter logistic family with Fisher condition number at
most $3$. We also give matching sample bounds for event tilts, sample bounds for damped
Fisher solves and a population-KL certificate for affine classifiers. In frozen-feature classifier heads, stopping at a sampled KL boundary
succeeds in about half of the trials, and a 10\% KL margin raises joint
success above 93\% at a KL budget of $0.01$.
Thus, knowing where to move is not enough: how
far to move can carry an update's entire cost.
\end{abstract}

\section{Introduction}
\label{sec:introduction}

An exact natural-gradient direction need not determine a useful finite update,
which also needs a step length. Natural-gradient and trust-region methods estimate the update direction with
care and usually take its length from the local quadratic model
\citep{Kakade2001,SchulmanTrustRegion2015,MartensGrosse2015}. 

\begin{figure}[!hb]
\centering
\input{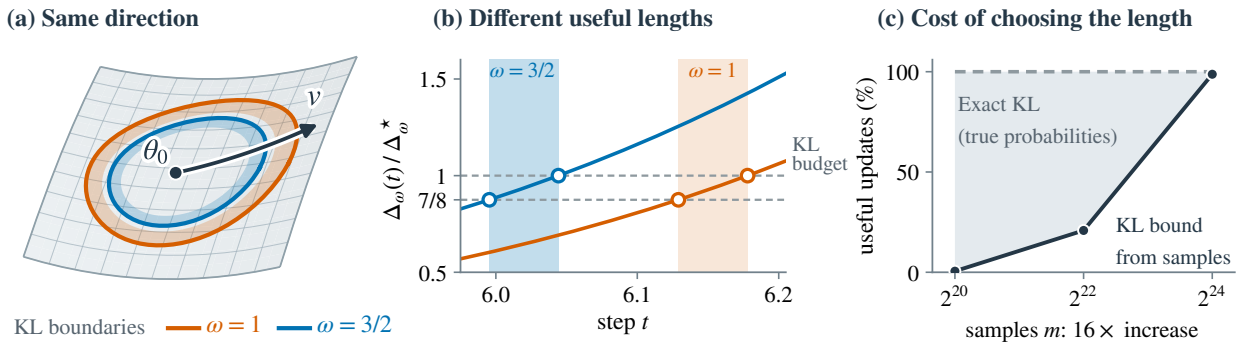}
\vspace{-1em}
\caption{\textbf{Direction, calibration and sampling cost.}
\textbf{(a)} Schematic: in the four-state laws $P_\omega$ of
\cref{thm:matched-moment-main}, $\omega$ scales the rare top state's
probability. At $\omega=1$ and $\omega=3/2$, they share the gradient,
Fisher matrix and natural gradient $v$ at the initial parameters $\theta_0$, but their KL boundaries
(curves) and useful steps (shaded rings) differ.
\textbf{(b)} Exact gain $\Delta_\omega(t)$ of the step $t$ along $v$ over the best feasible
gain $\Delta_\omega^\star$, reached at the KL budget. Steps are useful where this ratio lies in
$[7/8,1]$ (shaded), and the two intervals are disjoint.
\textbf{(c)} Logistic controls with directions estimated from $m$ samples.
A useful step respects the budget and attains at least $90\%$ of the
globally optimal KL-feasible query gain. Exact KL calibration is useful in every trial.
A $95\%$ confidence rule reusing the samples is useful in $0.5\%$ to
$98.7\%$ of trials.}
\label{fig:construction-geometry}
\end{figure}

Under a
population Kullback--Leibler (KL) budget, the step length can be far off even
along the exact direction. In our four-state family, the quadratic step is
over $450$ times longer than any feasible step, and its KL is about
$8.6\times10^3$ nats against a budget of $1/10$
(\cref{sec:matched-local-calibration}). 
Fisher approximation and inversion
address direction recovery \citep{MartensGrosse2015,Martens2020}, and
step-selection analyses assume step-size or accuracy conditions
(\cref{sec:bg}). We instead ask what information a useful length requires
beyond an exact direction, and what acquiring it costs, a question that remains under-explored. 
We prove that knowing an 
\textit{exact direction}, even with \textit{efficient exact sampling}, does not make a useful
finite step cheap to find. 
Here the cost comes from rare states rather than
from dimension or conditioning, so better Fisher approximations cannot remove
it, and the sample lower bound holds for \textit{every procedure}, including line searches
and backtracking.

At initial parameters $\theta_0\in\mathbb R^d$, the natural gradient $v=F^{-1}b$ combines the
nonzero objective gradient $b$ and positive definite Fisher matrix $F$
\citep{Amari1998,Martens2020}. For a KL budget $\eta>0$, the quadratic trust-region model maximizes
$b^{\mathsf T}\xi$ over parameter displacements $\xi$ subject to
$\dfrac12\xi^{\mathsf T}F\xi\leq\eta$. Its optimizer is
$\xi=\sqrt{\dfrac{2\eta}{b^{\mathsf T}F^{-1}b}}\,v$.
Let $\mathcal K(u)$ and $\Delta(u)$ denote the actual population KL cost and
gain of parameters $u$.
Over step lengths $t\geq0$ along $v$, the best feasible gain is
\begin{equation}
 \Delta^\star:=\sup_{t\geq0,\ \mathcal K(\theta_0+tv)\leq\eta}
 \Delta(\theta_0+tv).
 \label{eq:calibration-problem}
\end{equation}
For $0<\varepsilon<1$, a step is \emph{useful} if it stays within budget and
attains at least $(1-\varepsilon)\Delta^\star$.
\emph{Calibration} selects such a length. A zero step fails the gain requirement
whenever $\Delta^\star>0$.

\noindent\textit{Local information does not determine the length.}
Our four-state family is a softmax distribution over statistic values
$0,1,2,3$, with their mean as objective, and a natural-gradient step shifts
the logits along these values. The outcomes can be the actions of a policy at one state, where the best
action is rarely chosen and the probabilities are known only through draws. At a fixed
scale $\pM$ of the rare-state probabilities, all laws in the family share the initial gradient,
scalar Fisher information and natural gradient, yet some pairs of laws have disjoint intervals of useful
steps (\cref{fig:construction-geometry}(b)). Four fixed statistic values are the
fewest that allow distinct laws with equal first two moments
(\cref{sec:matched-local-calibration}).

\noindent\textit{Calibration costs as much as estimating a rare probability.}
With exact local information, a two-point lower bound and a matching
confidence-calibration upper bound show that the sample complexity
$m^\star(\pM,\varepsilon,\delta)$, the fewest draws that yield a useful step
with probability at least $1-\delta$, is
$\Theta(\log(1/\delta)/(\pM\varepsilon^2))$, as for estimating a rare
probability \citep{DagumEtAl2000}. The rare best outcome's
tilted weight grows exponentially, so it can control finite-step KL while
contributing little to the initial Fisher information. With the budget held fixed,
the optimal gain stays above a positive constant as $\pM$ decreases, so this
cost does not come from a vanishing target.

\noindent\textbf{Contributions.}
(i) Matching upper and lower bounds characterize calibration in the
four-state family with exact local information (\cref{thm:matched-moment-main}).
(ii) For event tilts, an explicit confidence rule attains matching sample
bounds, and selecting a useful tilt parameter from a succinct description is
NP-hard, although the exact natural gradient is supplied and exact sampling is
efficient (\cref{thm:accuracy-calibration,cor:fixed-budget-hardness}).
(iii) Recovering the unit natural gradient of a two-parameter logistic model
is NP-hard from its description, although its Fisher matrix has condition
number at most $3$, and needs exponentially many draws when only samples are
available (\cref{thm:logistic-preconditioning}).
(iv) For bounded statistics, our sample bounds for estimating the damped
natural gradient match up to a logarithmic factor in the dimension, at small
damping and fixed confidence (\cref{prop:damped-sampling-main}).
(v) For affine classifiers with known feature bounds, independent calibration
draws certify the KL feasibility of a fixed set of candidate steps, even
when the chosen step depends on those draws
(\cref{sec:sampling,prop:affine-calibration-main}). (vi) Experiments match these predictions (\cref{sec:experiments}). With $m$
draws, count-simulation success at fixed $mp\varepsilon^2$ is similar across
event probabilities $p$, exact-KL calibration meets the logistic gain target
despite inaccurate directions, and a $10\%$ KL margin lifts joint success in
classifier heads from about $50\%$ to over $93\%$.

\section{Background \& Related Work}
\label{sec:bg}
\noindent\textbf{Preliminaries.}
A \emph{law} is a probability distribution. On a finite state space
$\mathcal S$, the \emph{base law} $P$ describes the model before updating, and
a \emph{draw} is an independent exact sample from it unless stated otherwise.
For statistics $h:\mathcal S\to\mathbb R^d$, exponential tilting defines
$P_\theta(s)=e^{\langle\theta,h(s)\rangle}P(s)/Z(\theta)$, where
$Z(\theta)=\E_P e^{\langle\theta,h\rangle}$ \citep{WainwrightJordan2008}.
The \emph{reference law} for the KL budget is $P$, so an update costs
$\KL(P\Vert P_\theta)=\E_P\log(P/P_\theta)$ nats.
Substituting the tilted density gives
$\KL(P\Vert P_\theta)=\log Z(\theta)-\langle\theta,\E_P h\rangle$.
The Fisher information matrix $F(\theta)=(g_{ij}(\theta))$, with
$g_{ij}(\theta)=\Cov_{P_\theta}(h_i,h_j)$, determines the quadratic term of KL
between nearby parameters \citep{Amari1982}. For $d=1$, it is the scalar Fisher
information. We write $\one\{E\}$ for the indicator of a condition $E$
and $\one_A(s)=\one\{s\in A\}$. For conditional models, population KL averages predictive KL over a reference
input law. The population Fisher matrix averages outer products
of the score over inputs and model-generated labels, whereas the empirical
Fisher matrix uses observed labels \citep{KunstnerEtAl2019}.

\noindent\textbf{A primer on complexity theory.}
Our results bound two resources, \emph{draws} and \emph{computation}, with one
tool for each. A \emph{procedure} is a randomized algorithm that uses an
independent seed and at most a fixed number of draws, and returns a measurable
output defined for every seed and draw sequence (\cref{app:procedures}). Every
step-selection rule in use is a procedure, from quadratic-model rescaling to
TRPO's backtracking line search. The \emph{sample complexity} of a problem is the
smallest number of draws with which some procedure attains the required
success probability under every law considered. \emph{Sample lower bounds}
allow unlimited computation and use \emph{two-point testing}
\citep{Tsybakov2009}. If two laws admit no common valid output, any procedure
that succeeds under both must tell them apart from its draws, which takes more
draws the closer the laws are in per-draw KL (\cref{lem:padding}).
\emph{Computational lower bounds} start from a \emph{succinct description}, here
a Boolean formula $\varphi$ with encoding $\langle\varphi\rangle$ that
specifies the law. \emph{Boolean satisfiability} (SAT) asks whether $\varphi$ is
satisfiable, and $\SATcount(\varphi)$ counts its satisfying assignments
\citep{Valiant1979}. A problem is \emph{NP-hard} if a polynomial-time algorithm
for it would decide SAT, and hence every problem in \emph{NP}
(nondeterministic polynomial time). A \emph{language} is a set of binary
strings, and a language in NP is \emph{NP-complete} if every language in NP
reduces to it by a polynomial-time \emph{many-one reduction}, which maps each
input to one with the same answer. \emph{BPP} (bounded-error probabilistic
polynomial time) contains problems that randomized polynomial-time
algorithms solve with error at most $1/3$ \citep{AroraBarak2009}. Each
reduction decides SAT from one valid output, so a randomized
polynomial-time procedure succeeding with probability at least $2/3$ on every
input would imply $\mathrm{NP}\subseteq\mathrm{BPP}$, which is believed false.

\noindent\textbf{Inference and sampling access.}
Evaluating the expected Fisher matrix involves inference in Bayesian networks
\citep{Kontkanen2000}, where exact and relative-error inference are hard
\citep{Cooper1990,DagumLuby1993}. Compiling a network into an arithmetic
circuit makes inference linear in the circuit's size, but this size can be
exponential in the size of the network's description \citep{Darwiche2003}. Our reductions allow efficient exact sampling and
target the returned direction or step itself.

\noindent\textbf{Step selection.}
Natural-policy-gradient analyses establish convergence with inexact policy
evaluation \citep{AgarwalEtAl2021,CenEtAl2022}.
Conservative policy iteration \citep{KakadeLangford2002} and adaptive
policy-gradient step sizes \citep{PirottaEtAl2013} maximize lower bounds on
improvement. Probabilistic \citep{MahsereciHennig2015} and stochastic
\citep{VaswaniEtAl2019} line searches use noisy function values and gradients.
Stochastic trust-region \citep{ChenMenickellyScheinberg2018,BlanchetEtAl2019}
and line-search analyses \citep{PaquetteScheinberg2020} impose probabilistic
accuracy conditions. \citet{JinEtAl2025} quantify oracle and sample costs for
such methods.

\noindent\textbf{Rare events and certification.}
Estimating a rare Bernoulli mass $p$ to relative error $\varepsilon$ at
failure probability $\delta$ has classical sample order
$\log(1/\delta)/(p\varepsilon^2)$ \citep{DagumEtAl2000}.
Our construction transfers this estimation cost to calibration. Our matching upper bounds control relative error using multiplicative concentration \citep{Hoeffding1963}
and Bernoulli-KL confidence intervals \citep{GarivierCappe2011}.
\citet{BuEtAl2018} study the related problem of estimating KL between unknown
discrete laws.
The equivalence of approximate counting and almost-uniform generation for
self-reducible problems \citep{JerrumEtAl1986} concerns sampling solutions. Our reductions instead sample uniformly from all
assignments, among which satisfying ones can be exponentially rare.
Empirical Bernstein bounds \citep{MaurerPontil2009} certify population KL
using its observed per-input variance and known range. High-confidence policy
improvement \citep{ThomasEtAl2015} likewise uses finite-sample confidence
bounds, but to certify a new policy's value rather than an update's KL.

\section{Complexity of Useful Updates}
\label{sec:core-result}

After fixing what makes an update useful, we ask whether exact local
information (\cref{sec:matched-local-calibration}), efficient exact sampling
(\cref{sec:fixed-budget-calibration}) or a well-conditioned Fisher matrix
(\cref{sec:logistic}) makes a useful step or direction cheap to find. The
answers are negative in the worst case, through sample lower bounds that hold
for every procedure, NP-hardness or both. \Cref{sec:sampling} then shows what
draws can still guarantee.

\begin{definition}[Useful update and calibration]
\label{def:useful}
Fix initial parameters $\theta_0\in\mathbb R^d$, an objective $J$, a reference
population and a declared update family $\mathcal U\subseteq\mathbb R^d$.
An example is the ray $\{\theta_0+tv:t\geq0\}$ for a fixed direction
$v\in\mathbb R^d$. For updated parameters $u\in\mathcal U$, the \emph{KL cost}
$\mathcal K(u)$ is the population KL divergence from the initial model to the
updated one, and the \emph{gain} is $\Delta(u)=J(u)-J(\theta_0)$. Given a budget
$\eta>0$, let
$\Delta^\star=\sup\{\Delta(u):u\in\mathcal U,\ \mathcal K(u)\leq\eta\}$. For
$0<\varepsilon<1$, an update $u$ is \emph{useful} if $\mathcal K(u)\leq\eta$ and
$\Delta(u)\geq(1-\varepsilon)\Delta^\star$. \emph{Calibration} is the problem of selecting a useful update from the available information. On the ray above, it seeks a feasible
$(1-\varepsilon)$-approximation to \cref{eq:calibration-problem}. \emph{Local information} consists of the
objective gradient, Fisher information matrix and natural gradient at $\theta_0$. An
\emph{access model} specifies which of these are supplied exactly, and whether draws or a succinct description of the base law are available.
\end{definition}

\subsection{Calibration complexity with exact gradient and scalar Fisher information}
\label{sec:matched-local-calibration}
We first ask whether exact local information determines a useful length. To
test this, we look for laws that agree on every quantity in the local linear
objective and quadratic KL models \citep{Martens2020}. When
the objective is the mean of a statistic $h$, its gradient and scalar Fisher
information both equal the variance of $h$, so such laws share the first two
moments of $h$. The family below meets this constraint with four states and
one free parameter $\omega$, which moves mass to a rare top state. The paragraph
after the theorem explains each constant.

\begin{theorem}[Calibration with exact gradient and scalar Fisher information]
\label{thm:matched-moment-main}
Fix $0<\pM\leq2^{-30}$, $0<\varepsilon\leq1/8$ and $0<\delta\leq1/3$.
For an unknown $\omega\in[1,3/2]$, let $h\in\{0,1,2,3\}$ have law
\begin{equation}
 P_\omega=\bigl(1-(5+\omega)\pM,\ (3\omega-2)\pM,\ (7-3\omega)\pM,\ \omega\pM\bigr).
 \label{eq:matched-moment-main}
\end{equation}
For $t\geq0$, define $Z_\omega(t)=\E_{P_\omega}e^{th}$,
$P_{\omega,t}(h)=e^{th}P_\omega(h)/Z_\omega(t)$,
$J_\omega(t)=\E_{P_{\omega,t}}h$ and scalar Fisher information
$F_\omega(t)=\Var_{P_{\omega,t}}h$. The gain and KL cost are
$\Delta_\omega(t)=J_\omega(t)-12\pM$ and
${\mathcal K}_\omega(t)=\KL(P_\omega\Vert P_{\omega,t})
=\log Z_\omega(t)-12\pM t$, and the best feasible gain is $\Delta_\omega^\star=\max_{t\geq0:\,{\mathcal K}_\omega(t)\leq1/10}\Delta_\omega(t)$.
Every law has $\E h=12\pM$ and $\E h^2=26\pM$. Its initial gradient and Fisher information satisfy
$J_\omega'(0)=F_\omega(0)=26\pM-144\pM^2>0$, and the natural gradient
$F_\omega(t)^{-1}J_\omega'(t)$ is $1$ for all $t\geq0$.
A procedure receives $\pM$, this family, $h$, the exact initial gradient and
scalar Fisher information, and the exact natural gradient. Its only further information about $\omega$ is
draws from $P_\omega$. It must return a scalar
$\widehat t\geq0$ satisfying
${\mathcal K}_\omega(\widehat t)\leq1/10$ and
$\Delta_\omega(\widehat t)\geq(1-\varepsilon)\Delta_\omega^\star$.
Requiring success with probability at least $1-\delta$ for every $\omega\in[1,3/2]$
gives sample complexity $\Theta(\log(1/\delta)/(\pM\varepsilon^2))$.
\end{theorem}

\begin{proof}[Proof sketch]
The candidate $t=\tfrac13\log(1/(10\omega\pM))$ gives the rare top state
unnormalized weight $1/10$. Bounding the other tilted weights proves
feasibility and $\Delta_\omega^\star>27/100$, uniformly over the family.
Gain and KL increase for $t>0$. Useful steps form
$[L_{\omega,\varepsilon},T_\omega]$, where $T_\omega$ exhausts the KL budget
and $L_{\omega,\varepsilon}$ attains the required gain fraction.
The identity
$Z_{\omega'}(t)-Z_\omega(t)=(\omega'-\omega)\pM(e^t-1)^3$
shows that the normalizers agree through second order at zero, but KL
increases with $\omega$ for $t>0$. Hence $T_\omega$ decreases.
On each useful interval, the bounds $J_\omega'(t)\geq3/5$ and
$\Delta_\omega^\star\leq7/20$ give
$T_\omega-L_{\omega,\varepsilon}\leq7\varepsilon/12$ by integration.
Changing $\omega$ from $1$ to $1+4\varepsilon$ moves the KL boundary by at
least $4\varepsilon/5$, exceeding this width. Hence $T_{1+4\varepsilon}<L_{1,\varepsilon}$, so the
intervals are disjoint.
Thresholding a useful output therefore distinguishes
two laws with identical local information and one-draw KL
$O(\pM\varepsilon^2)$, requiring
$\Omega(\log(1/\delta)/(\pM\varepsilon^2))$ draws.
For the upper bound, the frequency of $h=3$ gives an upper confidence
estimate of $\omega$. On the coverage event, its KL boundary is no larger
than the true one.
At the matching sample size, concentration controls coverage and estimation
error, while boundary sensitivity bounds the lost gain by
$\varepsilon\Delta_\omega^\star$.
\Cref{app:matched-moments} gives the proofs.
\end{proof}

\noindent\textbf{Why four states and these constants?}
On at most three fixed, distinct statistic values,
normalization and the first two moments determine the law, since the
corresponding Vandermonde system has full column rank. On $\{0,1,2,3\}$, they
leave one free direction, $\pM(-1,3,-3,1)$, the third finite difference, which
is orthogonal to $1$, $h$ and $h^2$. Varying $\omega$ moves $P_\omega$ along it,
so every law has the same first two moments and local quantities, while the
rare top state's mass $\omega\pM$ grows with $\omega$. Its tilted weight
dominates finite-step KL, so $\omega$ shifts the KL boundary and the useful
steps. The range
$\omega\in[1,3/2]$ keeps every rare mass between $\pM$ and $4\pM$, and
$\varepsilon\leq1/8$ keeps the hard pair $\omega=1$ and $\omega=1+4\varepsilon$
inside it. The scale $\pM\leq2^{-30}$ keeps the middle states' tilted weights
negligible near the KL boundary, and the fixed budget $1/10$ keeps the target
from shrinking with $\pM$. Both values only make the constants explicit, and
$\delta\leq1/3$ is the usual regime for testing bounds.
Beyond the construction, the quadratic KL model
\citep{Amari1998,Martens2020} behind trust-region steps
\citep{SchulmanTrustRegion2015} uses only the statistic's covariance, whereas
finite-step KL depends on its whole distribution. Any law
with positive mass on at least four values of a scalar statistic has nearby
laws with the same first two moments but different finite-step KL.

\noindent\textit{Consequence for local step rules.}
A rule using only the supplied local information cannot depend on $\omega$.
Each returned step is useful under at most one law of each hard pair. At
$\pM=2^{-30}$, the quadratic-model length $\sqrt{\dfrac{2\eta}{b^{\mathsf T}F^{-1}b}}$
of \cref{sec:introduction} is about $2.9\times10^3$ for every $\omega$, with
finite KL of about $8.6\times10^3$ nats. The KL boundaries $T_\omega$ instead
lie between $6.04$ and $6.18$. This step misses them because the rare
state enters the initial curvature $F_\omega(0)<26\pM$ only through its
probability, whereas its tilted weight $\omega\pM e^{3t}$ grows exponentially until KL
rises sharply near the budget. \Cref{app:quadratic-finite-kl}
shows the same failure for a binary tilt.

\subsection{Complexity of computing useful steps at a fixed KL budget}
\label{sec:fixed-budget-calibration}

The four-state result concerns draws. We next ask whether a useful step is also
hard to compute when sampling is easy and the direction is exact. The simplest
such model tilts the base law toward one event and maximizes the
event's probability. Its natural gradient is identically $1$, so the direction
is known, while the useful step lengths depend on the unknown event mass.
\Cref{thm:accuracy-calibration} shows that calibration again costs as much as
estimating this mass, and \cref{cor:fixed-budget-hardness} encodes
satisfiability in it.
\begin{theorem}[Sample complexity of useful event-tilt selection]
\label{thm:accuracy-calibration}
Set $p_0=2^{-10}$. Let $P_0$ be a finite law and $A$ a known event with
unknown mass $r=P_0(A)\in(0,p_0]$. For $t\geq0$, tilt the law to
$P_t(s)=e^{t\one_A(s)}P_0(s)/(1-r+re^t)$ and maximize $J(t)=P_t(A)$.
Define the event probability, KL cost, gain and optimal feasible gain by
\begin{equation}
 \begin{aligned}
 q_r(t)&=\frac{re^t}{1-r+re^t},&
 \mathcal K_r(t)&=\KL(P_0\Vert P_t)=\log(1-r+re^t)-rt,\\
 \Delta_r(t)&=q_r(t)-r,&
 \Delta_r^\star&=\max_{t\geq0:\,\mathcal K_r(t)\leq1/10}\Delta_r(t).
 \end{aligned}
 \label{eq:event-calibration-main}
\end{equation}
A procedure receives the event, the tilted family and the exact natural
gradient $1$, and otherwise only draws from $P_0$. It returns a scalar step
$t\geq0$. For $0<\varepsilon\leq1/4$, this step is useful at mass $r$ when
$\mathcal K_r(t)\leq1/10$ and
$\Delta_r(t)\geq(1-\varepsilon)\Delta_r^\star$.
Fix $0<p\leq p_0/2$ and $0<\delta\leq1/3$.

\noindent\textup{(a) Lower bound.}
Consider Bernoulli base laws with masses $p$ and $(1+4\varepsilon)p$ and event
$A=\{1\}$. A procedure takes at most $m$ draws. If it returns a useful step
with probability at least $1-\delta$ under both laws, then
$m=\Omega(\log(1/\delta)/(p\varepsilon^2))$.

\noindent\textup{(b) Upper bound.}
Let $\widehat p$ be the event frequency in $m\geq1$ draws.
Use Bernoulli-KL inversion \citep{GarivierCappe2011} to form
$\mathcal I_m=\{r\in[0,1]:m\binaryKL(\widehat p\Vert r)\leq\log(2/\delta)\}$, where
$\binaryKL(a\Vert b)=a\log(a/b)+(1-a)\log((1-a)/(1-b))$
with boundary values defined by limits. Choose the largest $t\geq0$ such that
$\mathcal K_r(t)\leq1/10$ for every $r\in\mathcal I_m$, and round downward by
at most $\varepsilon/32$ to a nonnegative step. For every finite base law
with event mass $p$, this step is useful with probability at least $1-\delta$
whenever $m$ exceeds a universal constant multiple of
$\log(1/\delta)/(p\varepsilon^2)$. The constants in both bounds are uniform
over the stated parameter ranges.
\end{theorem}

Since $p\leq p_0/2$, the condition $\varepsilon\leq1/4$ keeps the hard pair
$p$ and $(1+4\varepsilon)p$ inside $(0,p_0]$, while $p_0=2^{-10}$ and the
finite-precision rounding by $\varepsilon/32$ only make the constants explicit
(\cref{app:accuracy-calibration}). As $p\downarrow0$, the optimal gain at any
fixed budget $\eta>0$ in place of
$1/10$ tends to $1-e^{-\eta}$. By the KL chain rule, restricting updates to event tilts preserves the
best feasible event-probability gain over all laws on the base support
(\cref{app:fixed-budget-bounds-proof}).
Supplying the initial gradient or scalar Fisher information would
reveal the event mass, because both equal $p(1-p)$, which increases strictly on
the stated range.

\begin{proof}[Proof sketch]
Disjoint useful-step intervals reduce the lower bound to Bernoulli testing.
For the upper bound, the worst-case KL at $t>0$ is attained at
$r=1/t-1/(e^t-1)$, clipped to $\mathcal I_m$.
Doubling and bisection locate the $1/10$ boundary.
Coverage ensures feasibility, while concentration and boundary sensitivity
bound the gain lost to conservatism and rounding. \Cref{app:accuracy-calibration} gives the proof,
and \cref{thm:fixed-budget-calibration} shows that even three quarters of the
optimal gain requires $\Omega(1/p)$ draws.
\end{proof}

\noindent\textit{From event rarity to computational hardness.}
To turn this sample cost into computational hardness, we map a Boolean formula
$\varphi$ in polynomial time to an event predicate whose mass under the uniform
law encodes $\SATcount(\varphi)$. An always-accepted anchor state keeps the mass
positive, and satisfiability at least doubles it. Eleven tag bits keep the mass
below $p_0$ (\cref{app:sat-encoding}), and the gain target $3/4$ is the most
lenient one \cref{thm:accuracy-calibration} covers. Padding with unused
variables preserves satisfiability, so requiring $n\geq3$ loses no generality.

\begin{corollary}[Computational hardness of useful event-tilt selection]
\label{cor:fixed-budget-hardness}
Let $\varphi$ be a Boolean formula on $n\geq3$ variables and let $P_0$ be
uniform on $\{0,1\}^{11}\times\{0,1\}^n$. Define the event
\[
 A_\varphi=\{(b,x):b=0^{11},\ \varphi(x)=1\}
 \cup\{(1^{11},0^n)\}.
\]
With $\DB=2^{n+11}$, its probability is
$p_\varphi=(1+\SATcount(\varphi))/\DB<2^{-10}$.
For $t\geq0$, use the tilt
$P_{\varphi,t}(s)\propto e^{t\one_{A_\varphi}(s)}P_0(s)$ and objective
$J_\varphi(t)=P_{\varphi,t}(A_\varphi)$.
Its natural gradient equals $1$, and exact base sampling and event
membership are polynomial-time computable from $\langle\varphi\rangle$.

The input is the formula encoding $\langle\varphi\rangle$ and the exact
natural gradient. The required output is a nonnegative rational step
$\widehat t$ whose bit length is polynomial in the input length. It must
satisfy $\mathcal K_{p_\varphi}(\widehat t)\leq1/10$ and
$\Delta_{p_\varphi}(\widehat t)\geq(3/4)\Delta_{p_\varphi}^\star$,
with KL cost and gain as in \cref{eq:event-calibration-main}.
This problem is NP-hard. Unless $\mathrm{NP}\subseteq\mathrm{BPP}$
\citep{AroraBarak2009}, no
polynomial-time procedure with these inputs and draws returns such a step with probability at least $2/3$ on every instance.
\end{corollary}

\begin{proof}[Proof sketch]
We run a hypothetical calibration procedure on the model encoded by
$\varphi$ and compare its returned step with the threshold
$\tsepB=(n+7)\log2$, which is computable from $n$. Useful outputs at the
unsatisfiable mass $1/\DB$ exceed $\tsepB+0.03$, whereas, by concavity of the
KL cost in the event mass, feasible outputs at every satisfiable mass
$p_\varphi\geq2/\DB$ lie below $\tsepB-0.0079$. A polynomial-time calibration
procedure would therefore decide SAT. A rational threshold accurate to
$10^{-3}$ suffices, and each base draw uses $n+11$ fair bits.
\Cref{app:fixed-budget-bounds-proof} proves the finite-precision separation
and efficient simulation of exact draws.
\end{proof}

\subsection{Complexity of natural-gradient direction recovery}
\label{sec:logistic}
The preceding results take the direction as given. We now ask whether the
direction itself is easy to recover when the Fisher matrix is well
conditioned. A bounded condition number controls only the
ratio of eigenvalues, and rare informative observations can still make both
eigenvalues small. In the construction below, the covariate $(1,1)^{\mathsf T}$
couples the two Fisher coordinates, so the unit natural gradient turns as the
number of satisfying assignments grows and thereby encodes satisfiability. Here $\DL=2^{n+5}$ counts the $n$ formula bits, three tag
bits and two branch bits (\cref{app:sat-encoding}).

\begin{theorem}[Direction recovery with a well-conditioned Fisher matrix]
\label{thm:logistic-preconditioning}
Fix a Boolean formula $\varphi$ on $n\geq3$ variables. Write
$K=1+\SATcount(\varphi)$ and $\DL=2^{n+5}$. With the standard basis $e_1,e_2$ of $\mathbb R^2$, let $X$ have law $Q_\varphi$ on
$\{0,e_1,e_2,(1,1)^{\mathsf T}\}$ given by
\begin{equation}
 \begin{aligned}
 Q_\varphi(e_1)=Q_\varphi(e_2)&=\frac K{\DL},&
 Q_\varphi((1,1)^{\mathsf T})&=\frac1{\DL},\\
 Q_\varphi(0)&=1-\frac{2K+1}{\DL}.
 \end{aligned}
 \label{eq:logistic-model}
\end{equation}
Consider the binary logistic model
$\pi_\theta(Y=1\mid X)=\sigm(\theta^{\mathsf T}X)$, where
$\theta\in\mathbb R^2$ and $\sigm(t)=(1+e^{-t})^{-1}$, with query objective
$\ell(\theta)=\log\sigm(\theta_1)$.
Let $F_\varphi^{\mathrm L}(\theta)$ be the population Fisher information matrix, averaging over
$X\sim Q_\varphi$ and $Y\sim\pi_\theta(\cdot\mid X)$.
At $\theta=0$, the exact gradient $b=\nabla\ell(0)=e_1/2$ is supplied in both recovery
problems below. Define $v_\varphi^{\mathrm L}=F_\varphi^{\mathrm L}(0)^{-1}b$
and the unit natural gradient
$\nu_\varphi^{\mathrm L}=v_\varphi^{\mathrm L}/\|v_\varphi^{\mathrm L}\|_2$.
Given $\varphi$, draws at zero, with law $Q_\varphi\otimes\Bern(1/2)$, are
polynomial-time computable.

\noindent\textup{(a) Fisher matrix and direction.} For every such formula,
\begin{equation}
 F_\varphi^{\mathrm L}(0)=\frac1{4\DL}
 \begin{pmatrix}K+1&1\\1&K+1\end{pmatrix},\qquad
 \nu_\varphi^{\mathrm L}=\frac{(K+1,-1)^{\mathsf T}}{\sqrt{(K+1)^2+1}}.
 \label{eq:logistic-fisher-direction}
\end{equation}
The spectral condition number satisfies
$\kappa_2(F_\varphi^{\mathrm L}(0))=1+2/K\leq3$, and
$\lambda_{\min}(F_\varphi^{\mathrm L}(0))=\Theta(2^{-n})$ for $K=1,2$.

\noindent\textup{(b) Direction recovery from the description.}
Consider the problem whose input consists of $(1^n,\langle\varphi\rangle)$
and $b$, where $1^n$ encodes $n$ in unary. It asks for a polynomial-length
rational vector $\widehat\nu$ with
$\|\widehat\nu-\nu_\varphi^{\mathrm L}\|_2\leq1/32$.
This problem is NP-hard. Unless $\mathrm{NP}\subseteq\mathrm{BPP}$
\citep{AroraBarak2009}, no
polynomial-time procedure achieves this accuracy with probability at least
$2/3$ on every input, even with draws at $\theta=0$.

\noindent\textup{(c) Sample-only direction recovery.}
Consider the two laws with $K=1$ and $K=2$. A procedure receives $n$, this family and $b$, and learns which law holds only
from draws $(X,Y)$ at $\theta=0$. Any procedure that takes at
most $m$ draws and returns $\widehat\nu$ with
$\|\widehat\nu-\nu_\varphi^{\mathrm L}\|_2\leq1/32$ with probability at
least $2/3$ under each law must have $m\geq\DL/6$.

\noindent\textup{(d) Coordinate-threshold decision problem.}
The language of valid encodings $(1^n,\langle\varphi\rangle)$, with $n\geq3$
and $\varphi$ on $n$ variables, for which
$(\nu_\varphi^{\mathrm L})_2>-3/8$ consists exactly of the satisfiable
instances and is NP-complete under polynomial-time many-one reductions.
\end{theorem}

\begin{proof}[Proof sketch]
Given $\varphi$, we build its model, run a recovery procedure and threshold
the output's second coordinate at $-3/8$. By
\cref{eq:logistic-fisher-direction}, this coordinate is $-1/\sqrt5$ when
$K=1$ and at least $-1/\sqrt{10}$ when $K\geq2$. Both values lie more than
$1/32$ from $-3/8$, so an accurate output reveals satisfiability. Satisfying
assignments certify membership in the threshold language.
From samples, the $K=1,2$ laws differ in total variation by $2/\DL$
per draw, so testing needs $m\geq\DL/6$. \Cref{app:logistic-proof} gives the proof,
and \cref{fig:logistic-direction-threshold} illustrates the threshold.
\end{proof}

The hardness extends to finite updates. At KL budget $1/(100\DL)$ and with no
supplied direction, returning a feasible update within $1/200$ of the constrained query optimum is
NP-hard, because the small budget keeps every feasible first coordinate at
most $1/5$ when $K\geq2$, whereas such an update for $K=1$ needs a larger one. Succeeding from draws alone with
probability $2/3$ needs $m\geq\DL/6$
(\cref{cor:logistic-trust-region,cor:logistic-useful-samples}).

\subsection{Damped estimation and KL certification}
\label{sec:sampling}

The remaining results are positive. The logistic construction
makes direction recovery hard through small Fisher eigenvalues. Damping is the standard remedy \citep{Martens2020,MartensGrosse2015}.
Replacing $F$ with $F+\lambda I$ for a damping parameter $\lambda>0$ raises
every eigenvalue to at least $\lambda$,
and we ask how many draws then suffice to estimate the damped direction. The
answer is tight up to a logarithmic factor, and \cref{app:regularization}
bounds the damping bias relative to the undamped solution.

\begin{proposition}[Damped natural gradient estimation]
\label{prop:damped-sampling-main}
Let $P$ be an unknown probability law on $\mathcal X$ and
$h:\mathcal X\to[0,1]^d$ a known measurable statistic, with integer $d\geq1$.
For known $\lambda>0$, set $F=\Cov_P(h)$ and
${\mathsf A}=F+\lambda I_d$. Fix $0<\varepsilon\leq1/16$ and
$0<\delta\leq1/3$. From $m=2s$ draws $S_1,\ldots,S_{2s}\sim P$,
where $s\geq1$ is an integer, form
\[
 \xi_i=\frac{h(S_{2i})-h(S_{2i-1})}{\sqrt2},\qquad
 \widehat F=\frac1s\sum_{i=1}^s\xi_i\xi_i^{\mathsf T},\qquad
 \widehat v(b)=(\widehat F+\lambda I_d)^{-1}b.
\]
For a universal constant $C$, if
$m\geq C(1+d/\lambda)\varepsilon^{-2}\log(2d/\delta)$, then with probability
at least $1-\delta$,
\begin{equation}
 \|\widehat v(b)-{\mathsf A}^{-1}b\|_{\mathsf A}
 \leq\varepsilon\|{\mathsf A}^{-1}b\|_{\mathsf A}
 \qquad\text{for every }b\in\mathbb R^d\setminus\{0\},
 \label{eq:damped-solve-main}
\end{equation}
where $\|z\|_{\mathsf A}^2=z^{\mathsf T}{\mathsf A}z$ and $b$ is the
supplied objective gradient.

Conversely, for each $0<\lambda\leq d/8$, take
$h(z)=z\mathbf1_d$ for $z\in\{0,1\}$ and $b=\mathbf1_d/\sqrt d$, where
$\mathbf1_d$ is the all-ones vector. Consider a procedure that learns an unknown Bernoulli law $P$ only from at
most $m$ draws. If it satisfies \cref{eq:damped-solve-main} for this fixed $b$
with probability at least $2/3$ for every such $P$, then
$m=\Omega(d/(\lambda\varepsilon^2))$.
\end{proposition}

\begin{proof}[Proof sketch]
Paired differences give an unbiased covariance estimate, and matrix
concentration bounds the solve error \citep{HsuEtAl2012,Tropp2012}. A Bernoulli
statistic repeated across coordinates gives the lower bound, and the ranges
$\varepsilon\leq1/16$ and $\lambda\leq d/8$ keep that hard pair's masses small and
on the damping scale. \Cref{app:damped-fisher-proof} proves both bounds.
\end{proof}

\noindent\textbf{Certifying the KL budget.}
The lower bounds rule out cheap useful steps in the worst case, yet respecting
the budget can still be certified from draws. For a softmax classifier with
affine logits, a known bound on the feature norm limits how far each candidate
step can move the logits, and hence its KL at every input, before any sampling. Fix $\Ncand$ candidate steps
along a supplied direction, and draw $m$ calibration inputs independently of
the classifier, direction and candidates. Empirical Bernstein bounds
\citep{MaurerPontil2009}, each at failure probability $\delta/\Ncand$, then
bound the population KL of every candidate at once. Their slack grows at most
logarithmically in $\Ncand/\delta$ and shrinks as $m^{-1/2}$, or as $m^{-1}$
when the sample variance is small. Any candidate
whose bound is at most $\eta$, even one chosen using the same draws, is
feasible with probability at least $1-\delta$, and the zero step is always
accepted (\cref{prop:affine-calibration-main}). The certificate controls
feasibility, not gain.

\section{Empirical Illustrations}
\label{sec:experiments}

The experiments examine three distinctions in the theory: calibration sample
cost, direction accuracy versus finite gain, and feasibility under sampled KL.
Each trial computes the KL cost and gain of its returned update $u$, and the
gain benchmark, exactly under the reference law. The trial is a \emph{joint
success} if $u$ is both feasible and within the gain tolerance, that is,
useful (\cref{def:useful}), and its \emph{feasible gain}
$\Delta(u)\,\one\{\mathcal K(u)\leq\eta\}$ is zero for infeasible updates.

\noindent\textbf{Calibration sample scales.}
The upper-bound rules of \cref{thm:matched-moment-main,thm:accuracy-calibration}
stop conservatively short of an estimated KL boundary, since stopping at the
estimate itself overshoots about half the time. Their rates also predict similar
success across rarity levels at a fixed normalized budget, $m\pM\varepsilon^2$
or $mp\varepsilon^2$ with $m$ draws.
Both four-state rules receive exact local information and estimate the mass
of $h=3$ from the same count. The plug-in rule stops at the estimated KL
boundary, whereas the conservative rule adds an upward offset to the mass
estimate and rounds the step downward. At $\omega=1.25$, joint
success reaches $100\%$ for the conservative rule at the largest sample
budget, while plug-in success remains near $50\%$
(\cref{tab:calibration-sampling-main}(a)). At that budget, the plug-in rule's
mean positive KL excess is only about $0.004\%$ of the budget
(\cref{tab:matched-rounding-ablation}), so its failures are strict but tiny
boundary crossings.
For binary laws, the confidence rule receives the event count and sample size
but not the true mass. At $mp\varepsilon^2=16$, joint success ranges from
$95.8\%$ to $97.1\%$ as $p$ varies from $10^{-6}$ to $2^{-11}$
(\cref{tab:calibration-sampling-main}(b) and
\cref{app:experimental-details}).

\noindent\textbf{Direction accuracy and finite gain.}
\Cref{thm:logistic-preconditioning} concerns direction accuracy, whereas
usefulness concerns the finite step. We separate the two in the logistic query
of \cref{sec:logistic}, with the globally optimal KL-feasible query gain as
benchmark. Fix $K=1$, $\DL=65{,}536$ and
$\eta=0.05/\DL$. At $m=16\DL$, only $31.9\%$ of estimated
pseudoinverse directions (\cref{fig:construction-geometry}(c)) meet the
theorem's Euclidean tolerance $1/32$.
Population-KL calibration nevertheless attains $99.5\%$ of optimal gain on
average and meets the $90\%$ gain target in every trial.
The mean cosine between the estimated and population rays in the Fisher
geometry is $99.46\%$, which explains why the Euclidean pass rate
need not predict finite gain
(\cref{app:logistic-finite-step,tab:logistic-fisher-angle}).
Even along the exact direction, confidence calibration is useful in only
$1.1\%$ of trials at $m=16\DL$, rising to $98.9\%$ at $m=256\DL$
(\cref{tab:logistic-oracle-controls}). It uses counts even though this
direction identifies the cell masses, so its failures measure its conservatism
rather than direction error, and \cref{fig:construction-geometry}(c) shows the
same pattern with estimated directions. Thus, at $m=16\DL$, calibrating the
length, not estimating the direction, limits usefulness.

\begin{figure}[!thbp]
\begingroup
\centering
\includegraphics[width=0.90\linewidth]{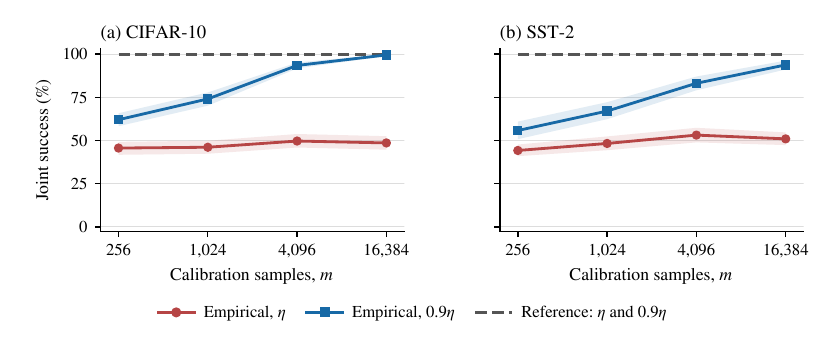}
\caption{\textbf{Calibration along supplied directions in classifier heads.}
\textbf{(a)} CIFAR-10 with ResNet-18 features and \textbf{(b)} SST-2 with
DistilBERT features. Directions and population KL use the full training set,
and calibration uses $m$ sampled inputs with KL budget
$\eta=0.01$ or the margin target $0.9\eta$.
Bands are 95\% bootstrap intervals over 32 queries, with 20 repetitions each.
Joint success requires KL feasibility and at least 75\% of the optimal feasible
gain along the supplied direction at the full budget. Dashed curves use
population KL, and colored curves use sampled KL.
}
\label{fig:learned-updates-recovery}
\endgroup
\end{figure}

\noindent\textbf{Calibration in learned classifier heads.}
Learned classifier heads test whether stopping at a sampled KL boundary is
reliable without constructed rarity.
We freeze ResNet-18 representations \citep{HeEtAl2016} on CIFAR-10
\citep{Krizhevsky2009} and DistilBERT representations \citep{SanhEtAl2019} on
SST-2 \citep{WangEtAl2019,SocherEtAl2013}. Each trial restores the original
head and increases an initially misclassified query's true-label log
probability. Directions computed on the full training set isolate calibration, and the
benchmark
is the optimal feasible gain along the supplied ray.
The empirical rules share calibration inputs and target the full KL budget or
reserve a $10\%$ margin.
Under exact KL, convexity of the cost and concavity of the query gain ensure
that this margin retains at least $90\%$ of the full-budget ray optimum
(\cref{lem:learned-margin-control}). At the primary budget $\eta=0.01$ and the largest sample size
$m=16{,}384$,
the margin rule's joint success is $99.53\%$ on CIFAR-10 and $93.75\%$ on
SST-2, compared with $48.59\%$ and $50.94\%$ at the full budget, which stays near
$50\%$ at every sample size because the sampled boundary lies beyond the
population one about half the time
(\cref{fig:learned-updates-recovery,tab:learned-calibration-main}).
Joint success uses a $75\%$ gain target, and every margin update at this setting
also meets a stricter $90\%$ target, so the remaining failures are KL violations. Including infeasible updates, the
margin keeps the mean raw-gain ratio to the full-budget optimum above $99\%$ in
both models, so it costs almost no gain (magnitudes, other
budgets and query uncertainty in \cref{app:learned-updates}).

\section{Discussion}
\label{sec:discussion}

Our results locate a cost that direction-centered analyses leave out. Choosing
a useful length needs information about rare states that local quantities do
not carry, acquiring it costs as many draws as estimating a rare probability,
and computing a useful step can be \textit{intractable} even when sampling is
easy and efficient. A natural-gradient update therefore has two distinct costs, estimating the
direction and calibrating its length, and the second can be large even when the first is zero.

Any procedure that guarantees useful steps over a larger
class containing the four-state family must succeed on the family, so the
sample lower bound transfers to that class under the same access and
permitted updates. The rate $\log(1/\delta)/(p\varepsilon^2)$ is classical for
rare-probability estimation, and our contribution is that calibration inherits
it, up to constants, even with every local quantity supplied. The hardness results hold with
efficient exact sampling, so they do not restate the hardness of inference
(\cref{sec:bg}).
Classifier heads cross the boundary about half the time
without constructed rarity.

For natural policy gradient, TRPO and K-FAC
\citep{Kakade2001,SchulmanTrustRegion2015,MartensGrosse2015}, the length
deserves the scrutiny now given to the direction. Sample budgets should
include the draws and computation spent on selecting it, and evaluations should
state the KL budget, gain benchmark and success probability together. Violation frequencies need their
magnitudes, because under the conditions in \cref{app:learned-access} the
probability of crossing the boundary tends to one half while the excess KL
vanishes. A small KL margin is a cheap safeguard (\cref{sec:experiments}), and
independent calibration draws can certify the budget (\cref{sec:sampling}).

\noindent\textbf{Limitations.}
Our bounds are worst case, concern single updates and hold under the stated
access and output restrictions. Probability tables or extra evaluation oracles can change the
information cost, and event-tilt hardness concerns returning a tilt
parameter. Under the reverse KL orientation, the optimal event-tilt gain
vanishes as $p\downarrow0$ (\cref{app:divergence-orientation}). The experiments
fix checkpoints, representations and queries under a finite reference law,
and do not test generalization or training trajectories.

\phantomsection\label{section:future}
\noindent\textbf{Open questions.}
We highlight three questions for future research.
First, how do single-update bounds accumulate
over a training run, under either KL orientation?
Second, what does it cost to select the direction
and the length from the same draws, when confidence bounds must remain
valid under reuse \citep{Howard2021ConfidenceSequences,DworkEtAl2015}? Third,
can importance sampling that oversamples the rare states reduce the draws
needed for a useful step
\citep{GlynnIglehart1989,ChatterjeeDiaconis2018}?

% References formatted by BibTeX (plainnat) from the paper's references.bib.

\clearpage
\setcounter{equation}{0}
\renewcommand{\theequation}{A\arabic{equation}}
\renewcommand{\theHequation}{appendix.\arabic{equation}}
\setcounter{table}{0}
\renewcommand{\thetable}{A\arabic{table}}
\renewcommand{\theHtable}{appendix.\arabic{table}}
\setcounter{figure}{0}
\renewcommand{\thefigure}{A\arabic{figure}}
\renewcommand{\theHfigure}{appendix.\arabic{figure}}
\crefalias{section}{appendix}
\crefalias{subsection}{appendix}
\crefalias{subsubsection}{appendix}
\appendix
\section*{Appendix}
The appendices follow the order of the main text. They first prove the
theoretical results and then report reproducibility details and the
experiments that illustrate them.
\Cref{app:procedures} fixes the model of a procedure, the padding lemma used
by every sample lower bound and the satisfiability encoding used by both
hardness results.
\Cref{app:fixed-budget-proof} treats calibration with an exact natural gradient
at a fixed budget. Within it, \cref{app:matched-moments} proves
\cref{thm:matched-moment-main} through
\cref{thm:matched-moment-calibration}(b)--(c).
\Cref{app:fixed-budget-bounds-proof} proves \cref{cor:fixed-budget-hardness}
using the constant-fraction separation in \cref{thm:fixed-budget-calibration},
and \cref{app:accuracy-calibration} proves \cref{thm:accuracy-calibration}.
\Cref{app:additional-calibration} compares quadratic and finite-step KL.
\Cref{app:logistic-proof} proves \cref{thm:logistic-preconditioning} on
direction recovery and extends it to useful finite updates
(\cref{cor:logistic-trust-region,cor:logistic-useful-samples}).
\Cref{app:regularization} proves \cref{prop:damped-sampling-main} through
\cref{thm:damped-fisher-sampling,cor:damped-dimension-lower} and bounds the bias introduced by damping. \Cref{app:finite-candidate-calibration} states the population-KL certificate
summarized in \cref{sec:sampling} (\cref{prop:affine-calibration-main}) and
proves it through \cref{prop:finite-candidate-calibration}.
\Cref{app:provenance} gives reproducibility details. The experiments then cover
count simulations (\cref{app:experimental-details}), logistic controls and
damping simulations (\cref{app:logistic-experiments}) and learned classifier heads
(\cref{app:learned-updates}).

\noindent\textbf{Notation across the appendices.}
The sample count \(m\) counts individual observations. Thus \(s\) independent
pairs contain \(2s\) observations. The four-state family uses supplied scale
\(\pM\) and unknown coefficient \(\omega\), while \(p\) denotes an event probability and
\(c\) a formula code. The count \(K\) includes one always-accepted state
in addition to the satisfying assignments. KL costs are written \(\mathcal K\),
and \(\DL,\DB\) denote the logistic and event-tilt encoding sizes. Absolute KL budgets are \(\eta\), and logistic experiments also use
\(\bar\eta=\DL\eta\).

\section{Procedures and lower-bound tools}
\label{app:procedures}

Every sample lower bound in the appendices uses the following model of a
procedure, which makes the description in \cref{sec:bg} precise.
Draws \(X_1,X_2,\ldots\) are independent with a common law \(P\) on a
measurable space, and an independent seed \(U\) has a law that does not depend
on \(P\). Write \(\mathcal F_k=\sigma(U,X_1,\ldots,X_k)\) for the information
available after \(k\) draws. A procedure with sample budget \(m\) stops after
\(N\leq m\) draws, where \(N\) is a stopping time for
\((\mathcal F_k)_{k\geq0}\). Thus whether the procedure takes another draw depends
only on the seed and the draws seen so far. Its output lies in a Euclidean space and
is \(\mathcal F_N\)-measurable. Supplied quantities, such as an exact natural
gradient, are identical under the laws being compared and enter as fixed
inputs. Write \(\mathcal L_P(Y)\) for the law of a random variable \(Y\) when the
draws have law \(P\).

\begin{lemma}[Padding and two-point testing]
\label{lem:padding}
The output equals \(g(U,X_1,\ldots,X_m)\) for a measurable map \(g\) that
does not depend on \(P\). Consequently, for two laws \(P\) and \(P'\) of the
draws, the output \(Y\) satisfies
\begin{equation}
 \begin{aligned}
 \KL(\mathcal L_P(Y)\Vert\mathcal L_{P'}(Y))&\leq m\KL(P\Vert P'),\\
 \TV(\mathcal L_P(Y),\mathcal L_{P'}(Y))&\leq1-(1-\TV(P,P'))^m\leq m\TV(P,P').
 \end{aligned}
 \label{eq:padding-bounds}
\end{equation}
If \(S\) and \(S'\) are disjoint measurable sets of outputs, then
\(\Prob_P(Y\notin S)+\Prob_{P'}(Y\notin S')\geq
1-\TV(\mathcal L_P(Y),\mathcal L_{P'}(Y))\).
\end{lemma}

\begin{proof}
On \(\{N=k\}\), the output is \(\mathcal F_k\)-measurable, so the Doob--Dynkin
lemma gives a measurable map \(g_k\) with \(Y=g_k(U,X_1,\ldots,X_k)\) there.
The map \(g=\sum_{k\leq m}\one\{N=k\}g_k\) ignores every draw after the
stopping time, so padding the transcript to \(m\) draws leaves the output
unchanged. The seed has the same law under both hypotheses and is independent
of the draws, so the chain rule for KL and independence give
\(\KL\) equal to \(m\KL(P\Vert P')\) for \((U,X_1,\ldots,X_m)\). Data processing
\citep{vanErvenHarremoes2014} transfers this bound to
\(Y=g(U,X_1,\ldots,X_m)\). For total variation, couple each pair of draws
maximally, independently across positions, and use the same seed under both
laws. The two transcripts then agree with probability \((1-\TV(P,P'))^m\), and
so do the outputs. Bernoulli's inequality gives the last bound in
\cref{eq:padding-bounds}. Finally, \(S'\) lies in the complement of \(S\), so
\(\Prob_P(Y\in S)+\Prob_{P'}(Y\in S')\leq
1+\TV(\mathcal L_P(Y),\mathcal L_{P'}(Y))\).
\end{proof}

In each lower bound below, useful outputs under two laws form disjoint sets.
The lemma then turns a success guarantee under both laws into a bound on the
number of draws, whatever the procedure's computational resources.

\noindent\textbf{How the constructions encode satisfiability.}
\phantomsection\label{app:sat-encoding}
Both hardness results start from a classical encoding. An assignment \(x\)
drawn uniformly from \(\{0,1\}^n\) satisfies \(\varphi\) with probability
\(\SATcount(\varphi)/2^n\), so fair bits and one formula evaluation sample a
law that encodes the number of satisfying assignments. This encoding makes exact and approximate
inference in Bayesian networks hard
\citep{Cooper1990,DagumLuby1993,Roth1996}. Computing the count is
\#P-complete \citep{Valiant1979}, yet sampling stays easy because draws range
over all assignments rather than satisfying ones \citep{JerrumEtAl1986}.
We add three ingredients. An \emph{anchor state} is always accepted, so the
accepted count \(K=1+\SATcount(\varphi)\) is at least one. This keeps event masses
positive, as \cref{thm:accuracy-calibration} requires, and the logistic Fisher
matrix positive definite. Satisfiability then becomes the question
whether \(K=1\) or \(K\geq2\). \emph{Tag bits}, drawn uniformly with \(x\),
split the states into disjoint blocks with separate acceptance rules, such as
one for satisfying assignments and one for the anchor. Each tag bit halves every state's mass. Eleven tag bits are the fewest that keep the total mass of \(2^n+1\) accepted states
below \(p_0=2^{-10}\) in \cref{cor:fixed-budget-hardness}.
Three tag bits separate the formula block, the anchor and one fixed state in
\cref{thm:logistic-preconditioning}. In the logistic model, two \emph{branch
bits} choose whether a state produces \(e_1\), \(e_2\), \((1,1)^\mathsf T\) or
\(0\), so every covariate probability is a count divided by \(\DL=2^{n+5}\). The shared covariate \((1,1)^\mathsf T\) couples the two Fisher coordinates,
so the direction turns as \(K\) grows. When there is only one satisfying assignment, a draw reveals it with
probability exponentially small in \(n\). Thus sampling cannot cheaply
distinguish \(K=1\) from \(K\geq2\) uniformly over formulas. The separating threshold depends only on \(n\),
so a polynomial-time procedure that returns a useful step or an accurate
direction would decide SAT.

\section{Fixed-budget calibration with an exact natural gradient}
\label{app:fixed-budget-proof}

This appendix proves the results of
\cref{sec:matched-local-calibration,sec:fixed-budget-calibration}.
\Cref{app:matched-moments} treats the four-state family,
\cref{app:fixed-budget-bounds-proof} gives the event-tilt bounds and the
satisfiability reduction, and \cref{app:accuracy-calibration} derives the
dependence on the required improvement.

\subsection{Calibration with exact gradient and scalar Fisher information}
\label{app:matched-moments}

We prove \cref{thm:matched-moment-main} by varying the finite-step KL
boundary while preserving the moments that determine the supplied gradient
and scalar Fisher information. We first derive the
local identities and boundary geometry, then state the quantitative
bounds in \cref{thm:matched-moment-calibration}.
Its parts~\textup{(b)}--\textup{(c)} give the main result on
\(\omega\in[1,3/2]\). The wider interval \(\omega\leq2\) also permits the
constant-fraction comparison in part~\textup{(a)}.

\noindent\textbf{Minimal support for preserving two moments.}
For fixed distinct statistic values $a_1,\ldots,a_k$, normalization,
mean and second moment constrain the probabilities through the matrix with
columns $(1,a_j,a_j^2)^{\mathsf T}$. Its Vandermonde rank is $\min\{3,k\}$.
For $k\leq3$, these constraints therefore determine the law uniquely.
Four values are the minimum permitting distinct laws with the same first
two moments. On $\{0,1,2,3\}$, varying $\omega$ in the construction below
changes the mass vector in direction $\pM(-1,3,-3,1)$, which lies in this
matrix's one-dimensional nullspace. Thus the family realizes this minimum
while retaining positive probabilities.

Set \(\pMmax=2^{-30}\). For \(0<\pM\leq\pMmax\) and \(1\leq\omega\leq2\),
define a law on \(h\in\{0,1,2,3\}\) by
\begin{equation}
 \bigl(P_{\pM,\omega}(0),P_{\pM,\omega}(1),P_{\pM,\omega}(2),P_{\pM,\omega}(3)\bigr)
 =\bigl(1-(5+\omega)\pM,(3\omega-2)\pM,(7-3\omega)\pM,\omega\pM\bigr).
 \label{eq:matched-moment-law}
\end{equation}
The scalar parameter \(t\geq0\) is optimized, while \(\omega\) indexes the
unknown base law. The main text suppresses the fixed scale \(\pM\) in
\(P_\omega\). Tilt by \(e^{th}\), and write
\begin{equation}
 \begin{aligned}
 Z_\omega(t)&=\E_{P_{\pM,\omega}}e^{th},&
 P_{\pM,\omega,t}(h)&=e^{th}P_{\pM,\omega}(h)/Z_\omega(t),\\
 J_\omega(t)&=\E_{P_{\pM,\omega,t}}h,&
 {\mathcal K}_\omega(t)&=\KL(P_{\pM,\omega}\Vert P_{\pM,\omega,t}),\\
 \Delta_\omega(t)&=J_\omega(t)-J_\omega(0),&
 \Delta_\omega^\star&=\max_{t\geq0:\,{\mathcal K}_\omega(t)\leq1/10}\Delta_\omega(t).
 \end{aligned}
 \label{eq:matched-moment-objective}
\end{equation}
This objective is a mean of a four-valued statistic. Let \(b_\omega(t)=J_\omega'(t)\)
and let \(F_\omega(t)\) denote the scalar Fisher information for the coordinate
\(t\). A procedure receives \(\pM\), the family specification, the statistic
\(h\), and the exact \(b_\omega(0),F_\omega(0)\) and natural gradient.
Its only further information about \(\omega\) consists of independent base draws.

\noindent\textbf{Local identities and the feasible boundary.}
All four masses are positive and sum to one. Direct summation gives
\(\E h=12\pM\) and \(\E h^2=26\pM\), independently of \(\omega\).
The normalizers satisfy
$Z_{\omega'}(t)-Z_\omega(t)=(\omega'-\omega)\pM(e^t-1)^3$.
Thus their difference vanishes through second order at zero while
changing the finite-step KL boundary.
The tilted score is \(h-J_\omega(t)\), so
\(J_\omega'(t)=F_\omega(t)=\Var_{P_{\pM,\omega,t}}h>0\), giving the shared local
quantities in \cref{eq:matched-moment-statistics} below.
Moreover,
\({\mathcal K}_\omega(t)=\log Z_\omega(t)-12\pM t\) and \({\mathcal K}_\omega'(t)=J_\omega(t)-12\pM>0\) for \(t>0\).
Since
\({\mathcal K}_\omega(t)\geq\log(\omega\pM)+(3-12\pM)t\), there is a unique positive
\(T_\omega\) with \({\mathcal K}_\omega(T_\omega)=1/10\). The objective increases, so
\(\Delta_\omega^\star=\Delta_\omega(T_\omega)\).
Thus a useful step must be large enough to achieve the requested gain
and small enough to remain below \(T_\omega\). For the lower bounds, we show that these allowable intervals are disjoint for selected pairs of laws.
For the upper bound, we estimate a boundary that keeps the returned step feasible under the unknown law.

\begin{theorem}[Calibration with matched base gradients and scalar Fisher information]
\label{thm:matched-moment-calibration}
Fix \(0<\pM\leq\pMmax=2^{-30}\) and \(0<\delta\leq1/3\).
Use the four-state law \cref{eq:matched-moment-law}, tilt, gain and benchmark
in \cref{eq:matched-moment-objective}. A step \(t\geq0\) is KL-feasible
when \(\mathcal K_\omega(t)\leq1/10\).
Every law with \(\omega\in[1,2]\) satisfies, for every \(t\geq0\),
\begin{equation}
 b_\omega(0)=F_\omega(0)=26\pM-144\pM^2>0,
 \qquad F_\omega(t)^{-1}b_\omega(t)=1.
 \label{eq:matched-moment-statistics}
\end{equation}
Consider total measurable procedures with an independent random seed and
at most a deterministic integer \(m\) base draws, including early stopping,
under the access specified above. Totality means that the procedure returns
an output for every sample sequence and every value of its seed.
\begin{enumerate}[label=\textup{(\alph*)},leftmargin=*]
 \item If an output \(\widehat t\geq0\) is KL-feasible and attains at least
 \(3/4\) of \(\Delta_\omega^\star\), with probability at least \(1-\delta\)
 under each of \(\omega=1,2\), then
 \begin{equation}
  m\geq\frac{\log(1/(2\delta))}{-\log(1-4\pM)}.
  \label{eq:matched-moment-constant-lower}
 \end{equation}
 At \(\delta=1/3\), this implies \(m\geq1/(16\pM)\).
 \item For \(0<\varepsilon\leq1/8\), if an output is KL-feasible and
 attains at least \((1-\varepsilon)\Delta_\omega^\star\), with probability
 at least \(1-\delta\) under both \(\omega=1\) and \(\omega=1+4\varepsilon\), then
 \begin{equation}
  m\geq\frac{(1-2\delta)\log((1-\delta)/\delta)}
                 {220\pM\varepsilon^2}.
  \label{eq:matched-moment-accuracy-lower}
 \end{equation}
 \item Conversely, for \(0<\varepsilon\leq1/8\), a rule using the supplied
 \(\pM\) returns a KL-feasible step with gain at least
 \((1-\varepsilon)\Delta_\omega^\star\), with probability at least
 \(1-\delta\) for every \(\omega\in[1,3/2]\), whenever
 \begin{equation}
  m\geq\frac{120000\log(2/\delta)}{\pM\varepsilon^2}.
  \label{eq:matched-moment-accuracy-upper}
 \end{equation}
 It can return a nonnegative rational step. This guarantee concerns
 sample count and does not assert a running-time bound.
\end{enumerate}
\end{theorem}

\begin{proof}[Proof of \cref{thm:matched-moment-calibration}]
We first establish a uniform lower bound on gain.
We then prove the sample lower bounds by separating useful-step intervals
and obtain the upper bound by confidence calibration.

\noindent\textbf{Uniform gain bound.}
A feasible candidate gives a lower bound on the optimal gain, uniformly in \(\omega\).
Take \(t_{\mathrm{M,cand}}=\frac13\log(1/(10\omega\pM))\), and put \(z=e^{t_{\mathrm{M,cand}}}\geq1\).
This step gives the state \(h=3\) a fixed unnormalized weight. The weights
of \(h=1,2\), involving lower powers of the tilt, remain small.
Then \(\omega\pM z^3=1/10\), \(\pM z\leq\pM z^2\), and
\((\pM z^2)^3=\pM/(100\omega^2)\leq1/4096^3\).
Both middle coefficients in \cref{eq:matched-moment-law} are at most four,
giving
\begin{equation}
 Z_\omega(t_{\mathrm{M,cand}})\leq\frac{11}{10}+\frac8{4096}
 <\frac{221}{200}\leq e^{1/10}.
\end{equation}
Thus \(t_{\mathrm{M,cand}} \) is feasible. Its first-moment numerator is at least \(3/10\),
and \(Z_\omega(t_{\mathrm{M,cand}})<111/100\), so
\begin{equation}
 \Delta_\omega^\star\geq\frac{10}{37}-12\pM
 \geq\frac{10}{37}-\frac1{10000}>\frac{27}{100}.
 \label{eq:matched-moment-gain-lower}
\end{equation}

\noindent\textbf{Part (a): a fixed separator.}
For part~\textup{(a)}, set \(\tsepM=\frac13\log(1/(16\pM))\) and \(z=e^{\tsepM}\).
This step is too short to attain the required gain under the first law and infeasible under the
second.
At this step, \(\pM z^3=1/16\) and \(\pM z\leq\pM z^2\leq1/4096\).
Under \(\omega=1\), \(Z_1(\tsepM)\geq1-6\pM\geq31/32\), and hence
\begin{equation}
 J_1(\tsepM)\leq\frac{3/16+9/4096}{31/32}
 <\frac15<\frac34\frac{27}{100}.
\end{equation}
Every useful output under this law therefore exceeds \(\tsepM\).
Under \(\omega=2\), \(Z_2(\tsepM)\geq9/8-7\pM>1.12\).
Using \(\log(1+x)\geq x-x^2/2\) for \(x\geq0\) gives
\(\log Z_2(\tsepM)>0.11\). Since \(\pM\log(1/\pM)\leq2\sqrt\pM\),
\begin{equation}
 {\mathcal K}_2(\tsepM)>0.11-8\sqrt\pM>0.1.
\end{equation}
The bound on \(\pM\log(1/\pM)\) follows from \(\log x\leq x-1\) at
\(x=1/\sqrt\pM\). Every feasible output under \(\omega=2\) is below \(\tsepM\).
The mass differences between the two base laws are \(\pM(-1,3,-3,1)\),
so their total variation is \(4\pM\). Their useful-step sets are disjoint, so
\cref{lem:padding} bounds the error sum below by \((1-4\pM)^m\), proving \cref{eq:matched-moment-constant-lower}.
At \(\delta=1/3\), use \(\log(3/2)\geq1/3\),
\(-\log(1-4\pM)\leq4\pM/(1-4\pM)\), and \(\pM\leq1/16\).

\noindent\textbf{Part (b): interval width and boundary motion.}
For part~\textup{(b)}, restrict \(\omega\) to \([1,3/2]\).
Let \(L_{\omega,\varepsilon}\in(0,T_\omega)\) be the unique step with
\(\Delta_\omega(L_{\omega,\varepsilon})=(1-\varepsilon)\Delta_\omega^\star\).
The useful-step interval is \([L_{\omega,\varepsilon},T_\omega]\).
The contribution of \(h=3\) to the normalizer and the boundary equation give
\begin{equation}
 12\pM T_\omega\leq\frac{4\pM}{1-4\pM}\left(\frac1{10}-\log\pM\right)
 \leq8\left(\frac\pM {10}+2\sqrt\pM\right)<\frac1{1024}.
\end{equation}
Using \(e^x\leq(1-x)^{-1}\) for \(0\leq x<1\), we obtain
\(Z_\omega(T_\omega)=e^{1/10+12\pM T_\omega}\leq9/8\).
Since \(h\leq3\),
\begin{equation}
 J_\omega(T_\omega)\leq3\left(1-\frac{P_{\pM,\omega}(0)}{Z_\omega(T_\omega)}\right)
 \leq\frac13+\frac{52}{3}\pM\leq\frac7{20}.
 \label{eq:matched-moment-boundary-upper}
\end{equation}

Write \(q_j=P_{\pM,\omega,t}(h=j)\) and \(z=e^t\).
For \(0\leq t\leq T_\omega\), \(\omega\pM z^3\leq Z_\omega(t)\leq9/8\), so
\((\pM z^2)^3\leq81\pM/64\leq1/512^3\).
Also \(Z_\omega(t)\geq1\), \(z\geq1\), \(3\omega-2\leq5/2\), and
\(7-3\omega\leq4\). Consequently \(q_1+q_2\leq13/1024\).
On the useful-step interval,
\(189/800=(7/8)(27/100)\leq J_\omega(t)\leq7/20\).
Under the tilted law, the identity \(\E h^2=3J_\omega(t)-2(q_1+q_2)\) yields
\begin{equation}
 J_\omega'(t)=3J_\omega(t)-J_\omega(t)^2-2(q_1+q_2)
 \geq3\frac{189}{800}-\left(\frac{189}{800}\right)^2
       -\frac{13}{512}>\frac35.
\end{equation}
Here \(3J-J^2\) increases on the stated interval.
Integrating and using \cref{eq:matched-moment-boundary-upper} gives
\begin{equation}
 T_\omega-L_{\omega,\varepsilon}\leq\frac{\varepsilon\Delta_\omega^\star}{3/5}
 \leq\frac7{12}\varepsilon.
 \label{eq:matched-moment-success-width}
\end{equation}

The implicit-function theorem applies to \({\mathcal K}_\omega(T_\omega)=1/10\), since
\(\partial_t{\mathcal K}_\omega(T_\omega)=\Delta_\omega^\star>0\).
Uniqueness makes the local solutions agree on overlaps, so they define a
differentiable boundary function on a neighborhood of \([1,3/2]\).
Differentiating in \(\omega\), with \(z=e^{T_\omega}\), gives
\begin{equation}
 -\partial_\omega T_\omega=\frac{\pM(z-1)^3}{Z_\omega(T_\omega)(J_\omega(T_\omega)-12\pM)}.
\end{equation}
The boundary satisfies \(z>64\): at \(t=\log64\),
\(Z_\omega(t)\leq1+409760\pM<11/10\), so \({\mathcal K}_\omega(t)<1/10\).
We enlarge the denominator by dropping the negative term
\(-12\pM Z_\omega(T_\omega)\) and bounding the three remaining coefficients
using the stated parameter range. This gives
\begin{equation}
 -\partial_\omega T_\omega\geq
 \frac{(1-1/z)^3}{(5/2)/z^2+8/z+9/2}
 \geq\frac{(63/64)^3}{(5/2)/64^2+8/64+9/2}>\frac15.
\end{equation}
Thus \(T_1-T_{1+4\varepsilon}\geq(4/5)\varepsilon\), which exceeds
the width in \cref{eq:matched-moment-success-width}.
In particular, \(L_{1,\varepsilon}>T_{1+4\varepsilon}\).

\noindent\textbf{Per-observation KL bound.}
The base mass differences are \(\varepsilon\pM(-4,12,-12,4)\).
Applying \(\log x\leq x-1\) termwise gives
\begin{equation}
 \begin{aligned}
 \KL(P_{\pM,1}\Vert P_{\pM,1+4\varepsilon})
 &\leq\varepsilon^2\pM\left[
 \frac{16\pM}{1-(6+4\varepsilon)\pM}+\frac{144}{1+12\varepsilon}
 +\frac{144}{4-12\varepsilon}+\frac{16}{1+4\varepsilon}\right]\\
 &\leq\varepsilon^2\pM\left[1+144+\frac{288}{5}+16\right]
 <220\pM\varepsilon^2.
 \end{aligned}
\end{equation}
Thresholding the output between the disjoint intervals defines a test
with both errors at most \(\delta\).
\Cref{lem:padding} and monotonicity in the two test probabilities give
\begin{equation}
 m\KL(P_{\pM,1}\Vert P_{\pM,1+4\varepsilon})
 \geq(1-2\delta)\log((1-\delta)/\delta).
\end{equation}
The supplied quantities are identical under both laws. This proves \cref{eq:matched-moment-accuracy-lower}.

\noindent\textbf{Part (c): confidence calibration and rounding.}
The supplied scale converts the observed frequency of \(h=3\) into an
estimate of \(\omega\). We add a positive offset to this estimate and clip the result to the allowed parameter interval.
On the concentration event below, the resulting step stays within the true KL boundary.
We then bound the gain lost by this adjustment.
For part~\textup{(c)}, let \(\widehat q\) be the observed frequency of
\(h=3\), and set
\begin{equation}
 \omega_U=\min\left\{\frac32,\max\left\{1,\frac{\widehat q}{\pM}
                                      +\frac{\varepsilon}{100}\right\}\right\}.
\end{equation}
The true probability of \(h=3\) is \(\omega\pM\).
Multiplicative Bernoulli concentration with relative tolerance
\(\varepsilon/200\) bounds the failure probability by
\(2\exp(-m\omega\pM\varepsilon^2/120000)\) \citep{Hoeffding1963}.
Under \cref{eq:matched-moment-accuracy-upper}, with probability at least
\(1-\delta\),
\begin{equation}
 \left|\frac{\widehat q}{\pM}-\omega\right|\leq\frac{3\varepsilon}{400},
 \qquad\omega\leq\omega_U\leq\omega+\frac{7\varepsilon}{400}<\omega+\frac{\varepsilon}{50}.
\end{equation}
Return the boundary \(T_{\omega_U}\) of the upper estimate. Since \(\partial_\omega {\mathcal K}_\omega(t)=\pM(e^t-1)^3/Z_\omega(t)\geq0\),
this step is feasible under the true law on the same event.
At a boundary, \(\partial_\omega {\mathcal K}_\omega\leq\pM e^{3T_\omega}/Z_\omega(T_\omega)=q_3/\omega\leq1\).
Together with \cref{eq:matched-moment-gain-lower}, this gives
\(\lvert\partial_\omega T_\omega\rvert\leq100/27\), and hence
\(0\leq T_\omega-T_{\omega_U}\leq2\varepsilon/27\).
Every law on \([0,3]\) has variance at most \(9/4\), as follows by
centering at \(3/2\). Therefore
\begin{equation}
 \Delta_\omega^\star-\Delta_\omega(T_{\omega_U})
 \leq\frac94(T_\omega-T_{\omega_U})\leq\frac{\varepsilon}{6}
 <\varepsilon\Delta_\omega^\star.
\end{equation}
For a rational output, take \(\Nround=\lceil100/\varepsilon\rceil\) and return
\(\lfloor\Nround T_{\omega_U}\rfloor/\Nround\). This is a measurable nonnegative downward
rounding with error at most \(\varepsilon/100\).
It preserves feasibility and adds at most \(9\varepsilon/400\) gain loss.
The total gain loss is less than \(0.19\varepsilon<\varepsilon\Delta_\omega^\star\)
by \cref{eq:matched-moment-gain-lower}, proving the gain guarantee.
\end{proof}

By \cref{eq:accuracy-confidence-order}, the confidence factor in part~\textup{(b)}
is of order $\log(1/\delta)$ uniformly for $0<\delta\leq1/3$.
Together with part~\textup{(c)}, this proves the optimal worst-case sample
complexity in \cref{thm:matched-moment-main}.

\subsection{Calibration bounds and computational reduction}
\label{app:fixed-budget-bounds-proof}

We use event tilts to isolate the cost of scalar calibration when the
natural gradient is already known. We give explicit constants for
$0<p\leq p_0=2^{-10}$.
\Cref{thm:fixed-budget-calibration} supplies the boundary and confidence
bounds used in \cref{thm:accuracy-calibration}. Its separation argument
also supports the proof of \cref{cor:fixed-budget-hardness} below.
Let $P_0$ be a finite base law and $A$ a known event of unknown mass
$p=P_0(A)\in(0,p_0]$. For $t\geq0$, define the event tilt and objective by
\begin{equation}
 P_t(s)=\frac{e^{t\one_A(s)}P_0(s)}{1-p+pe^t},\qquad
 J(t)=q_p(t)=\frac{pe^t}{1-p+pe^t}.
\end{equation}
Its population KL cost, gain and optimal feasible gain are
\begin{equation}
 \begin{aligned}
 \mathcal K_p(t)&=\KL(P_0\Vert P_t)=\log(1-p+pe^t)-pt,\\
 \Delta_p(t)&=q_p(t)-p,\qquad
 \Delta_p^\star=\max_{t\geq0:\,\mathcal K_p(t)\leq1/10}\Delta_p(t).
 \end{aligned}
\end{equation}
Write $F(t)$ for the scalar Fisher information of this tilt.
A procedure receives the family specification, event-membership access and
the exact natural gradient $1$. Independent base draws are its only further
information about $p$. Neither the initial gradient nor scalar Fisher information is supplied.
Write
$\binaryKL(a\Vert b)=a\log(a/b)+(1-a)\log((1-a)/(1-b))$
for Bernoulli KL, with boundary values defined by limits.
Write $\logit(q)=\log(q/(1-q))$ for $0<q<1$.

The required output is a tilt parameter, which matters even for a binary
base law. At KL budget $1/10$, the fixed law $\Bern(0.09)$ is feasible
and 75\%-useful for every $0<p\leq2^{-10}$. Its prescribed tilt,
however, is $t=\logit(0.09)-\logit(p)$ and depends on the unknown mass.
Thus specifying a useful law need not require the information needed to
select its tilt.

\begin{theorem}[Useful-step calibration with an exact natural gradient]
\label{thm:fixed-budget-calibration}
For the event-tilt model above with $0<p\leq p_0=2^{-10}$, the following
hold under the stated sampling access. A feasible step is a scalar
$t\geq0$ satisfying $\mathcal K_p(t)\leq1/10$, and its gain is $\Delta_p(t)$.
\begin{enumerate}[label=\textup{(\alph*)},leftmargin=*]
 \item\textup{Exact natural gradient and attainable improvement.}
 The scalar Fisher information satisfies
 $J'(t)=F(t)=q_p(t)(1-q_p(t))$, so $F(t)^{-1}J'(t)=1$.
 Moreover, $\Delta_p^\star\geq9/100-p$ and
 $\lim_{p\downarrow0}\Delta_p^\star=1-e^{-1/10}$.
 \item\textup{Sample lower bound.}
 For every $0<p\leq p_0/2$, there exist two base laws with the same
 known event $A$ and event probabilities $p$ and $2p$.
 Consider any procedure using at most $m$ draws. Success means returning a
 feasible step attaining at least $3/4$ of optimal gain.
 If the procedure succeeds with probability at least
 $1-\delta$ under each law, $0<\delta\leq1/3$, then
 $m\geq\log(1/(2\delta))/[-\log(1-p)]$.
 At $\delta=1/3$, this gives $m\geq1/(3p)$, regardless of
 computational power.
 \item\textup{Confidence-based calibration.}
 Fix $0<\delta<1$. For any known event $A$, let $\widehat p$
  be its observed frequency in $m\geq1$ draws.
  Use Bernoulli-KL inversion \citep{GarivierCappe2011} to form the interval
  $\mathcal I_m=\{r\in[0,1]:m\,\binaryKL(\widehat p\Vert r)
  \leq\log(2/\delta)\}$.
  Let $\widehat t_m$ be the largest $t\geq0$ satisfying $\mathcal K_r(t)\leq1/10$
  for every $r\in\mathcal I_m$.
 For every fixed $m$, this rule is KL-feasible with probability
 at least $1-\delta$. If $m\geq(2048/p)\log(8/\delta)$, the returned
 step is feasible and attains at least $3/4$ of optimal gain with
 probability at least $1-\delta$ under $P_0$.
\end{enumerate}
\end{theorem}

\begin{proof}[Proof of \cref{thm:fixed-budget-calibration}]
The KL boundary locates the optimal gain and separates useful outputs
under two event masses. For the upper bound, the returned step is feasible whenever the confidence interval contains the true mass.
Controlling the interval's width also gives usefulness on this event.

\noindent\textbf{Boundary geometry.}
Throughout, \(0<p\leq p_0=2^{-10}\).
The identities \(q_p'=q_p(1-q_p)>0\),
\(\mathcal K_p'=q_p-p>0\) for \(t>0\), and
\(\mathcal K_p(t)=(1-p)t+\log p+o(1)\) as \(t\to\infty\), together with
\(\mathcal K_p(0)=0\), give a unique positive boundary \(T(p)\) satisfying
\(\mathcal K_p(T(p))=1/10\). Strict increase of \(q_p\) places the optimum at
this boundary.
Substituting \(t=\log[q(1-p)/(p(1-q))]\) gives
\begin{equation}
 \mathcal K_p(t)=\binaryKL(p\Vert q_p(t)).
\label{eq:fixed-budget-binary-kl}
\end{equation}
The benchmark also optimizes event-probability gain over all distributions
on the base support. The KL chain rule gives
$\KL(P_0\Vert Q)\geq\binaryKL(p\Vert Q(A))$, with equality when $Q$
preserves the conditional laws inside $A$ and its complement. The event
tilt has exactly these conditional laws, so its optimal gain equals the
unrestricted event-probability optimum.

\noindent\textbf{Part (a): local information and attainable gain.}
The score of the finite tilt family is
\(\partial_t\log P_t(s)=\one_A(s)-q_p(t)\).
Its variance gives
\(F(t)=q_p(t)(1-q_p(t))=J'(t)>0\), so the full natural gradient is one.
The candidate \(t_{\mathrm{B,cand}}=\log(1+1/(10p))\) has normalizer \(11/10\), so
\(\mathcal K_p(t_{\mathrm{B,cand}})=\log(11/10)-pt_{\mathrm{B,cand}}\leq1/10\) and
\(q_p(t_{\mathrm{B,cand}})-p=(1-p)/11\geq9/100-p\).
For the limit, the bound \(\mathcal K_p(t)\geq(1-p)t+\log p\) gives
\begin{equation}
 0\leq pT(p)\leq\frac{p/10-p\log p}{1-p}\longrightarrow0.
\end{equation}
At the boundary, the normalizer equals \(e^{1/10+pT(p)}\). Hence
\begin{equation}
 \Delta_p^\star=(1-p)\bigl(1-e^{-1/10-pT(p)}\bigr)
 \longrightarrow1-e^{-1/10}.
\end{equation}
Replacing \(1/10\) by any fixed \(\eta>0\) in this argument gives
limiting gain \(1-e^{-\eta}\).

\noindent\textbf{Part (b): separation and indistinguishable samples.}
We choose a step too short to be useful under the smaller mass and already
infeasible under the larger mass, turning calibration into a two-law test.
Choose \(\tsep(p)=\log(1/(16p))\) to separate the two laws:
the factor \(16\) gives limiting event probability \(1/17<27/400\)
at mass \(p\), but limiting KL \(\log(9/8)>1/10\) at mass \(2p\).
Usefulness requires \(q_p(t)\geq27/400+p/4\), whereas for \(p\leq p_0/2\),
\(q_p(\tsep(p))=1/(17-16p)\leq128/2175<3/50\).
Thus every useful output at \(p\) exceeds \(\tsep(p)\).
For the alternative mass, both terms in the expression below decrease on
\((0,p_0]\), because their derivatives \(-2/(9/8-2p)\) and \(2(\log(16p)+1)\)
are negative there. Evaluating at \(p_0\) therefore bounds the cost for every
\(p\leq p_0/2\):
\begin{equation}
 \mathcal K_{2p}(\tsep(p))=\log(9/8-2p)+2p\log(16p)
 \geq\kappa_{\mathrm{sep}}:=\log(575/512)-\frac{\log64}{512}
 >0.1079225.
 \label{eq:fixed-budget-count-gap}
\end{equation}
The fixed margin above the KL budget will also allow finite-precision
thresholding in the computational reduction below.
Every feasible output at \(2p\) is therefore below
\(\tsep(p)\). Bernoulli base laws on the same event \(\{1\}\) have
the same supplied vector, identically one.
The laws \(\Bern(p)\) and \(\Bern(2p)\) have total variation distance \(p\), and
their useful-step sets are disjoint. \Cref{lem:padding} therefore bounds the
error sum below by \((1-p)^m\), while the two success guarantees bound it above
by \(2\delta\).
Hence \((1-p)^m\leq2\delta\), proving the lower bound.
At \(\delta=1/3\), \(-\log(1-p)\leq p/(1-p)\) gives
\(m\geq(1-p)\log(3/2)/p\geq1/(3p)\) on the stated range.

\noindent\textbf{Part (c): confidence calibration.}
We first show that the true mass belongs to the confidence interval with
the required probability and that the confidence boundary defines a finite
step for every sample outcome.
The exponential Markov inequality, minimized over its parameter, gives
\(\Prob_p(\widehat p\geq x)\leq e^{-m\binaryKL(x\Vert p)}\) for \(x>p\).
Negative parameters give the lower tail \citep{Hoeffding1963}.
Monotonicity on each side and endpoint limits yield
\begin{equation}
 \Prob\{m\binaryKL(\widehat p\Vert p)>\log(2/\delta)\}\leq\delta.
 \label{eq:event-confidence-coverage}
\end{equation}
As a sublevel set of the lower-semicontinuous convex function
\(r\mapsto\binaryKL(\widehat p\Vert r)\), \(\mathcal I_m\) is a compact
interval. Because \(\log(2/\delta)>0\), it contains a mass in \((0,1)\)
even at zero or full counts. The function
\(\mathcal K^{\max}_{\mathcal I_m}(t):=\max_{r\in\mathcal I_m}\mathcal K_r(t)\) is continuous, with
\(\mathcal K^{\max}_{\mathcal I_m}(0)=0\), and diverges as \(t\to\infty\).
For \(t>0\), a maximizing mass lies in \((0,1)\), since
\(\mathcal K_0(t)=\mathcal K_1(t)=0\) and \(\mathcal K_r(t)>0\) for \(0<r<1\).
Evaluating at such a maximizer shows that \(\mathcal K^{\max}_{\mathcal I_m}\) is strictly increasing.
It therefore has a unique finite \(1/10\) boundary, which defines
\(\widehat t_m\) and can be bracketed by doubling and approximated by
bisection. On the coverage event, this step is feasible for \(p\).
For implementation, $r\mapsto\mathcal K_r(t)$ is concave at each fixed
$t>0$. Its unconstrained maximum occurs at $r=1/t-1/(e^t-1)$. Over $\mathcal I_m$, use
this value if it lies in the interval, and otherwise the nearest endpoint. This evaluates the
worst-case KL needed at each step of the scalar boundary search.

To establish usefulness, we bound the largest plausible mass and
exhibit a step feasible for every mass up to that bound.
Put \(\Lambda_8=\log(8/\delta)\geq\log(2/\delta)\), and condition on coverage.
Applying \(-\log x\geq2(1-\sqrt x)\) to the Bernoulli outcomes gives
\begin{equation}
 \binaryKL(a\Vert b)\geq
 (\sqrt a-\sqrt b)^2+(\sqrt{1-a}-\sqrt{1-b})^2.
 \label{eq:event-square-root-kl-bound}
\end{equation}
Endpoint limits preserve this inequality, allowing infinite KL.
For \(r,p\in\mathcal I_m\), applying this bound to
\((\widehat p,r)\) and \((\widehat p,p)\), then using the triangle
inequality, gives
\(\sqrt r\leq\sqrt p+2\sqrt{\Lambda_8/m}\).
We relax \(2\sqrt{\Lambda_8/m}\) to \(\sqrt{8\Lambda_8/m}\).
The sample bound with coefficient \(2048\) ensures
\(\sqrt{8\Lambda_8/(mp)}\leq1/16\), yielding
\begin{equation}
 p\leq {p_{\mathrm U}}:=\max\mathcal I_m
 \leq(\sqrt p+\sqrt{8\Lambda_8/m})^2\leq(17/16)^2p=289p/256.
 \label{eq:fixed-budget-uniform-mass-bound}
\end{equation}
The candidate \(t_{\mathrm{B,conf}}=\log(1+1/(10{p_{\mathrm U}}))\) has
\(\mathcal K_r(t_{\mathrm{B,conf}})\leq\log(1+r/(10{p_{\mathrm U}}))\leq\log(11/10)<1/10\)
for every \(r\leq {p_{\mathrm U}}\). Hence \(\widehat t_m\geq t_{\mathrm{B,conf}}\), and
\begin{equation}
 \Delta_p(\widehat t_m)\geq\frac{p(1-p)}{10{p_{\mathrm U}}+p}
 \geq\frac{1-p_0}{10(17/16)^2+1}>0.08129.
 \label{eq:fixed-budget-reward-lower}
\end{equation}
For binary entropy \(H(p)\), feasible \(q\) satisfies
\((1-p)\log(1/(1-q))=\mathcal K_p+H(p)+p\log q\leq1/10+H(p)\).
Therefore
\begin{equation}
 \Delta_p^\star\leq1-\exp\left(-\frac{1/10+H(p)}{1-p}\right)
 \leq1-\exp\left(-\frac{1/10+H(p_0)}{1-p_0}\right)<0.103.
 \label{eq:fixed-budget-optimum-upper}
\end{equation}
The exponent magnitude increases with \(p\), with derivative
\((1/10-\log p)/(1-p)^2\).
For the final inequality, use
\(H(p_0)\leq p_0\log(1/p_0)+p_0\) and \(\log2<7/10\):
the exponent magnitude is below 0.108, and
\(1-e^{-0.108}\leq0.108-0.108^2/2+0.108^3/6<0.103\).
Since \(0.08129>(3/4)0.103=0.07725\), both conditions for a useful step hold on the same coverage event.
The guaranteed gain exceeds the required fractional gain by more than 0.004,
so \(\Delta_p'\leq1/4\) permits
downward rational rounding by at most \(10^{-3}\) while preserving both
conditions. Rounding can be chosen nonnegative, and monotonicity of
\(\mathcal K_p\) preserves feasibility.
\end{proof}

\noindent\textbf{Encoding and complexity conventions.}
We use the standard binary encoding of Boolean formulas. This gives a
succinct model description because it specifies probabilities without
listing the full state space. Input and output lengths are measured in bits.
NP consists of decision problems whose yes answers have polynomial-length
certificates verifiable in polynomial time. BPP consists of decision problems
solvable in randomized polynomial time with error at most one third on every
input \citep{AroraBarak2009}.

\begin{proof}[Proof of \cref{cor:fixed-budget-hardness}]
We reduce SAT to calibration by constructing an event from the input
formula and thresholding the step returned by a hypothetical calibration
procedure. The threshold depends only on the formula's number of variables.
We establish finite-precision separation and polynomial-length valid outputs.

We first specify the event used in the main-text corollary.
Draw $B_{\mathrm{tag}}\in\{0,1\}^{11}$ and $X\in\{0,1\}^n$
independently and uniformly, and define
\begin{equation}
 A_\varphi=\{B_{\mathrm{tag}}=0^{11},\varphi(X)=1\}
 \cup\{B_{\mathrm{tag}}=1^{11},X=0^n\}.
\end{equation}
The tag bits separate the satisfying assignments from one additional
accepted state. Their number ensures that even $2^n+1$ accepted states
have total mass below $p_0=2^{-10}$, as required by the calibration bounds.
The two blocks are disjoint, so this event has mass \(p=K/\DB\), where
\(K=1+\SATcount(\varphi)\), \(\DB=2^{n+11}\), and \(p<p_0\).
This identity describes the mass. We construct the event and its base
sampler without computing \(K\).
Set \(\tsepB=(n+7)\log2=\log(\DB/16)\).
When \(K=1\), useful outputs have event probability at least
\(27/400+p/4\), while \(q_p(\tsepB)<3/50\).
Since \(q_p'\leq1/4\), the probability gap
\(27/400-3/50=3/400\) gives \(t>\tsepB+4(3/400)=\tsepB+0.03\).

For \(K\geq2\), \(p\in[2/\DB,p_0]\).
Unlike the sampling pair, the satisfiable case allows many masses,
so separation must hold throughout \([2/\DB,p_0]\).
Concavity in the mass, from
\begin{equation}
 \partial_r \mathcal K_r(t)=\frac{e^t-1}{1+r(e^t-1)}-t,\qquad
 \partial_r^2\mathcal K_r(t)=-\frac{(e^t-1)^2}{(1+r(e^t-1))^2}<0,
 \label{eq:fixed-budget-count-concavity}
\end{equation}
reduces the minimum at \(t=\tsepB\) to the endpoints.
At \(2/\DB\), \eqref{eq:fixed-budget-count-gap} gives cost at least \(\kappa_{\mathrm{sep}}\).
At \(p_0\), since \(\DB\geq2^{14}\),
\begin{equation}
 \mathcal K_{p_0}(\tsepB)\geq\mathcal K_{p_0}(\log1024)
 =\log(2047/1024)-\log(1024)/1024>0.685>\kappa_{\mathrm{sep}}.
\end{equation}
Thus any feasible output satisfies
\(\tsepB-t\geq\mathcal K_p(\tsepB)-\mathcal K_p(t)\geq\kappa_{\mathrm{sep}}-1/10>0.0079\),
using \(0\leq\mathcal K_p'\leq1\).
A polynomial-time rational approximation to \(\tsepB\) within \(10^{-3}\)
separates these cases. For example, the convergent expansion
\(\log2=2\sum_{j\geq0}3^{-(2j+1)}/(2j+1)\) needs only
\(\bigO(\log(n+7))\) terms to achieve error below
\(10^{-3}/(n+7)\), using rational arithmetic of polynomial bit length.

Useful steps of polynomial bit length exist. The step
\(t=\log(1+1/(10p))=\bigO(n)\) has cost at most \(\log(11/10)<1/10\) and gain
\((1-p)/11\geq(1-p_0)/11>0.09082>0.07725\).
Downward rounding within \(10^{-3}\) preserves feasibility and usefulness.
The rounded step needs only \(\bigO(\log(n+1))\) integer bits and a
constant number of fractional bits. This existence argument uses \(p\), which is not supplied to the procedure.
The predicate is polynomial-time constructible and uses at most one
formula evaluation, and its base sampler uses \(n+11\) fair bits.
Fresh fair bits simulate every draw the procedure takes, and the supplied
natural gradient is one for every formula. Run a hypothetical calibration procedure once on this input.
Comparing its returned step with the rational threshold decides the original
formula's satisfiability, giving the stated randomized implication
\citep{AroraBarak2009}.
\end{proof}

For the confidence rule in \cref{thm:fixed-budget-calibration}(c),
at \(\delta=0.05\) the sufficient bound is \(m\geq10394/p\),
about \(1.04\times10^9\) draws at \(p=10^{-5}\).

\noindent\textbf{Dependence on divergence orientation.}
\phantomsection\label{app:divergence-orientation}
The positive limiting gain above depends on the orientation of KL.
The following bounds isolate that dependence by examining feasible gain
before asking how many observations calibration would require.
For an event tilt, let $q=P_t(A)\geq p=P_0(A)$.
Because the conditional laws within $A$ and its complement are unchanged,
update-to-base KL is $\binaryKL(q\Vert p)$. The binary entropy bound gives
\begin{equation}
 \binaryKL(q\Vert p)
 \geq q\log(1/p)-\log2.
\end{equation}
At any fixed budget $\eta$, feasibility therefore requires
$q\leq(\eta+\log2)/\log(1/p)$, so the optimal gain tends to zero as
$p\downarrow0$. With $\chi^2(Q\Vert P)=\E_P[(dQ/dP-1)^2]$,
direct summation gives
\begin{equation}
 \chi^2(P_t\Vert P_0)
 =\frac{(q-p)^2}{p(1-p)},
 \qquad q-p\leq\sqrt{\eta p(1-p)}
 \quad\text{when }\chi^2(P_t\Vert P_0)\leq\eta.
\end{equation}
Thus the nonvanishing-gain conclusion in
\cref{thm:fixed-budget-calibration}(a) does not extend to either of these
constraints. These comparisons do not determine their optimal calibration
sample rates.

\subsection{Dependence on the required improvement}
\label{app:accuracy-calibration}

We now quantify how the sampling cost grows as the allowed gain loss
shrinks. The explicit bounds below imply the order statement
in \cref{thm:accuracy-calibration}. Use the event-tilt model and sampling
access defined in \cref{app:fixed-budget-bounds-proof}. At event mass $r$, a step $t\geq0$ is useful if $\mathcal K_r(t)\leq1/10$ and
$\Delta_r(t)\geq(1-\varepsilon)\Delta_r^\star$.
Fix $p_0=2^{-10}$, $0<p\leq p_0/2$, $0<\varepsilon\leq1/4$ and
$0<\delta\leq1/3$. Any procedure that uses at most $m$ draws and returns a
useful step with probability at least $1-\delta$ under both Bernoulli masses
$p$ and $(1+4\varepsilon)p$ requires
\begin{equation}
 m\geq
 \frac{(1-p_0)(1-2\delta)\log((1-\delta)/\delta)}
 {16p\varepsilon^2}.
 \label{eq:accuracy-lower}
\end{equation}
Conversely, apply the confidence rule in \cref{thm:fixed-budget-calibration}(c).
Round its output downward by at most $\varepsilon/32$ to a nonnegative step.
This step is useful at mass $p$ with probability at least $1-\delta$ whenever
\begin{equation}
 m\geq\frac{1024\log(2/\delta)}{p\varepsilon^2}.
 \label{eq:accuracy-upper}
\end{equation}
At fixed confidence, the matching worst-case rate corresponds to an expected
event count $mp=\Theta(\varepsilon^{-2})$, even with the exact natural
gradient supplied.

\begin{proof}[Proof of \cref{thm:accuracy-calibration}]
For the lower bound, a relative change in mass moves the KL boundary
farther than the width of the useful-step interval. For the upper
bound, the same sensitivity estimate converts mass-estimation error into
gain loss, with a separate allowance for rounding.

\noindent\textbf{Width and motion of the useful-step interval.}
Retain \(p_0=2^{-10}\), \(0<p\leq p_0/2\),
\(0<\varepsilon\leq1/4\), and \(0<\delta\leq1/3\) from the theorem.
Let \(q_p^\star>p\) solve \(\binaryKL(p\Vert q_p^\star)=1/10\),
and put \(\alpha=1-\varepsilon\). Inverting the event-probability map
gives the largest feasible step and the first step attaining the required gain:
\begin{equation}
 T(p)=\logit(q_p^\star)-\logit(p),\qquad
 L_\varepsilon(p)=\logit(p+\alpha\Delta_p^\star)-\logit(p).
 \label{eq:accuracy-success-endpoints}
\end{equation}
A step is useful exactly when it lies in $[L_\varepsilon(p),T(p)]$.
The proof of \cref{thm:fixed-budget-calibration} gives
$0.09-p\leq\Delta_p^\star<0.103$ and $q_p^\star<0.104$ for $p\leq p_0$.
At $t=\log(1/p)$,
${\mathcal K}_p(t)=\log(2-p)-p\log(1/p)>\log(3/2)-7/1024>1/10$,
so $T(p)<\log(1/p)$. Since ${\mathcal K}_p'(T(p))=\Delta_p^\star>0$,
implicit differentiation of ${\mathcal K}_p(T(p))=1/10$ yields
\begin{equation}
 -\frac{\dd T}{\dd\log p}
 =\frac1{1-p}-\frac{pT(p)}{\Delta_p^\star},\qquad
 0.923<-\frac{\dd T}{\dd\log p}\leq\frac1{1-p_0}.
 \label{eq:accuracy-boundary-derivative}
\end{equation}
For the lower inequality, $p\log(1/p)$ increases up to $p_0$ and is
smaller than $7/1024$ there, while
$(7/1024)/(0.09-p_0)=700/9116<0.077$.
Integrating the derivative of $\logit$ between the two event-probability targets gives
\begin{equation}
 T(p)-L_\varepsilon(p)
 =\int_\alpha^1
 \frac{\Delta_p^\star\,\dd u}
 {(p+u\Delta_p^\star)(1-p-u\Delta_p^\star)}
 \leq\frac{\varepsilon}{\alpha(1-q_p^\star)}<1.489\varepsilon.
 \label{eq:accuracy-success-width}
\end{equation}
Choose ${\chi_{\varepsilon}}=1+4\varepsilon\leq2$ so that changing the mass by this factor
moves the boundary by more than the width of the useful-step interval. Since
$\log(1+4\varepsilon)\geq2\varepsilon$, \cref{eq:accuracy-boundary-derivative}
implies $T(p)-T({\chi_{\varepsilon}}p)>1.846\varepsilon$.
Consequently $L_\varepsilon(p)>T({\chi_{\varepsilon}}p)$: every useful output at $p$
exceeds every feasible output at ${\chi_{\varepsilon}}p$.

\noindent\textbf{Sample complexity lower bound.}
Thresholding the output between these endpoints selects $p$ with
probability at least $1-\delta$ under $p$ and at most $\delta$ under ${\chi_{\varepsilon}}p$.
By \cref{lem:padding}, the law of this test outcome has KL divergence at most
$m$ times the single-draw value. Monotonicity in the two test probabilities
then gives
\begin{equation}
 m\,\binaryKL(p\Vert {\chi_{\varepsilon}}p)
 \geq(1-2\delta)\log\frac{1-\delta}{\delta}.
 \label{eq:appendix-accuracy-detail-3}
\end{equation}
Using $\log x\leq x-1$ in both terms of binary KL gives
$\binaryKL(a\Vert b)\leq(a-b)^2/[b(1-b)]$ for $0<a,b<1$. Hence
$\binaryKL(p\Vert {\chi_{\varepsilon}}p)\leq({\chi_{\varepsilon}}p-p)^2/[{\chi_{\varepsilon}}p(1-{\chi_{\varepsilon}}p)]
\leq16p\varepsilon^2/(1-p_0)$ proves \cref{eq:accuracy-lower}.
The supplied natural gradient $1$ is identical under both laws.

\noindent\textbf{Confidence calibration and gain loss.}
For the upper bound, fix a sample outcome with $p\in\mathcal I_m$.
Set $\Lambda_2=\log(2/\delta)$ and let ${p_{\mathrm U}}=\max\mathcal I_m$ be the
upper endpoint of the nonempty compact confidence interval, so ${p_{\mathrm U}}\geq p>0$.
The coefficient $1024=(2\cdot16)^2$ makes the relative square-root
error at most $\varepsilon/16$:
\begin{equation}
 \sqrt {p_{\mathrm U}}\leq\sqrt p+2\sqrt{\Lambda_2/m},\qquad
 m\geq\frac{1024\Lambda_2}{p\varepsilon^2}
 \ \Longrightarrow\ {p_{\mathrm U}}\leq(1+\varepsilon/16)^2p
 \leq(1+\varepsilon/4)p<p_0.
 \label{eq:accuracy-confidence-mass}
\end{equation}
For fixed positive $t$, $\partial {\mathcal K}_r(t)/\partial r$ decreases in $r$.
At $r={p_{\mathrm U}}$ and $t=T({p_{\mathrm U}})$ it equals
$-\Delta_{p_{\mathrm U}}^\star T'({p_{\mathrm U}})>0$ by \cref{eq:accuracy-boundary-derivative}.
Hence ${\mathcal K}_r(T({p_{\mathrm U}}))\leq {\mathcal K}_{p_{\mathrm U}}(T({p_{\mathrm U}}))=1/10$ for every $0\leq r\leq {p_{\mathrm U}}$.
The confidence rule therefore chooses $\widehat t_m\geq T({p_{\mathrm U}})$ and, on
the coverage event, $\widehat t_m\leq T(p)$.
Applying \cref{eq:accuracy-boundary-derivative} once more gives
\begin{equation}
 T(p)-T({p_{\mathrm U}})\leq\frac{\log({p_{\mathrm U}}/p)}{1-p_0}
 \leq\frac{\varepsilon}{4(1-p_0)}.
 \label{eq:accuracy-boundary-error}
\end{equation}
Since $0\leq\Delta_p'(t)\leq1/4$ and
$\Delta_p^\star\geq0.09-p_0$, the selected step and downward rounding
by at most $\varepsilon/32$ together incur relative gain loss at most
\begin{equation}
 \varepsilon\left[
 \frac1{16(1-p_0)(0.09-p_0)}+
 \frac1{128(0.09-p_0)}\right]<0.8\varepsilon<\varepsilon.
 \label{eq:accuracy-rounded-reward-loss}
\end{equation}
The first bracketed term accounts for conservative calibration and the
second for rounding.
Round downward to a nonnegative step. This rounding preserves feasibility because
${\mathcal K}_p$ increases on $[0,\infty)$. Coverage has
probability at least $1-\delta$, proving \cref{eq:accuracy-upper}.
This rounding budget must scale with $\varepsilon$ to preserve the
stated accuracy. For $0<\delta\leq1/3$,
$1-2\delta\geq1/3$ and $(1-\delta)^2\geq\delta$, so
\begin{equation}
 \frac16\log\frac1\delta
 \leq(1-2\delta)\log\frac{1-\delta}{\delta}
 \leq\log\frac1\delta.
 \label{eq:accuracy-confidence-order}
\end{equation}
Thus the testing lower bound has the same logarithmic confidence
dependence as the constructive upper bound, proving the uniform order
statement.
\end{proof}

\section{Quadratic rescaling and finite KL}
\label{app:additional-calibration}
\label{app:quadratic-finite-kl}

We separate Fisher matrix estimation error from error in the quadratic approximation
to finite-step KL. The scalar example with exact Fisher information isolates approximation error,
which is distinct from sampling error in the learned-head KL estimates.
Let $F,\widehat F\succeq0$ be the population and estimated Fisher information matrices.

For $v^\top\widehat Fv>0$ and $\eta>0$, the proposed step and its
population quadratic cost satisfy
\begin{equation}
 \Delta\theta_\eta=\sqrt{\frac{2\eta}{v^\top\widehat Fv}}\,v,
 \qquad
 \frac{\Delta\theta_\eta^\top F\Delta\theta_\eta}{2\eta}
 =\frac{v^\top Fv}{v^\top\widehat Fv}.
\end{equation}
For a chosen budget fraction $0<{\chi_{\mathrm{bud}}}\leq1$,
targeting ${\chi_{\mathrm{bud}}}\eta$ under budget $\eta$ multiplies this ratio by ${\chi_{\mathrm{bud}}}$. Scaling both budgets by
${\chi_{\mathrm{bud}}}$ leaves it unchanged.
This identity describes error within the quadratic approximation.

Even with exact scalar Fisher information, a quadratic proposal can violate the
finite-step KL budget. For the binary tilt at $p=1/2048$ and budget $0.1$,
the quadratic step $t=\sqrt{0.2/[p(1-p)]}\approx20.24$ has true KL $12.609$.
These values follow by substituting the step into
\begin{equation}
 \mathcal K_p(t)=\log(1-p+pe^t)-pt.
\end{equation}
A $10\%$ quadratic-budget margin scales the step by $\sqrt{0.9}$ and still
gives KL $11.571$. Thus reserving budget in the quadratic approximation does not ensure finite-KL feasibility in this example. \Cref{sec:matched-local-calibration} gives a four-state example in which all laws have the same quadratic step.
That step is about 470 times longer than their KL-feasible boundary steps.

\section{Direction recovery in logistic models}
\label{app:logistic-proof}

We first prove \cref{thm:logistic-preconditioning} by showing that a coordinate of the
normalized direction separates unsatisfiable from satisfiable formulas.
We then prove \cref{cor:logistic-trust-region,cor:logistic-useful-samples}
by separating the first coordinates of useful finite updates.
The second argument establishes hardness directly from the returned update.
\Cref{app:sat-encoding} describes the anchor state, tag bits and branch bits
used below.

\noindent\textbf{Encoding conventions.}
For a formula $\varphi$ with code $c=\langle\varphi\rangle$, the proofs index
the construction by $c$: $Q_c=Q_\varphi$, $F_c^{\mathrm L}=F_\varphi^{\mathrm L}$,
$v_c^{\mathrm L}=v_\varphi^{\mathrm L}$, $\nu_c^{\mathrm L}=\nu_\varphi^{\mathrm L}$
and $M(c)=\SATcount(\varphi)$.
The count describes the resulting law. We construct its sampler by
evaluating assignments, without computing $M(c)$.
Malformed formula codes have $M(c)=0$, and decision languages reject malformed
outer encodings and $n<3$.

\begin{proof}[Proof of \cref{thm:logistic-preconditioning}]
\noindent\textbf{Sampler construction and local geometry.}
The sampler must implement the covariate law by checking assignments, without
counting them. The anchor state ensures a positive definite Fisher matrix even for
an unsatisfiable formula.
Write $\Eval(c,x)$ for evaluation of the formula encoded by $c$ on
$x\in\{0,1\}^n$, with value zero for malformed formula codes.
Three auxiliary tag bits distinguish the formula-evaluation block, the
anchor state and a separate fixed state. Two additional branch bits select
which covariate value to produce. Together with the assignment bits, these bits give $\DL=2^{n+5}$ equally likely
inputs.
Let \(W=(a,x)\) be uniform on
\(\Omega^{\mathrm L}_n=\{0,1\}^3\times\{0,1\}^n\).
Let \(a_c(a,x)\) be the indicator of
\((a=000,\Eval(c,x)=1)\) or \((a=111,x=0^n)\).
These disjoint cases contain \(K=M(c)+1\) states.
With \(w_0=(001,0^n)\) and an independent uniform \({B_{\mathrm{br}}}\in\{0,1,2,3\}\), set
\begin{equation}
 X_c({B_{\mathrm{br}}},W)=
 \begin{cases}
 e_1,&{B_{\mathrm{br}}}=0,\ a_c(W)=1,\\
 e_2,&{B_{\mathrm{br}}}=1,\ a_c(W)=1,\\
 (1,1)^\mathsf T,&{B_{\mathrm{br}}}=2,\ W=w_0,\\
 0,&\text{otherwise}.
 \end{cases}
 \label{eq:logistic-sampler}
\end{equation}
The branch counts are \(K,K,1,\DL-2K-1\). The last is positive because
\(K\leq2^n+1\) and \(2^{n+1}+3<2^{n+5}=\DL\).
The resulting law is \(Q_c\), sampled using \(n+5\) fair bits and at most one
formula evaluation. At \(\theta=0\), one extra independent fair bit
samples the label.

The score is \((Y-\sigm(\theta^\mathsf TX))X\).
Averaging over model labels gives the Fisher matrix
\(\E_Q[\sigm(\theta^\mathsf TX)(1-\sigm(\theta^\mathsf TX))XX^\mathsf T]\).
At zero this is \(\E_Q[XX^\mathsf T]/4\), the matrix in
\eqref{eq:logistic-fisher-direction}, with eigenvalues
$K/(4\DL)$ and $(K+2)/(4\DL)$.
The query gradient \(e_1/2\) gives
\(v_c^{\mathrm L}=\frac{2\DL}{K(K+2)}(K+1,-1)^\mathsf T\),
whose normalization proves (a). The query-only Fisher matrix
\(e_1e_1^\mathsf T/4\) is singular.

\begin{figure}
\centering
\begin{minipage}[c]{0.46\textwidth}
\centering
\begin{tikzpicture}[x=1cm,y=1cm,font=\small,>=Stealth]
  \path[use as bounding box] (0,-0.25) rectangle (5.7,1.70);

  \node[text=cbblue] at (1.311301,1.48) {$K=1$};
  \node[text=cborange] at (4.192989,1.48) {$K=2$};

  \draw[->,black!55,line width=0.65pt] (0.15,0.94) -- (5.65,0.94);
  \draw[black!75,dashed,line width=0.8pt] (2.90,0.48) -- (2.90,1.30);

  \draw[cbblue,line width=1.3pt] (0.623801,0.94) -- (1.998801,0.94);
  \draw[cbblue,line width=0.9pt] (0.623801,0.81) -- (0.623801,1.07);
  \draw[cbblue,line width=0.9pt] (1.998801,0.81) -- (1.998801,1.07);
  \fill[cbblue] (1.311301,0.94) circle[radius=2.5pt];
  \draw[cborange,line width=1.3pt] (3.505489,0.94) -- (4.880489,0.94);
  \draw[cborange,line width=0.9pt] (3.505489,0.81) -- (3.505489,1.07);
  \draw[cborange,line width=0.9pt] (4.880489,0.81) -- (4.880489,1.07);
  \fill[cborange] (4.192989,0.94) circle[radius=2.5pt];

  \node[anchor=north,text=cbblue] at (1.311301,0.53) {$-1/\sqrt5$};
  \node[anchor=north] at (2.90,0.53) {$-3/8$};
  \node[anchor=north,text=cborange] at (4.192989,0.53) {$-1/\sqrt{10}$};
  \node[anchor=north] at (2.90,0.02) {$(\nu_\varphi^{\mathrm L})_2$};
\end{tikzpicture}
\end{minipage}\hfill
\begin{minipage}[c]{0.50\textwidth}
\caption{Exact second coordinates of normalized directions with error intervals of radius $1/32$.
The threshold $-3/8$ separates unsatisfiable ($K=1$) from satisfiable
($K\geq2$) formulas. The coordinate increases with $K$.}
\label{fig:logistic-direction-threshold}
\end{minipage}
\end{figure}
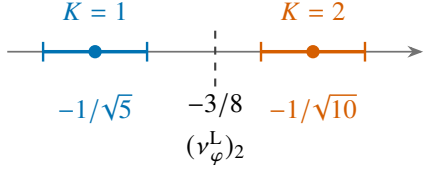

\noindent\textbf{Computational reduction.}
Given a formula, construct the sampler above and run a hypothetical direction-recovery procedure on \((1^n,c)\). We now show how to decide
satisfiability from its returned direction.
The second coordinate of the unit direction is \(-1/\sqrt5\) for \(K=1\)
and at least \(-1/\sqrt{10}\) for \(K\geq2\). Thus
\((\nu_c^{\mathrm L})_2>-3/8\) exactly when \(c\) is satisfiable.
A verifier checks well-formedness and a satisfying assignment.
Renumbering variables and padding to \(n\geq3\) preserves satisfiability,
giving a polynomial many-one reduction and proving (d).
Since \(-1/\sqrt5+1/32<-3/8<-1/\sqrt{10}-1/32\), every
allowed approximate output also decides SAT by a rational comparison,
proving (b) \citep{AroraBarak2009}.
The description has polynomial size, \(v_c^{\mathrm L}\) has
\(\bigO(n)\)-bit rational coordinates, and constant-accuracy rational
unit-vector approximations exist without requiring exact square roots.
The sampler simulates adaptive independent draws with fresh bits, so
a polynomial-time recovery procedure would imply
\(\mathrm{NP}\subseteq\mathrm{BPP}\).

\noindent\textbf{Sample complexity lower bound.}
For (c), withhold codes of a contradiction and a conjunction fixing
all variables. Their \(K=1,2\) laws in \eqref{eq:logistic-model} have
total variation distance \(2/\DL\) per observation, unchanged by the independent
fair label. Thresholding the estimated second coordinate at \(-3/8\) gives a
test with error at most \(1/3\) under each law. By \cref{lem:padding}, its error
sum is at least \(1-2m/\DL\). Hence \(m\geq\DL/6\), even with unrestricted
computation and known \(n\).
\end{proof}

\begin{corollary}[Finite KL-constrained query improvement]
\label{cor:logistic-trust-region}
Fix an integer \(n\geq3\) and a formula code \(c\) on \(n\) variables.
Write \(M(c)\) for its number of satisfying assignments, set
\(K=M(c)+1\) and \(\DL=2^{n+5}\), and let \(e_1,e_2\) be the
standard basis of \(\mathbb R^2\). The covariate law satisfies
\begin{equation}
 Q_c(e_1)=Q_c(e_2)=K/\DL,\qquad
 Q_c(e_1+e_2)=1/\DL,\qquad Q_c(0)=1-(2K+1)/\DL.
\end{equation}
For \(\theta\in\mathbb R^2\), let
\(\pi_\theta(Y=1\mid X)=\sigm(\theta^{\mathsf T}X)\), where
\(\sigm(s)=(1+e^{-s})^{-1}\), and define
\begin{equation}
 \begin{aligned}
 \KLlog_c(\theta)&=\E_{Q_c}\KL(\pi_0(\cdot\mid X)\Vert\pi_\theta(\cdot\mid X)),
 &\eta_{\mathrm L}&=\frac1{100\DL},\\
 \ell_c^*&=\max_{\theta\in\mathbb R^2:\,\KLlog_c(\theta)\leq\eta_{\mathrm L}}\log\sigm(\theta_1).
 \end{aligned}
 \label{eq:logistic-finite-kl-problem}
\end{equation}
The input is \((1^n,c)\), with \(n\) encoded in unary.
Neither \(K\), the population Fisher matrix nor the natural gradient is supplied.
Returning a polynomial-length rational
\(\widehat\theta\in\mathbb Q^2\) with \(\KLlog_c(\widehat\theta)\leq\eta_{\mathrm L}\) and
\(\log\sigm(\widehat\theta_1)\geq\ell_c^*-1/200\) is NP-hard
under polynomial-time Turing reductions.
Unless \(\mathrm{NP}\subseteq\mathrm{BPP}\), no randomized polynomial-time
procedure succeeds on every input with probability at least \(2/3\), even with
adaptive independent exact draws \((X,Y)\) from \(Q_c\pi_0\).
\end{corollary}

\begin{proof}[Proof of \cref{cor:logistic-trust-region}]
We compare the first coordinates attainable in the two cases, then translate
their separation into an objective gap. Multiplying the cost by the common
denominator removes the rare-covariate scale.

\noindent\textbf{Exact KL cost.}
Let \(\phi(t)=\log\cosh(t/2)\). The identity
\(\KL(\Bern(1/2)\Vert\Bern(\sigm(t)))=\phi(t)\) gives
the exact cost
\begin{equation}
 \DL\KLlog_c(\theta)
 =K\phi(\theta_1)+K\phi(\theta_2)+\phi(\theta_1+\theta_2).
 \label{eq:logistic-main-kl}
\end{equation}
Since \(K\geq1\) and \(\phi(t)\geq |t|/2-\log2\), the cost tends to
infinity as \(\lVert\theta\rVert_2\to\infty\). Its nonempty sublevel
set at \(\eta_{\mathrm L}\) is therefore compact, so the continuous objective
attains its maximum. Since \(\phi(0)=\phi'(0)=0\) and
\(\phi''=(1-\tanh^2(t/2))/4\), integration gives
\begin{equation}
 \phi(t)\leq t^2/8\quad(t\in\mathbb R),\qquad
 \phi(t)\geq3t^2/32\quad(|t|\leq1).
 \label{eq:logistic-kl-quadratic-bounds}
\end{equation}
The lower bound uses \(|\tanh(t/2)|\leq|t|/2\).

\noindent\textbf{Separation of useful outputs.}
We exhibit one feasible point for the unsatisfiable case and bound every
feasible first coordinate in the satisfiable case.
At \(K=1\), \((t,-t/2)\) with \(t=23/100\) has scaled cost
at most \(3t^2/16<1/100\).

For \(K\geq2\), a feasible positive \(t=\theta_1\) obeys
\(2\phi(t)\leq1/100\), hence \(t<1\).
For fixed \(t>0\), the function
\(h_t(r)=K\phi(r)+\phi(t+r)\) is strictly convex and coercive, with
\(h_t'(-t)=K\phi'(-t)<0<h_t'(0)=\phi'(t)\).
Its unique minimizer \(r\) therefore lies in \((-t,0)\).
Replacing \(\theta_2\) by \(r\) can only decrease the cost, and all
three arguments \(t,r,t+r\) then lie in \([-1,1]\). Thus
\begin{equation}
 \DL\KLlog_c(\theta)\geq\frac3{32}\{2t^2+2r^2+(t+r)^2\}
 =\frac9{32}(r+t/3)^2+\frac{t^2}4.
 \label{eq:logistic-finite-feasible-upper}
\end{equation}
Thus every feasible positive first coordinate is at most \(1/5\).
Nonpositive coordinates satisfy the same upper bound.
On \([0,23/100]\), \(e^t\leq(1-t)^{-1}\leq100/77<3/2\),
so \((\log\sigm)'(t)>2/5\) and
\begin{equation}
 \log\sigm(23/100)-\log\sigm(43/200)>3/500>1/200.
\end{equation}
A valid output has \(\widehat\theta_1>43/200\) for \(K=1\)
and \(\widehat\theta_1\leq1/5\) for \(K\geq2\).
Thus a hypothetical finite-update procedure would decide satisfiability
by thresholding its returned first coordinate. The reduction need not
evaluate the output's KL cost or objective value.

\noindent\textbf{Polynomial-length rational outputs.}
To obtain polynomial-length rational outputs, first observe that every
feasible coordinate lies in \((-1,1)\).
Each contributes at least \(\phi(\theta_i)\) to the scaled cost, while
\(\phi(1)\geq3/32>1/100\) and \(\phi\) increases with absolute value.
Choose a shrinkage \(\chi_{\mathrm{shrink}}=10^{-3}\) and scale an optimizer by \(1-{\chi_{\mathrm{shrink}}}\).
Convexity and \(\KLlog_c(0)=0\) leave scaled-KL slack at least
\({\chi_{\mathrm{shrink}}}/100=10^{-5}\). Since \(|(\log\sigm)'|\leq1\), the objective loss
is below \({\chi_{\mathrm{shrink}}}\).
Round each coordinate to a rational with a power-of-two denominator within \(1/(10^6\DL)\), requiring
\(\bigO(n)\) fractional bits. Since \(|\phi'|\leq1/2\),
scaled-cost error is at most \((K+1)/(10^6\DL)\leq10^{-6}\).
The rounding cost is below the available slack, and total objective loss is below
\(10^{-3}+10^{-6}<1/200\).
The separating threshold and exact-draw simulation therefore give both
hardness claims with polynomial-length outputs.
\end{proof}

\begin{corollary}[Sample-only finite KL-constrained improvement]
\label{cor:logistic-useful-samples}
Fix an integer \(n\geq3\), set \(\DL=2^{n+5}\), and use the
two covariate laws with \(K=1,2\) in \cref{cor:logistic-trust-region}.
Use the same logistic model, KL cost \(\KLlog_c\), and optimum \(\ell_c^*\)
defined there.
The learner knows \(n\) and the common budget \(\eta_{\mathrm L}=1/(100\DL)\),
but not the codes, the population Fisher matrix or the natural-gradient
direction.
Consider any procedure using at most \(m\) independent draws \((X,Y)\) at
\(\theta=0\) and returning \(\widehat\theta\in\mathbb R^2\).
Suppose its output satisfies
\begin{equation}
 \KLlog_c(\widehat\theta)\leq\eta_{\mathrm L},
 \qquad\log\sigm(\widehat\theta_1)\geq\ell_c^*-1/200
\end{equation}
with probability at least \(2/3\) under each law. Then
\(m\geq\DL/6\), regardless of computational power.
The same bound holds if the objective requirement is replaced by
\begin{equation}
 \log\sigm(\widehat\theta_1)+\log2
 \geq\frac{99}{100}(\ell_c^*+\log2).
\end{equation}
\end{corollary}

\begin{proof}[Proof of \cref{cor:logistic-useful-samples}]
Reuse the preceding output separation to test the two observation laws.
The proof of \cref{cor:logistic-trust-region} gives
\(\widehat\theta_1>43/200\) on success at \(K=1\), whereas feasibility
at \(K=2\) requires \(\widehat\theta_1\leq1/5\).
Thresholding at \(43/200\) therefore tests the two laws with error at most
\(1/3\) under each. The total variation distance for one observation is \(2/\DL\).
By \cref{lem:padding}, the testing error sum is at least \(1-2m/\DL\),
yielding \(m\geq\DL/6\).

For the relative-gain requirement, feasibility implies \(|\theta_1|<1\),
because \(\phi(1)\geq3/32>1/100\).
A sufficiently small positive first-coordinate step is feasible, and
\((\log\sigm)'(t)\leq1/2\) for \(t\geq0\), so
\(0<\ell_c^*+\log2<1/2\).
Attaining \(99/100\) of this gain leaves objective error below
\(1/200\), to which the first claim applies.
\end{proof}

\section{Damped natural-gradient estimation and lower bounds}
\label{app:regularization}

\Cref{thm:damped-fisher-sampling,cor:damped-dimension-lower} establish the
upper and lower bounds summarized in \cref{prop:damped-sampling-main}.
We then quantify the bias introduced by damping and identify when it makes the
rare-event contribution negligible at the requested accuracy.

\subsection{Fisher matrix estimation}
\label{app:damped-fisher-proof}

Differences of independent observations let us estimate covariance without
knowing the population mean.
We measure matrix error relative to the damped population matrix
and apply a concentration inequality. The resulting spectral bounds
control the error in the estimated inverse.

\begin{proposition}[Estimating a damped natural gradient from samples]
\label{thm:damped-fisher-sampling}
Let \(P\) be a probability law on \(\mathcal X\) and let
\(h:\mathcal X\to[0,1]^d\) be a known measurable statistic.
Set \(F=\Cov_P(h)\). Fix a known \(\lambda>0\), parameters
\(0<\rho,\delta<1\), and integers \(d,s\geq1\).
From \(2s\) independent draws from \(P\), form
\(\xi_i=(h(S_{2i})-h(S_{2i-1}))/\sqrt2\),
\(\widehat F=s^{-1}\sum_i\xi_i\xi_i^\mathsf T\),
\({\mathsf A}=F+\lambda I_d\), and \(\widehat {\mathsf A}=\widehat F+\lambda I_d\).
If
\begin{equation}
 s\geq\frac3{\rho^2}\left(1+\frac d{2\lambda}\right)\log\frac{2d}{\delta},
 \label{eq:damped-fisher-sample-bound}
\end{equation}
then with probability at least \(1-\delta\),
\begin{equation}
 \begin{aligned}
 (1-\rho){\mathsf A}&\preceq\widehat {\mathsf A}\preceq(1+\rho){\mathsf A},\\
 \|\widehat {\mathsf A}^{-1}b-{\mathsf A}^{-1}b\|_{\mathsf A}
 &\leq\frac{\rho}{1-\rho}\|{\mathsf A}^{-1}b\|_{\mathsf A}
 \quad\text{for all }b\in\mathbb R^d,
 \end{aligned}
 \label{eq:damped-fisher-solve-bound}
\end{equation}
where \(\|z\|_{\mathsf A}^2=z^\mathsf T{\mathsf A}z\).
For the uncentered moment \(F=\E_P hh^\mathsf T\), the same guarantees
hold for \(\widehat F=m^{-1}\sum_{i=1}^m h(S_i)h(S_i)^\mathsf T\)
from \(m\) independent draws whenever
\(m\geq3\rho^{-2}(1+d/\lambda)\log(2d/\delta)\).
\end{proposition}

This specializes bounds for regularized covariance estimation and inversion
\citep{HsuEtAl2012}.
Taking $\rho=\varepsilon/(1+\varepsilon)$ and counting $m=2s$ individual
draws gives the upper bound in \cref{prop:damped-sampling-main}.

\begin{proof}[Proof of \cref{thm:damped-fisher-sampling}]
\noindent\textbf{Normalized matrix concentration.}
Independence gives \(\E\xi_i\xi_i^\mathsf T=F\), and bounded coordinates
give \(\|\xi_i\|_2^2\leq d/2\).
The independent matrices
\begin{equation}
 \mathsf W_i={\mathsf A}^{-1/2}(\xi_i\xi_i^\mathsf T+\lambda I_d){\mathsf A}^{-1/2}
\end{equation}
have mean \(I_d\) and largest eigenvalue at most
\(R_\lambda=1+d/(2\lambda)\).
Matrix Chernoff \citep{Tropp2012} gives
\begin{equation}
 \Prob\{\|s^{-1}\sum_i\mathsf W_i-I_d\|_2>\rho\}
 \leq2d e^{-s\rho^2/(3R_\lambda)}\leq\delta.
 \label{eq:damped-matrix-chernoff}
\end{equation}

\noindent\textbf{Uniform inverse error.}
On the complement of the event in \cref{eq:damped-matrix-chernoff},
\({\mathsf B}={\mathsf A}^{-1/2}\widehat {\mathsf A} {\mathsf A}^{-1/2}\)
has spectrum in \([1-\rho,1+\rho]\) and
\(\|{\mathsf B}^{-1}-I_d\|_2\leq\rho/(1-\rho)\).
Using
\begin{equation}
 {\mathsf A}^{1/2}(\widehat {\mathsf A}^{-1}b-{\mathsf A}^{-1}b)
 =({\mathsf B}^{-1}-I_d){\mathsf A}^{-1/2}b
\end{equation}
proves both bounds simultaneously, even when \(b\) is chosen from the sample.
The matrix normalization is used only in the proof. The estimator does not
require the population matrix. Both linear systems
use the same right-hand side, so the bound isolates error from estimating the Fisher matrix.
For uncentered moments, \(\|h\|^2\leq d\) gives the replacement constant.
The base-point logistic Fisher information matrix is obtained with \(h=X/2\).
\end{proof}

\begin{proposition}[Scalar lower bound under damping]
\label{prop:damped-scalar-lower}
Fix a known \(0<\lambda\leq1/8\) and \(0<\rho\leq1/16\).
Let \(G_\lambda(p)=[p(1-p)+\lambda]^{-1}\). Consider any sample-only procedure
using at most \(m\) independent Bernoulli draws and returning \(\widehat G\).
Suppose its output satisfies
\(\lvert\widehat G-G_\lambda(p)\rvert\leq\rho G_\lambda(p)\)
with probability at least \(2/3\) for every Bernoulli law with
\(p\in(0,1)\). Then \(m\geq1/(360\lambda\rho^2)\) observations are necessary
in the worst case.
\end{proposition}

\begin{proof}[Proof of \cref{prop:damped-scalar-lower}]
Choose event probabilities on the damping scale to separate accurate
inverse estimates while keeping the observation laws close.
To separate the relative-error intervals, take
\(p_-=\lambda,p_+=(1+8\rho)\lambda\) and \(I_\pm=p_\pm(1-p_\pm)\).
Then \(p_+\leq3/16\), \(I_-+\lambda\leq2\lambda\), and
\(I_+-I_-=8\rho\lambda(1-p_--p_+)\geq(11/2)\rho\lambda\).
Consequently
\begin{equation}
 \frac{G_\lambda(p_-)}{G_\lambda(p_+)}
 \geq1+\frac{11}4\rho>\frac{1+\rho}{1-\rho}.
 \label{eq:damped-scalar-inverse-gap}
\end{equation}
The accuracy intervals are disjoint, so a sufficiently accurate estimate
identifies which law generated the observations.
The bound \(\log x\leq x-1\) gives
\begin{equation}
 \binaryKL(p_-\Vert p_+)
 \leq\frac{(p_+-p_-)^2}{p_+(1-p_+)}
 \leq80\lambda\rho^2.
\end{equation}
By \cref{lem:padding} and Pinsker's inequality
\citep{vanErvenHarremoes2014}, the output laws under the two
masses are within total variation \(\sqrt{40m\lambda\rho^2}\). Testing with both
errors at most \(1/3\) requires total variation at least \(1/3\). Thus \(40m\lambda\rho^2\geq1/9\),
giving \(360=9\cdot40\) in the denominator.
\end{proof}

\begin{corollary}[Dimension dependence under coordinate-wise boundedness]
\label{cor:damped-dimension-lower}
For an integer \(d\geq1\), a fixed known \(0<\lambda\leq d/8\),
and \(0<\varepsilon\leq1/16\),
there exist a known statistic \(h:\{0,1\}\to[0,1]^d\) and a fixed known unit vector \(b\)
with the following property.
Write \(F=\Cov_P(h)\) and \({\mathsf A}=F+\lambda I_d\).
Consider a sample-only procedure returning \(\widehat v\in\mathbb R^d\) from
at most \(m\) independent draws. If its output satisfies
\begin{equation}
 \Prob_P\{\|\widehat v-{\mathsf A}^{-1}b\|_{\mathsf A}
 \leq\varepsilon\|{\mathsf A}^{-1}b\|_{\mathsf A}\}\geq2/3,
 \label{eq:damped-dimension-accuracy}
\end{equation}
for every Bernoulli law \(P\), then
\(m\geq d/(360\lambda\varepsilon^2)\).
\end{corollary}

\begin{proof}[Proof of \cref{cor:damped-dimension-lower}]
Embed the scalar problem along one known direction and project any
accurate vector estimate onto that direction.
For \(Z_{\mathrm{Bern}}\sim\Bern(p)\), set
\(h(Z_{\mathrm{Bern}})=Z_{\mathrm{Bern}}\mathbf1_d\), \(b=\mathbf1_d/\sqrt d\), and
\(a=\lambda+dp(1-p)\). Then
\(F=dp(1-p)bb^\mathsf T\), \({\mathsf A}b=ab\), \({\mathsf A}^{-1}b=b/a\).
Orthogonal decomposition gives
\(\|\widehat v-b/a\|_{\mathsf A}^2\geq a|b^\mathsf T\widehat v-1/a|^2\),
whereas \(\|b/a\|_{\mathsf A}^2=1/a\).
The assumed relative error bound therefore implies
\(|ab^\mathsf T\widehat v-1|\leq\varepsilon\), and
\(db^\mathsf T\widehat v\) estimates
\(d/a=G_{\lambda/d}(p)\) to relative error \(\varepsilon\).
Apply \cref{prop:damped-scalar-lower} at \(\lambda/d\), which proves the
lower bound in \cref{prop:damped-sampling-main}.
Each vector observation contains one Bernoulli draw. The dimension factor comes from the allowed norm \(\sqrt d\) and is already
present in this rank-one construction.
\end{proof}

\noindent\textbf{Bias and perturbation.}
The sampling guarantee above controls error relative to the damped natural gradient.
Comparing with the undamped natural gradient also requires accounting for the bias introduced by damping.
For positive definite \(F\) and \(\lambda>0\), choose an orthonormal eigenbasis and write \(Fu_j=\mu_j^{\mathrm{eig}}u_j\),
\(b_j=u_j^\mathsf Tb\), \(v=F^{-1}b\) and
\(v_\lambda=(F+\lambda I)^{-1}b\), with \(b\ne0\).
Here \(\|z\|_F^2=z^{\mathsf T}Fz\).
Eigenbasis expansion gives the damping bias
\begin{equation}
 \frac{\|v_\lambda-v\|_F^2}{\|v\|_F^2}
 =\frac{\sum_j(b_j^2/\mu_j^{\mathrm{eig}})[\lambda/(\mu_j^{\mathrm{eig}}+\lambda)]^2}
 {\sum_jb_j^2/\mu_j^{\mathrm{eig}}}.
 \label{eq:damping-bias}
\end{equation}
The bias depends on the gradient components in the eigenbasis and on the eigenvalues of \(F\).
For \({\mathsf A}=F+\lambda I\), let \(\widehat v_\lambda=({\mathsf A}+\Delta F)^{-1}b\)
be the direction computed from the perturbed Fisher matrix. If
\({r_{\mathrm{pert}}}=\|{\mathsf A}^{-1}\|_2\|\Delta F\|_2<1\), the identity
\(\widehat v_\lambda-v_\lambda
=-(I+{\mathsf A}^{-1}\Delta F)^{-1}{\mathsf A}^{-1}\Delta Fv_\lambda\)
and the Neumann series yield
\begin{equation}
 \frac{\|\widehat v_\lambda-v_\lambda\|_2}{\|v_\lambda\|_2}
 \leq\frac{\kappa_2({\mathsf A})}{1-{r_{\mathrm{pert}}}}\frac{\|\Delta F\|_2}{\|{\mathsf A}\|_2}.
\end{equation}
Damping reduces sensitivity to matrix error \citep{Martens2020}, although
attaining the required matrix accuracy still depends on sampling variance
and spectral scale. These error bounds for estimated directions do not control
finite-step KL.

\subsection{Resolution under fixed damping}
\label{app:damping-resolution-proof}
When damping dominates the scalar Fisher information of an event indicator,
the reciprocal of the damped information is close
to the known value \(1/\lambda\). The next result identifies when this
approximation attains the required relative accuracy.

\begin{proposition}[Resolution under fixed damping]
\label{prop:damping-resolution}
Let \(P\) be a probability law and \(A\) a known measurable event with
unknown mass \(p=P(A)\in(0,1/2]\). Set \(I=p(1-p)\), fix a known
\(\lambda>0\), and let \(0<\rho\leq1\) and \(0<\delta<1\). Define
\(\mathcal G_{A,\lambda}:=[I+\lambda]^{-1}\). If
\(I\leq\rho\lambda\), then
\(1/[(1+\rho)\lambda]\leq\mathcal G_{A,\lambda}\leq1/\lambda\), and
\(1/\lambda\) has relative error at most \(\rho\). Also,
\(p\leq\rho\lambda\) implies \(I\leq\rho\lambda\), and
\(I\leq\rho\lambda\) implies \(p\leq2\rho\lambda\).
Given \(m\geq1\) independent draws \(S_1,\ldots,S_m\sim P\), set
\(\widehat p=m^{-1}\sum_i\one_A(S_i)\) and
\(\widehat{\mathcal G}_{A,\lambda}
=[\widehat p(1-\widehat p)+\lambda]^{-1}\). This estimate has relative error
at most \(\rho\) with probability at least \(1-\delta\) whenever
\(m\geq10\log(2/\delta)/(\rho^2\lambda)\).
\end{proposition}

\begin{proof}[Proof of \cref{prop:damping-resolution}]
First check the deterministic approximation. For the sample estimate, choose
an error tolerance for the estimated event probability that guarantees the required relative error
after inversion, then apply Bernstein's inequality.
If \(I\leq\rho\lambda\), then
\(\lambda\leq I+\lambda\leq(1+\rho)\lambda\), and the relative error
of \(1/\lambda\) is exactly \(I/\lambda\).
The implications for event mass follow from \(p/2\leq I\leq p\).
Set \({\tau_{\mathrm{dev}}}=\rho(I+\lambda)/(1+\rho)\).
Since \(p(1-p)\) is 1-Lipschitz, on \(|\widehat p-p|\leq {\tau_{\mathrm{dev}}}\)
the reciprocal estimate has relative error at most
\({\tau_{\mathrm{dev}}}/(I+\lambda-{\tau_{\mathrm{dev}}})=\rho\).
Bernstein for bounded centered indicators of variance \(I\)
\citep{Tropp2012} gives
\begin{equation}
 \Prob\{|\widehat p-p|>{\tau_{\mathrm{dev}}}\}
 \leq2e^{-m{\tau_{\mathrm{dev}}}^2/(2I+2{\tau_{\mathrm{dev}}}/3)}
 \leq2e^{-3m\rho^2\lambda/28}.
\end{equation}
The last bound uses
\(\rho(I+\lambda)/2\leq {\tau_{\mathrm{dev}}}\leq(I+\lambda)/2\) and
\(2I+2{\tau_{\mathrm{dev}}}/3\leq7(I+\lambda)/3\).
Since \(10>28/3\), the stated sample count makes the probability at most \(\delta\).
\end{proof}

\section{Population-KL certification for finite candidate sets}
\label{app:practical-calibration}
\label{app:finite-candidate-calibration}

\Cref{sec:sampling} summarizes the following guarantee for affine
classifiers. We prove it below and then extend it to gain comparisons and to
sequential validation of proposed updates.

\begin{proposition}[Population-KL certification]
\label{prop:affine-calibration-main}
Fix a softmax classifier with affine logits $w_a^{\mathsf T}z(x)$ on a finite
class set and an update direction with rows $u_a$. Include any bias coordinate in $z(x)$, and let
$\pi_t(\cdot\mid x)$ have logits
$(w_a+t u_a)^{\mathsf T}z(x)$. Set the per-input cost
$k_t(x)=\KL(\pi_0(\cdot\mid x)\Vert\pi_t(\cdot\mid x))$.
For integers $\Ncand\geq1$ and $m\geq2$, fix candidate steps
$t_0=0,t_1,\ldots,t_{\Ncand}\geq0$.
Let $X_1,\ldots,X_m$ be draws from a finite reference law $Q$,
independent of the classifier, direction and candidates.
Let $\eta>0$, $0<\delta<1$, and suppose known constants
${\mathcal M}_j\geq0$ satisfy $0\leq k_{t_j}(x)\leq{\mathcal M}_j$ on the
support of $Q$. If $\|z(x)\|_2\leq R_{\mathrm{feat}}$ on this support for
known $R_{\mathrm{feat}}\geq0$, one valid choice is ${\mathcal M}_j=\min\{R_j,R_j^2/8\}$, where
$R_j=t_jR_{\mathrm{feat}}\max_{a,b}\|u_a-u_b\|_2$.
For $1\leq j\leq\Ncand$, let $\widehat{\mathcal K}_j$ and $\widehat V_j$ be the
sample mean and unbiased sample variance of $k_{t_j}(X_i)$. Set
\begin{equation}
 {\mathcal U}_j=\min\left\{{\mathcal M}_j,\ \widehat{\mathcal K}_j
 +\sqrt{\frac{2\widehat V_j\log(2\Ncand/\delta)}m}
 +\frac{7{\mathcal M}_j\log(2\Ncand/\delta)}{3(m-1)}\right\},
 \qquad {\mathcal U}_0=0.
 \label{eq:affine-bernstein-main}
\end{equation}
Any measurable choice of $\widehat j$ among candidates with
${\mathcal U}_{\widehat j}\leq\eta$, including one that uses the calibration data, satisfies $\E_Q k_{t_{\widehat j}}\leq\eta$ with probability at least
$1-\delta$. The zero step is always accepted.
\end{proposition}

To certify updates from sampled KL values, we bound the cost of each
candidate separately and require all bounds to hold together. A bounded per-input KL cost permits empirical Bernstein certification of a finite set of candidate updates
\citep{MaurerPontil2009}.
Conditioning on an independent estimation batch fixes the direction and
candidate set before calibration.

Consider a softmax model with affine logits and a reference law \(Q\).
Let \(u_a\) be the direction's row for class \(a\), including any fixed zero reference row.
Suppose the input representation satisfies \(\|z(x)\|_2\leq R_{\mathrm{feat}}\)
on the reference support, with known \(R_{\mathrm{feat}}\geq0\).
At step \(t\geq0\), write
\(d_a=t u_a^{\mathsf T}z(x)\), \(\pi_a=\pi_\theta(a\mid x)\), and
\begin{equation}
 R_{\mathrm{logit}}(t)=tR_{\mathrm{feat}}\max_{a,b}\|u_a-u_b\|_2,\qquad
 {\mathcal M}(t)=\min\{R_{\mathrm{logit}}(t),R_{\mathrm{logit}}(t)^2/8\}.
\end{equation}
The categorical KL from the current model to the updated model obeys
\begin{equation}
 k_t(x)=\log\sum_a\pi_a e^{d_a}-\sum_a\pi_a d_a,
 \qquad 0\leq k_t(x)\leq {\mathcal M}(t).
 \label{eq:categorical-kl-envelope}
\end{equation}
We call \(\mathcal M(t)\) the per-input KL envelope.
This bound on the KL cost at every input is known before sampling. To verify this bound, note that the logit increments have range at most \(R_{\mathrm{logit}}(t)\).
Jensen's inequality proves nonnegativity, while bounding the log-sum by
the largest increment gives \(k_t\leq R_{\mathrm{logit}}(t)\).
For the quadratic bound, the second derivative of
\(s\mapsto\log\sum_a\pi_a e^{s d_a}\) is the tilted variance of \(d_a\).
This variance is at most \(R_{\mathrm{logit}}(t)^2/4\), since each increment's squared distance from
the midpoint of the range is at most that value. Integrating the second
derivative against \(1-s\) on \([0,1]\) gives \(k_t\leq R_{\mathrm{logit}}(t)^2/8\).

\begin{proposition}[Finite-candidate certification]
\label{prop:finite-candidate-calibration}
Fix \(0<\delta<1\), \(\eta>0\) and an integer \(\Ncand\geq1\). Let
\(\mathcal G\) be a \(\sigma\)-algebra with respect to which a current conditional
model \(\pi_\theta\), a direction \(u\), nonnegative steps
\(t_0=0,t_1,\ldots,t_{\Ncand}\) and constants \({\mathcal M}_j\geq0\) are
measurable. Let \(m\geq2\) independent draws from a finite reference law \(Q\) be independent
of \(\mathcal G\), and suppose that
\begin{equation}
 0\leq k_{t_j}(x):=\KL(\pi_\theta(\cdot\mid x)\Vert
 \pi_{\theta+t_j u}(\cdot\mid x))\leq {\mathcal M}_j
\end{equation}
on the support of \(Q\).
For each candidate with a positive bound on its per-input KL cost, let \(\widehat {\mathcal K}_j\) and
\(\widehat V_j\) be the sample mean and unbiased sample variance of
its measurable per-input KL costs, and set
\begin{equation}
 \begin{split}
 r_j&=\sqrt{\frac{2\widehat V_j\log(2\Ncand/\delta)}{m}}
       +\frac{7{\mathcal M}_j\log(2\Ncand/\delta)}{3(m-1)},\\
 {\mathcal U}_j&=\min\{{\mathcal M}_j,\widehat {\mathcal K}_j+r_j\}.
 \end{split}
 \label{eq:finite-candidate-bernstein}
\end{equation}
Set \(r_0={\mathcal U}_0=0\), and \(r_j={\mathcal U}_j=0\) when \({\mathcal M}_j=0\).
Choose measurably from the candidates satisfying \({\mathcal U}_j\leq\eta\),
using any or all calibration observations. The selected candidate has
\(\E_Q k_{t_j}\leq\eta\) with conditional probability at least \(1-\delta\)
given \(\mathcal G\), almost surely, and hence with probability at least
\(1-\delta\). The zero candidate makes this set nonempty.
\end{proposition}
\begin{proof}[Proof of \cref{prop:finite-candidate-calibration}]
A simultaneous upper bound permits selection after observing the calibration
data, provided the candidates are fixed before sampling. The per-input cost
\(k_t(x)\) is continuous in the model parameters, the direction and \(t\), so
every \(k_{t_j}(X_i)\), \({\mathcal U}_j\) and \(\E_Q k_{t_j}\) is a random
variable. Because the draws are independent of \(\mathcal G\), their conditional
law given \(\mathcal G\) is that of \(m\) independent draws from \(Q\), while the
candidates and constants are fixed. We prove the bound under this conditional
law. The failure-probability bound is uniform in the conditioning, so taking
expectations gives the unconditional statement.

For \({\mathcal M}_j>0\), apply the empirical Bernstein bound of \citet{MaurerPontil2009} to
\(k_{t_j}/{\mathcal M}_j\) with failure probability \(\delta/\Ncand\).
Rescaling its mean and unbiased variance gives
\(\E_Q k_{t_j}\leq\widehat {\mathcal K}_j+r_j\).
The deterministic inequality \(\E_Q k_{t_j}\leq {\mathcal M}_j\) justifies the
minimum in \({\mathcal U}_j\). A union bound gives these inequalities for every
candidate on one event of probability at least \(1-\delta\).
Candidates whose bounds are zero have zero KL. On the simultaneous event,
every accepted candidate is feasible, regardless of how it is selected.
\end{proof}

\begin{proof}[Proof of \cref{prop:affine-calibration-main}]
The draws are independent of the classifier, direction and candidates, so
let $\mathcal G$ be the $\sigma$-algebra these generate. Since
$R_j=R_{\mathrm{logit}}(t_j)$, \cref{eq:categorical-kl-envelope} shows that
${\mathcal M}_j=\min\{R_j,R_j^2/8\}$ is a valid bound. With these inputs,
\cref{eq:finite-candidate-bernstein} coincides with
\cref{eq:affine-bernstein-main}, and
\cref{prop:finite-candidate-calibration} gives the claim.
\end{proof}

Under the assumptions of \cref{prop:finite-candidate-calibration}, suppose
the query gain \(G_j\) is known exactly. The rule may maximize it over
accepted candidates. A comparison with the finite-grid optimum follows by
replacing \(\log(2\Ncand/\delta)\) by \(\log(4\Ncand/\delta)\) in \(r_j\).
Applying the same bound to \(k_{t_j}\) and \({\mathcal M}_j-k_{t_j}\) gives
\(|\widehat {\mathcal K}_j-{\mathcal K}_j|\leq r_j\) simultaneously with
probability at least \(1-\delta\), where
\({\mathcal K}_j=\E_Q k_{t_j}\). Thus every candidate with
\({\mathcal K}_j+2r_j\leq\eta\) is accepted, and the selected gain is at least
the best gain among those candidates.

\noindent\textbf{Independent validation of proposed updates.}
Fix an overall failure probability $0<\delta<1$. The same bound can
validate a proposal chosen by an empirical margin rule.
At validation attempt $r\geq1$, construct the direction and one proposed step
from the training history and any additional data reserved for constructing the proposal. Before validation, fix a known upper bound on the proposal's per-input KL cost
and a sample size of at least two.

Then draw a fresh independent batch from the finite reference law. Apply
\cref{eq:finite-candidate-bernstein} with $\Ncand=1$ and
$\delta_r=\delta/[r(r+1)]$. Accept the proposal if its upper bound is at most the current KL budget, and
otherwise take the zero step. Every new or
revised proposal uses the next index $r$ and a fresh independent batch,
including after a rejection.

Let $\mathcal G_r$ be the $\sigma$-algebra generated by everything observed before
batch $r$, including whether attempt $r$ is made. The proposal, its KL bound, the sample size and
the decision to attempt are $\mathcal G_r$-measurable. Batch $r$, conditional
on $\mathcal G_r$, consists of independent draws from the reference law. Let $B_r$ be the event that attempt $r$ is made and
accepts an update that violates its population-KL budget.
\Cref{prop:finite-candidate-calibration} with $\mathcal G=\mathcal G_r$ gives
$\Prob(B_r\mid\mathcal G_r)\leq\delta_r$ almost surely, so
$\Prob(B_r)\leq\delta_r$ by the tower property. A union bound over $r$ gives
\begin{equation}
 \Prob\{\text{some accepted update violates its population-KL budget}\}
 \leq\sum_{r=1}^\infty\frac{\delta}{r(r+1)}=\delta.
\end{equation}
Because the decision to attempt is $\mathcal G_r$-measurable, this also permits
stopping after any completed attempt. Each validation
batch counts toward the total sampling budget. Validating one proposal at a time removes the multiple-candidate correction, but
it keeps the range term and uses a smaller failure probability at later updates.
The sequential guarantee controls KL feasibility. Objective gain is a separate criterion.

\section{Reproducibility details}
\label{app:provenance}

\noindent\textbf{Reporting conventions.}
Pointwise percentile bootstrap intervals \citep{DavisonHinkley1997}
resample queries for learned-head studies, preserving paired method comparisons.
In the logistic direction and scalar damping studies, intervals for means use
replicate standard errors with a normal approximation.
Intervals for success rates use Wilson's method \citep{Wilson1927}.
The grids, endpoints, seeds and confidence rules in
\cref{app:logistic-calibration,app:binary-calibration} were fixed before evaluation.

\noindent\textbf{Computational environments.}
Learned-head experiments used Python~3.12 with PyTorch~2.8 and CUDA~12.6,
torchvision~0.23, Transformers~4.57, NumPy~2.3 and SciPy~1.16 on NVIDIA
RTX~3090 and RTX~A6000 GPUs. The logistic, damping and count studies used
Python~3.14 with NumPy~2.4 and SciPy~1.18 on an Apple M2 CPU.
The text head uses the checkpoint
\path{distilbert/distilbert-base-uncased-finetuned-sst-2-english}
at revision \path{714eb0fa89d2f80546fda750413ed43d93601a13} and the
\texttt{sst2} split of \path{nyu-mll/glue} at revision
\path{bcdcba79d07bc864c1c254ccfcedcce55bcc9a8c}.

\noindent\textbf{Randomization.}
Every study uses fixed seeds. Each learned-head draw comes from a NumPy seed
sequence keyed by the study seed, model, query, repetition and sampling
stream, and samples at smaller budgets are prefixes of those at larger
budgets.

\noindent\textbf{Numerical solvers.}
Learned-head Fisher matrix and KL computations use float64, and conjugate
gradients must reach relative residual at most \(10^{-8}\) within 2,000
iterations. Scalar searches include the zero step, expand a bracket and bisect
the KL boundary, so the query gain may peak before the boundary. Feasibility
uses tolerance \(10^{-7}\eta+10^{-12}\), and margin-experiment gains use
\(10^{-7}\) relative plus \(10^{-10}\) absolute tolerance.

\section{Count-based calibration experiments}
\label{app:experimental-details}

We vary rarity and gain requirements in the four-state family and the event-tilt construction to examine the sample scales in the calibration bounds.
Simulated counts represent independent observations. The methods receive the
counts and the information specified below.
Implementation and reporting conventions appear in \cref{app:provenance}.

\subsection{Calibration in the four-state family}
\label{app:matched-moment-study}
\begin{table}[htb]
\caption{Joint success over 2,048 paired count simulations per setting.
(a) Four-state law, $\omega=1.25$, $\pM=2^{-30}$, $\varepsilon=1/8$.
(b) Event-tilt confidence rule, 95\% event-mass confidence set, $\varepsilon=1/32$, $mp\varepsilon^2=16$.}
\label{tab:controlled-experiments}
\label{tab:calibration-sampling-main}
\begin{center}
\begin{tabular*}{\linewidth}{@{\extracolsep{\fill}}rrrrr}
\toprule
\multicolumn{3}{c}{\textbf{(a) Four-state calibration}} &
\multicolumn{2}{c}{\textbf{(b) Binary calibration}} \\
\cmidrule(lr){1-3}\cmidrule(lr){4-5}
$m\pM\varepsilon^2$ & Plug-in (\%) & Conservative (\%) & $p$ & Confidence (\%) \\
\midrule
$2^4$ & 48.97 & 52.78 & $10^{-6}$ & 95.80 \\
$2^{12}$ & 48.05 & 93.90 & $10^{-5}$ & 96.39 \\
$2^{20}$ & 50.29 & 100.00 & $10^{-4}$ & 96.29 \\
 & & & $2^{-11}$ & 97.07 \\
\bottomrule
\end{tabular*}
\end{center}
\end{table}

\begin{figure}[htbp]
  \centering
  \includegraphics[width=\linewidth]{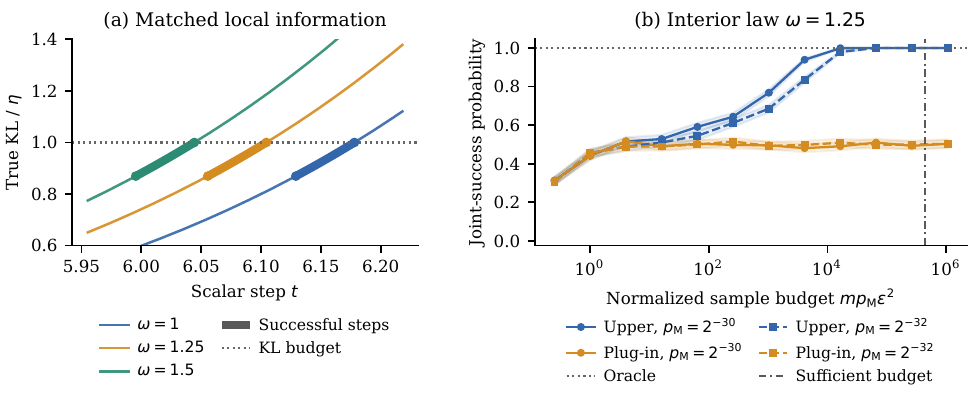}
  \caption{Calibration in the four-state family, whose laws share their initial gradients and scalar Fisher information.
  (a) True finite KL at $\pM=2^{-30}$. Thick segments are feasible and attain
  at least $7/8$ of optimal gain. The $\omega=1$ and $\omega=1.5$ segments are disjoint.
  (b) Paired rules based on counts of $h=3$ at $\omega=1.25$, $\varepsilon=1/8$, with
  2,048 binomial simulations per point and pointwise 95\% Wilson intervals \citep{Wilson1927}.
  The conservative rule, labeled Upper, rounds steps downward. The dash-dot line marks the sample budget sufficient for its theorem guarantee,
  which does not apply at smaller budgets. \Cref{app:matched-moment-study} gives the protocol and ranges over the full grid.}
  \label{fig:matched-moment-study}
\end{figure}

We compare conservative and plug-in calibration on the four-state law in
\cref{eq:matched-moment-law}. Supplying $\pM$, the initial gradient, scalar
Fisher information and natural gradient isolates the remaining calibration problem.
Write $T_\omega$ for the unique boundary with ${\mathcal K}_\omega(T_\omega)=1/10$, and denote
this KL budget by $\eta$.
The parameter grid is
\begin{equation}
 \pM\in\{2^{-30},2^{-32}\},\qquad
 \omega\in\{1,1.25,1.5\},\qquad
 \varepsilon\in\{1/8,1/16\}.
\end{equation}
We use $\delta=0.05$, KL budget $1/10$, and $m\pM\varepsilon^2=2^k$ for
$k=-2,0,2,\ldots,20$. Each setting has 2,048 repetitions.
We generate $n_3\sim\operatorname{Binomial}(m,\omega\pM)$ directly. This
yields the count of the state $h=3$ without generating the $m$ base draws. The
rules use only this count, although it is not a sufficient statistic for the
full family, and each trial is charged $m$ observations, up to $m=2^{60}$.
Both sampled rules use the same count in each repetition. The conservative rule from the theorem uses
\begin{equation}
 \omega_U=\min\{3/2,\max\{1,n_3/(m\pM)+\varepsilon/100\}\}
\end{equation}
and returns $\lfloor\Nround T_{\omega_U}\rfloor/\Nround$, where
$\Nround=\lceil100/\varepsilon\rceil$.
The plug-in rule returns $T_{\widehat\omega}$, where $\widehat\omega$ truncates
$n_3/(m\pM)$ to $[1,3/2]$, and the oracle returns $T_\omega$.
True finite KL and gain are evaluated under the actual law.
The plotted joint-success indicator allows $10^{-12}$ absolute tolerance
for KL feasibility and for the gain ratio. Pointwise 95\% Wilson
intervals quantify Monte Carlo uncertainty.

\begin{table}[htbp]
\caption{All normalized budgets in the four-state study.
Each success range spans the 12 settings of $(\pM,\omega,\varepsilon)$:
two scale values, three law coefficients and two gain targets.
Upper denotes the conservative rule. The last two columns give the largest
rates of true KL exceeding $1.1\eta$ across these settings.
Entries are observed percentages over 2,048 repetitions per setting.}
\label{tab:matched-moment-grid-ranges}
\begin{center}
\begin{tabular}{rcccc}
\toprule
$m\pM\varepsilon^2$ & Upper success & Plug-in success & Upper max. rate & Plug-in max. rate\\
\midrule
$2^{-2}$ & 24.02--73.19 & 24.02--73.19 & 40.09 & 40.09 \\
$2^{0}$ & 40.87--89.31 & 40.87--89.50 & 19.97 & 19.97 \\
$2^{2}$ & 47.61--99.17 & 46.78--99.32 & 3.81 & 4.25 \\
$2^{4}$ & 50.78--100.00 & 48.93--100.00 & 0.00 & 0.05 \\
$2^{6}$ & 53.12--100.00 & 49.95--100.00 & 0.00 & 0.00 \\
$2^{8}$ & 55.71--100.00 & 48.10--100.00 & 0.00 & 0.00 \\
$2^{10}$ & 60.74--100.00 & 47.80--100.00 & 0.00 & 0.00 \\
$2^{12}$ & 72.80--100.00 & 48.05--100.00 & 0.00 & 0.00 \\
$2^{14}$ & 85.79--100.00 & 48.78--100.00 & 0.00 & 0.00 \\
$2^{16}$ & 99.12--100.00 & 48.29--100.00 & 0.00 & 0.00 \\
$2^{18}$ & 100.00--100.00 & 49.02--100.00 & 0.00 & 0.00 \\
$2^{20}$ & 100.00--100.00 & 49.32--100.00 & 0.00 & 0.00 \\
\bottomrule
\end{tabular}
\end{center}
\end{table}

\begin{table}[htbp]
\caption{Paired offset/rounding ablation at $\omega=1.25$, $\varepsilon=1/8$,
reusing the 2,048 counts per row. P is plug-in, O adds the upward offset,
and R rounds the step downward. The four central columns give joint-success
percentages. The last two give the mean of $\max\{{\mathcal K}/\eta-1,0\}$ over all
trials, as a percentage of the KL budget.}
\label{tab:matched-rounding-ablation}
\begin{center}
\begin{tabular}{rrrrrrrr}
\toprule
$\pM$ & $m\pM\varepsilon^2$ & P & P+R & P+O & P+O+R & P excess & P+O+R excess \\
\midrule
$2^{-30}$ & $2^{0}$ & 44.14 & 44.14 & 44.14 & 44.14 & 4.603 & 4.457 \\
$2^{-30}$ & $2^{4}$ & 48.97 & 51.56 & 50.39 & 52.78 & 1.071 & 0.935 \\
$2^{-30}$ & $2^{8}$ & 49.76 & 59.47 & 55.37 & 64.36 & 0.263 & 0.152 \\
$2^{-30}$ & $2^{12}$ & 48.05 & 83.01 & 70.61 & 93.90 & 0.068 & 0.010 \\
$2^{-30}$ & $2^{16}$ & 50.93 & 100.00 & 99.02 & 100.00 & 0.016 & 0.000 \\
$2^{-30}$ & $2^{20}$ & 50.29 & 100.00 & 100.00 & 100.00 & 0.004 & 0.000 \\
$2^{-32}$ & $2^{0}$ & 45.51 & 45.51 & 45.51 & 45.51 & 4.511 & 4.366 \\
$2^{-32}$ & $2^{4}$ & 48.93 & 48.93 & 50.00 & 50.98 & 1.097 & 0.962 \\
$2^{-32}$ & $2^{8}$ & 51.46 & 55.47 & 56.54 & 61.13 & 0.260 & 0.153 \\
$2^{-32}$ & $2^{12}$ & 49.71 & 65.82 & 70.56 & 83.59 & 0.068 & 0.012 \\
$2^{-32}$ & $2^{16}$ & 49.80 & 95.12 & 98.58 & 100.00 & 0.017 & 0.000 \\
$2^{-32}$ & $2^{20}$ & 50.29 & 100.00 & 100.00 & 100.00 & 0.004 & 0.000 \\
\bottomrule
\end{tabular}
\end{center}
\end{table}

\Cref{fig:matched-moment-study} shows the interior setting
$\omega=1.25$, $\varepsilon=1/8$, fixed in advance.
\Cref{tab:matched-moment-grid-ranges} summarizes every grid budget.
At $m\pM\varepsilon^2=2^{20}$, the conservative rule succeeds in $2048/2048$ trials at every
family setting (pointwise Wilson lower endpoint $99.81\%$). The plug-in rule ranges from $49.32\%$ to $51.71\%$ for $\omega\in\{1.25,1.5\}$ and
attains $100\%$ at $\omega=1$. The oracle succeeds throughout.
At smaller budgets, the conservative rule can fail in a substantial fraction of trials.
Its uniform $95\%$ theorem guarantee applies only when
$m\pM\varepsilon^2\geq120000\log40\approx442666$, which only the last
grid budget exceeds.

\noindent\textbf{Separating the offset from rounding.}
A post hoc ablation reuses the same counts in a two-by-two comparison.
One factor adds or omits the positive offset $\varepsilon/100$ to the estimate of $\omega$.
The other includes or omits downward rounding of the step with $\Nround=\lceil100/\varepsilon\rceil$.
Clipping, boundary search and evaluation remain fixed.
\Cref{tab:matched-rounding-ablation} reports this interior setting.
At normalized budget $2^{12}$, every rule meets the gain target in every
trial for both masses, so the success differences measure feasibility.
Both modifications contribute, and the different rounded rates show
sensitivity to the location of the true boundary within the rounding grid.
At smaller budgets, rounding can instead reduce joint success by moving a
step below the gain threshold. The largest observed decrease is 1.86
percentage points across the full grid.
At the largest budget, the offset alone and the full rule succeed in every
trial across all 12 law/accuracy settings, whereas rounding alone ranges from
77.44\% to 100\%.
For the two interior settings in the table, the unmodified
plug-in succeeds in 50.29\% of trials, while its mean positive KL excess is
only about 0.004\% of the budget. Every method meets the gain target there,
and none exceeds $1.1\eta$. Strict boundary crossings can thus persist as their size shrinks.

\subsection{Event tilts: rarity and required gain}
\label{app:binary-calibration}

We vary event rarity and the allowed fractional gain loss $\varepsilon$
to separate their contributions to sample cost. The grid uses
\begin{equation}
 p\in\{10^{-6},10^{-5},10^{-4},2^{-11}\},\qquad
 \varepsilon\in\{1/4,1/8,1/16,1/32\}.
\end{equation}
Normalized sample budgets \(mp\varepsilon^2\) double from \(1/4\) to \(64\),
with 2,048 independent simulations of the cell counts per setting.
The evaluator uses \(p\) to set \(m\), while the rules receive only counts
and \(m\). At KL budget 0.1, we compare empirical 90\%-budget calibration,
four-cell confidence calibration and the direct event-count rule.
Both confidence rules use 95\% coverage. In the four-cell control, the event
and its complement each split into two equal cells, and the confidence set in
\eqref{eq:coupled-cell-confidence} for the four cell probabilities yields an
interval for the combined probability of the two event cells. Each confidence
rule maximizes KL over its entire event-probability interval, including any
interior maximizer. All rules restrict the tilt parameter to at most 64 and
use relative feasibility tolerance \(10^{-7}\).

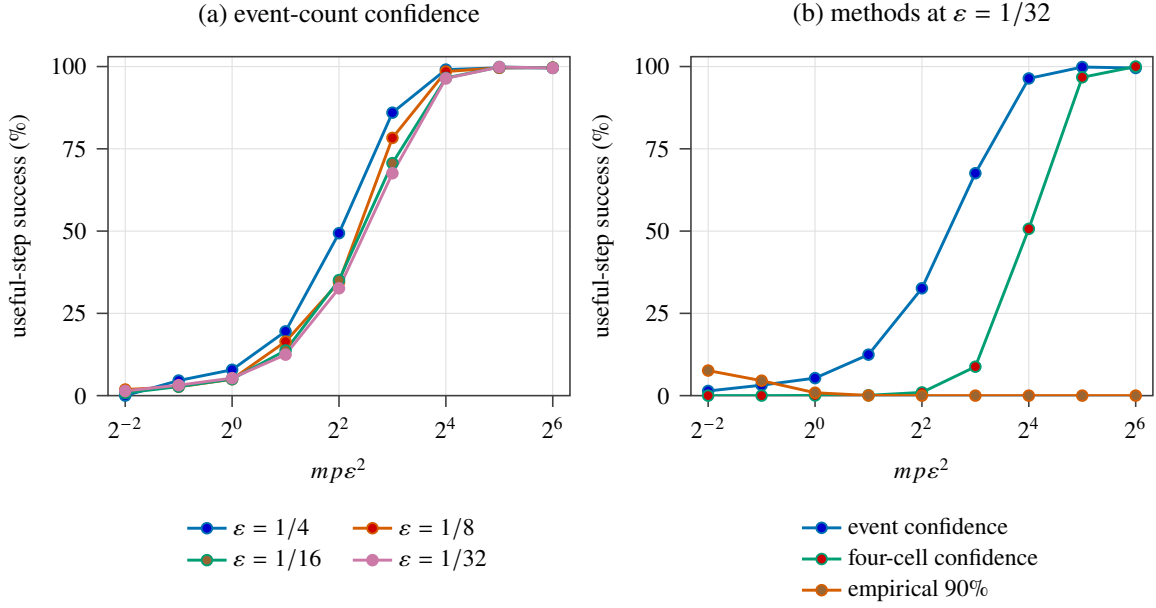
\begin{figure}[htbp]
\centering
\begin{tikzpicture}
\begin{groupplot}[paper pair,group style={group size=2 by 1,horizontal sep=\paperpairsep},xmode=log,log basis x=2,xmin=.2,xmax=80,xtick={.25,1,4,16,64},xlabel={$mp\varepsilon^2$},ymin=0,ymax=103,ytick={0,25,50,75,100}]
\nextgroupplot[title={(a) event-count confidence},ylabel={useful-step success (\%)},legend to name=binarycalibrationleft,legend columns=2]
\addplot+[cbblue,mark=*] table[col sep=comma,x=scaled,y=event0]{\binarycalibrationtable};
\addlegendentry{$\varepsilon=1/4$}
\addplot+[cborange,mark=*] table[col sep=comma,x=scaled,y=event1]{\binarycalibrationtable};
\addlegendentry{$\varepsilon=1/8$}
\addplot+[cbgreen,mark=*] table[col sep=comma,x=scaled,y=event2]{\binarycalibrationtable};
\addlegendentry{$\varepsilon=1/16$}
\addplot+[cbpurple,mark=*] table[col sep=comma,x=scaled,y=event3]{\binarycalibrationtable};
\addlegendentry{$\varepsilon=1/32$}
\nextgroupplot[title={(b) methods at $\varepsilon=1/32$},ylabel={useful-step success (\%)},legend to name=binarycalibrationright,legend columns=1]
\addplot+[cbblue,mark=*] table[col sep=comma,x=scaled,y=event3]{\binarycalibrationtable};
\addlegendentry{event confidence}
\addplot+[cbgreen,mark=*] table[col sep=comma,x=scaled,y=cell3]{\binarycalibrationtable};
\addlegendentry{four-cell confidence}
\addplot+[cborange,mark=*] table[col sep=comma,x=scaled,y=empirical3]{\binarycalibrationtable};
\addlegendentry{empirical 90\%}
\end{groupplot}
\coordinate(binarycalibrationlegendtop) at (current bounding box.south);
\node[anchor=north,inner sep=0pt,yshift=-5pt]
  at (group c1r1.center |- binarycalibrationlegendtop) {\ref*{binarycalibrationleft}};
\node[anchor=north,inner sep=0pt,yshift=-5pt]
  at (group c2r1.center |- binarycalibrationlegendtop) {\ref*{binarycalibrationright}};
\end{tikzpicture}
\caption{Dependence on the required improvement at $p=10^{-5}$, $\eta=0.1$. The two confidence-based rules use 95\% confidence. Each point summarizes 2,048 simulated count vectors. Usefulness requires feasibility and a fraction $1-\varepsilon$ of the reference optimum. At a fixed horizontal coordinate, the number of observations varies with $\varepsilon$. Empirical calibration at 90\% of the budget can limit gain at stringent targets.}
\label{fig:binary-accuracy}
\end{figure}
At $p=10^{-5}$ and $mp\varepsilon^2=16$, the event-count confidence rule
achieves 96.4--99.1\% joint success over $\varepsilon\in\{1/4,1/8,1/16,1/32\}$
(\cref{fig:binary-accuracy}), with expected event counts from 256 to 16,384.
A smaller gain tolerance therefore requires more expected rare events at the
same normalized sample budget.

\begin{figure}[htbp]
\centering
\begin{tikzpicture}
\begin{groupplot}[paper pair,group style={group size=2 by 1,horizontal sep=\paperpairsep},xmode=log,ymin=0,ymax=103,ytick={0,25,50,75,100}]
\nextgroupplot[title={(a) normalized observation budget},log basis x=2,xmin=.2,xmax=80,xtick={.25,1,4,16,64},xlabel={$mp\varepsilon^2$},ylabel={useful-step success (\%)},legend to name=binaryraritieslegend,legend columns=4]
\addplot+[cbblue,mark=*,error bars/.cd,y dir=both,y explicit] table[col sep=comma,x=scaled,y=p0_success,y error plus expr=\thisrow{p0_high}-\thisrow{p0_success},y error minus expr=\thisrow{p0_success}-\thisrow{p0_low}]{\binaryraritiestable};
\addlegendentry{$p=10^{-6}$}
\addplot+[cborange,dashed,mark=square*,error bars/.cd,y dir=both,y explicit] table[col sep=comma,x=scaled,y=p1_success,y error plus expr=\thisrow{p1_high}-\thisrow{p1_success},y error minus expr=\thisrow{p1_success}-\thisrow{p1_low}]{\binaryraritiestable};
\addlegendentry{$p=10^{-5}$}
\addplot+[cbgreen,dashdotted,mark=triangle*,error bars/.cd,y dir=both,y explicit] table[col sep=comma,x=scaled,y=p2_success,y error plus expr=\thisrow{p2_high}-\thisrow{p2_success},y error minus expr=\thisrow{p2_success}-\thisrow{p2_low}]{\binaryraritiestable};
\addlegendentry{$p=10^{-4}$}
\addplot+[cbpurple,densely dotted,mark=diamond*,error bars/.cd,y dir=both,y explicit] table[col sep=comma,x=scaled,y=p3_success,y error plus expr=\thisrow{p3_high}-\thisrow{p3_success},y error minus expr=\thisrow{p3_success}-\thisrow{p3_low}]{\binaryraritiestable};
\addlegendentry{$p=2^{-11}$}
\nextgroupplot[title={(b) absolute observation budget},log basis x=10,xmin=3e5,xmax=1e11,xtick={1e6,1e8,1e10},xlabel={observation budget $m$},ylabel={useful-step success (\%)}]
\addplot+[cbblue,mark=*,error bars/.cd,y dir=both,y explicit] table[col sep=comma,x=p0_m,y=p0_success,y error plus expr=\thisrow{p0_high}-\thisrow{p0_success},y error minus expr=\thisrow{p0_success}-\thisrow{p0_low}]{\binaryraritiestable};
\addplot+[cborange,dashed,mark=square*,error bars/.cd,y dir=both,y explicit] table[col sep=comma,x=p1_m,y=p1_success,y error plus expr=\thisrow{p1_high}-\thisrow{p1_success},y error minus expr=\thisrow{p1_success}-\thisrow{p1_low}]{\binaryraritiestable};
\addplot+[cbgreen,dashdotted,mark=triangle*,error bars/.cd,y dir=both,y explicit] table[col sep=comma,x=p2_m,y=p2_success,y error plus expr=\thisrow{p2_high}-\thisrow{p2_success},y error minus expr=\thisrow{p2_success}-\thisrow{p2_low}]{\binaryraritiestable};
\addplot+[cbpurple,densely dotted,mark=diamond*,error bars/.cd,y dir=both,y explicit] table[col sep=comma,x=p3_m,y=p3_success,y error plus expr=\thisrow{p3_high}-\thisrow{p3_success},y error minus expr=\thisrow{p3_success}-\thisrow{p3_low}]{\binaryraritiestable};
\end{groupplot}
\end{tikzpicture}
\plotlegend{binaryraritieslegend}
\caption{\textbf{Calibration cost across event rarities.}
The event-count rule uses 95\% confidence for event mass,
$\varepsilon=1/32$ and KL budget $\eta=0.1$.
Success requires KL feasibility and at least $31/32$ of optimal event gain.
Both panels show the same 2,048-repetition estimates with pointwise
95\% Wilson intervals \citep{Wilson1927}. Multinomial counts simulate $m$ independent observations.}
\label{fig:binary-rarities}
\end{figure}
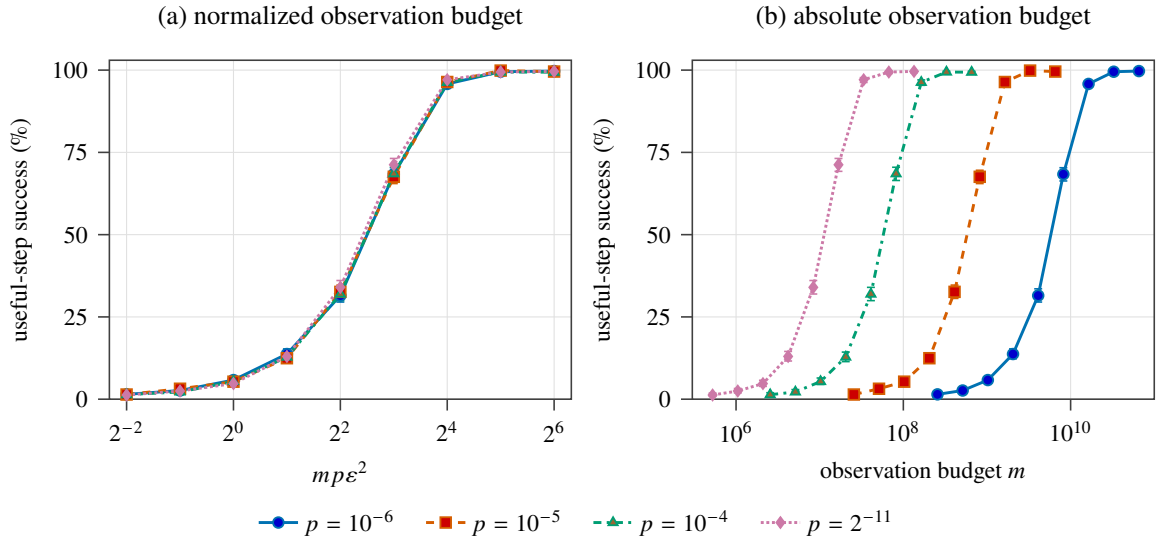

\Cref{fig:binary-rarities} compares the four event masses on normalized and
absolute sample scales. At $mp\varepsilon^2=16$ and $\varepsilon=1/32$, the
event-count rule succeeds in 95.8\% to 97.1\% of trials across rarities, at
sample sizes from \(3.36\times10^7\) to \(1.64\times10^{10}\). High success
rates already appear below the explicit sufficient threshold in
\cref{eq:accuracy-upper}.

\section{Logistic controls and damped estimation}
\label{app:logistic-experiments}

We first compare direction accuracy with finite gain under exact KL calibration.
We then measure the effect of estimating KL from samples and how damping changes
the required sample count. Finally, we extend the direction-recovery construction
to higher dimensions.

\subsection{Finite-step improvement with a sampled Fisher matrix}
\label{app:logistic-finite-step}
\begin{figure}[htbp]
 \centering
 \begin{tikzpicture}
\begin{groupplot}[
 paper pair,
 group style={group size=2 by 1, horizontal sep=\paperpairsep},
 xmode=log, log basis x=4, xmin=0.2, xmax=1400,
 xtick={0.25,1,4,16,64,256,1024},
 xticklabels={$1/4$,$1$,$4$,$16$,$64$,$256$,$1024$},
 xlabel={expected joint-feature count $b_{\mathrm L}=m/\DL$},
]
\nextgroupplot[title={(a) unit-direction estimation},
 ymin=-0.03,ymax=1.03,ytick={0,0.25,0.5,0.75,1},
 ylabel={probability of error $\leq 1/32$},
 legend to name=logisticapplicationlegend,
 legend columns=3,
 legend style={font=\footnotesize, draw=none, fill=none,
   cells={anchor=west},
   /tikz/every even column/.append style={column sep=0.85em}}]
 \addplot[color=black!60,mark=none,dashed]
   table[x=b,y=population_success]{\logisticapplication};
 \addlegendentry{exact natural gradient}
 \addplot[color=cbblue,mark=*]
   table[x=b,y=pseudoinverse_success]{\logisticapplication};
 \addlegendentry{sample pseudoinverse}
 \addplot[color=cborange,mark=square*]
   table[x=b,y=ridge_0_1_success]{\logisticapplication};
 \addlegendentry{damped, $\tau_\lambda=0.1$}
 \addplot[color=cbgreen,mark=star]
   table[x=b,y=ridge_1_success]{\logisticapplication};
 \addlegendentry{damped, $\tau_\lambda=1$}
 \addplot[color=cbpurple,mark=none,densely dotted]
   table[x=b,y=gradient_success]{\logisticapplication};
 \addlegendentry{gradient}
\nextgroupplot[title={(b) fixed-query improvement},
 ymin=18,ymax=104,ytick={20,40,60,80,100},
 ylabel={improvement (\% of optimum)}]
 \addplot[color=black!60,mark=none,dashed]
   table[x=b,y expr=100*\thisrow{population_gain}]{\logisticapplication};
 \addplot[color=cbblue,mark=*,error bars/.cd,y dir=both,y explicit]
   table[x=b,y expr=100*\thisrow{pseudoinverse_gain},
   y error plus expr=100*(\thisrow{pseudoinverse_gain_hi}-\thisrow{pseudoinverse_gain}),
   y error minus expr=100*(\thisrow{pseudoinverse_gain}-\thisrow{pseudoinverse_gain_lo})]{\logisticapplication};
 \addplot[color=cborange,mark=square*,error bars/.cd,y dir=both,y explicit]
   table[x=b,y expr=100*\thisrow{ridge_0_1_gain},
   y error plus expr=100*(\thisrow{ridge_0_1_gain_hi}-\thisrow{ridge_0_1_gain}),
   y error minus expr=100*(\thisrow{ridge_0_1_gain}-\thisrow{ridge_0_1_gain_lo})]{\logisticapplication};
 \addplot[color=cbgreen,mark=star,error bars/.cd,y dir=both,y explicit]
   table[x=b,y expr=100*\thisrow{ridge_1_gain},
   y error plus expr=100*(\thisrow{ridge_1_gain_hi}-\thisrow{ridge_1_gain}),
   y error minus expr=100*(\thisrow{ridge_1_gain}-\thisrow{ridge_1_gain_lo})]{\logisticapplication};
 \addplot[color=cbpurple,mark=none,densely dotted]
   table[x=b,y expr=100*\thisrow{gradient_gain}]{\logisticapplication};
\end{groupplot}
\end{tikzpicture}
\plotlegend{logisticapplicationlegend}
 \caption{Logistic query updates with \(K=1\), \(\DL=65{,}536\),
 and conditional KL budget \(0.05/\DL\).
 Left: fraction of directions with Euclidean error at most \(1/32\).
 Right: mean log-likelihood improvement as a percentage of the global
 constrained optimum, with 95\% intervals over 2,048 independent repetitions.
 Every nonzero direction uses the same exact reference-KL line search.
 Damping is \(\lambda=\tau_\lambda/(4\DL)\).}
 \label{fig:logistic-application}
\end{figure}
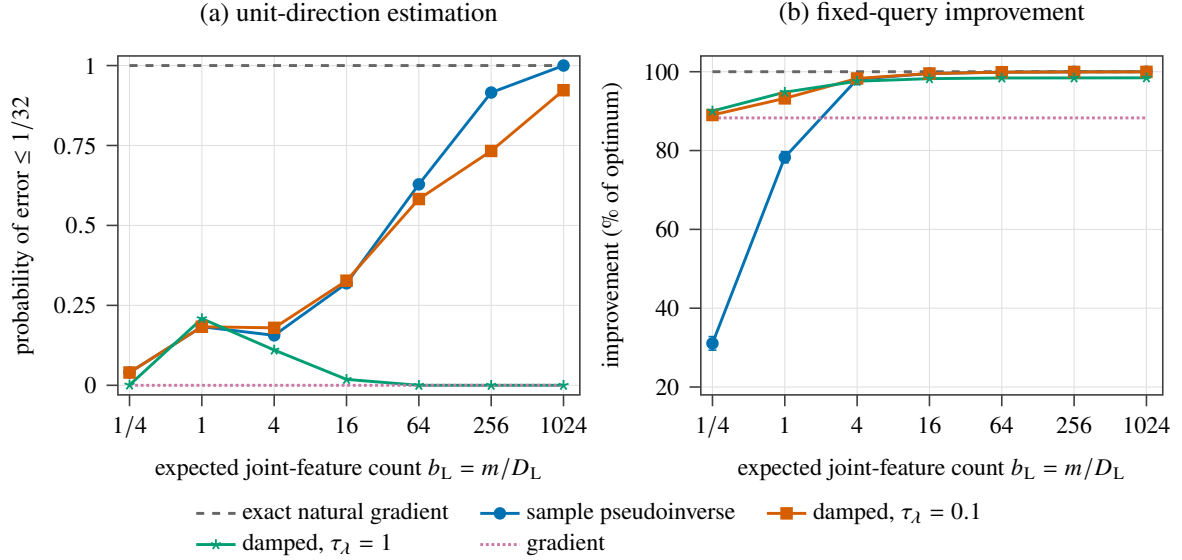

For the logistic law in \cref{eq:logistic-model}, we use
\(K\in\{1,2\}\), \(\DL\in\{256,4096,65536\}\), and
\({b_{\mathrm L}}=m/\DL\in\{1/4,1,4,16,64,256,1024\}\).
The 2,048 independent multinomial count vectors per setting are shared
across methods and KL budgets. Each vector
represents \(m\) independent draws without generating them individually.

The sampled Fisher information matrix is \(\widehat F=(4m)^{-1}\sum_iX_iX_i^\mathsf T\).
At zero this also equals the average outer product of the observed scores,
since \((Y-1/2)^2=1/4\).
We compare the population natural-gradient direction, \(\widehat F^\dagger e_1\),
\((\widehat F+\tau_\lambda I_2/(4\DL))^{-1}e_1\) for \(\tau_\lambda=0.1,1\),
and the gradient \(e_1\). Any positive diagonal preconditioner gives the same ray as the gradient for this
query, so the gradient also represents diagonal methods. Nonzero directions are normalized,
and a zero pseudoinverse direction gives a zero update. Direction success means
Euclidean error at most \(1/32\) relative to the normalized population natural-gradient direction.

Query improvement is \(J(\theta)=\log2-\log(1+e^{-\theta_1})\).
Write \(B_K(\theta)=\DL\KLlog_c(\theta)\) for the scaled KL in
\eqref{eq:logistic-main-kl}.
The scaled budget is \(\bar\eta=\DL\eta\), where \(\eta\) is the
absolute conditional-KL budget.
For each ray, we calibrate the step by bisection using the reference KL at
\(\bar\eta\in\{0.005,0.01,0.05,0.5\}\), so the comparison isolates direction quality.
Reference calibration receives the true covariate probabilities.
The benchmark is \(J_K^*(\bar\eta)=\max_{B_K(\theta)\leq\bar\eta}J(\theta)\).
To compute this benchmark, note that at fixed \(\theta_1=t>0\), strict convexity
gives a unique minimizing
\(\theta_{2,K}^\star(t)\in(-t,0)\) satisfying
\begin{equation}
 K\tanh(\theta_{2,K}^\star(t)/2)+\tanh((t+\theta_{2,K}^\star(t))/2)=0.
 \label{eq:logistic-experiment-root}
\end{equation}
The KL after minimizing over the second coordinate increases strictly to infinity, so a second bisection
solves \(B_K(t,\theta_{2,K}^\star(t))=\bar\eta\).
For \(K=1\), \(\theta_{2,1}^\star(t)=-t/2\): the population natural-gradient ray is globally optimal.
For \(K=1,\DL=65536,\bar\eta=0.05,{b_{\mathrm L}}=1\), the sampled Fisher matrix \(\widehat F\) is singular in 28.8\%
of trials. The pseudoinverse, the two damping levels and the gradient achieve mean gain
fractions of \(78.3,93.2,94.8,88.3\%\), respectively.
The corresponding fractions are \(90.3,96.0,97.4,94.9\%\) at \(K=2\).
The methods have the same ordering at all four budgets.
At \(K=1,{b_{\mathrm L}}=16\), pseudoinverse and lightly damped gains reach 99.5\%
and 99.6\%. Only 31.9\% of pseudoinverse directions meet the Euclidean tolerance.
Direction success rises to 91.6\% at \({b_{\mathrm L}}=256\).
Stronger damping retains its bias, with gains of 98.2\% and 98.4\% at these two
sample budgets. \Cref{fig:logistic-application} reports uncertainty for these estimates.

\noindent\textbf{Direction alignment in the Fisher metric.}
We evaluate direction alignment and reference-calibrated gain using the
same count vectors. Consider a linear objective with gradient
$b$, Fisher matrix $F\succ0$ and ray $u$ satisfying $b^\mathsf Tu>0$.
Under quadratic KL, the optimal gain along this ray divided by the global optimal gain is
\begin{equation}
 \frac{b^\mathsf Tu}{\sqrt{u^\mathsf TFu}\sqrt{b^\mathsf TF^{-1}b}}.
\end{equation}
This is the Fisher cosine between the estimated ray and the population
natural gradient. Zero directions contribute zero to both the local and finite-gain means.
\Cref{tab:logistic-fisher-angle} compares Fisher alignment with the Euclidean
criterion and actual nonlinear gain at \(K=1\), \(\DL=65536\) and
\(\bar\eta=0.05\).
At $m/\DL=16$, the mean cosine is 99.46\% and mean finite gain is 99.54\%,
despite the low Euclidean pass rate in \cref{tab:logistic-fisher-angle}.
Across all 840
combinations of method, law, sample budget and KL budget, the largest mean absolute
difference between the Fisher cosine and finite-gain fraction is 6.08 percentage points, at $\bar\eta=0.5$.

\begin{table}[htbp]
\caption{Post hoc Fisher alignment for the pseudoinverse direction at
$K=1$, $\DL=65536$ and $\bar\eta=0.05$, reusing the 2,048 trials per row.
All entries after the sample budget are percentages. Euclidean pass uses the $1/32$ tolerance.
Fisher cosine and finite gain are means over trials. $S_{90}$ is the
fraction attaining 90\% of global optimal finite gain under reference-KL
calibration. Zero directions remain in every denominator.}
\label{tab:logistic-fisher-angle}
\begin{center}
\begin{tabular}{rrrrr}
\toprule
$m/\DL$ & Euclidean pass & Fisher cosine & Finite gain & $S_{90}$ \\
\midrule
0.25 & 4.0 & 30.21 & 31.08 & 5.1 \\
1 & 18.3 & 77.31 & 78.30 & 41.1 \\
4 & 15.6 & 97.80 & 98.10 & 96.2 \\
16 & 31.9 & 99.46 & 99.54 & 100.0 \\
64 & 62.8 & 99.88 & 99.89 & 100.0 \\
256 & 91.6 & 99.97 & 99.97 & 100.0 \\
1024 & 100.0 & 99.99 & 99.99 & 100.0 \\
\bottomrule
\end{tabular}
\end{center}
\end{table}

\subsection{Sample-based logistic updates}
\label{app:logistic-calibration}

We now estimate the KL cost as well as the direction to measure their
combined effect on useful-step selection. We use the fixed query \((e_1,1)\) and global reference-KL gain benchmark
from \cref{app:logistic-finite-step}. The grid is
\(K\in\{1,2\}\), \(\DL\in\{4096,65536\}\),
\(\bar\eta\in\{0.01,0.05,0.5\}\), \(\eta=\bar\eta/\DL\), and
\(m/\DL\in\{1/4,1,4,16,64,256,1024\}\).
Each count setting has 2,048 exact multinomial repetitions
shared across methods and budgets.
This calibration comparison uses draws independent of those used in the direction-estimation experiment in
\cref{app:logistic-finite-step}.
Sample-based methods receive counts, absolute damping and budget values, and the query
gradient \(e_1/2\). They do not impose the construction's equality between
the masses of \(e_1\) and \(e_2\). In this family, the exact direction and
encoding size determine the cell masses, but sampled calibration uses only
the counts. This isolates the chosen rule's response to KL uncertainty.
We compare the population natural-gradient and pseudoinverse directions, damped directions
with \(\lambda\in\{0.1,1\}/(4\DL)\), and the ordinary gradient.
Nonzero directions are normalized, with zero directions giving zero updates.
We apply empirical, reference and confidence calibration to each direction.
For counts \(n_i\) from \(m\) covariate draws, let
\(\widehat p_i=n_i/m\), and let \(\Delta_4\) be the simplex on the four
covariate cells. Confidence calibration uses
\begin{equation}
 \mathcal P_m=\{{\boldsymbol w}\in\Delta_4:
 m\binaryKL(\widehat p_i\Vert w_i)\leq\log(8/\delta)
 \text{ for all }i\},\qquad\delta=0.05.
 \label{eq:coupled-cell-confidence}
\end{equation}
Two-sided Bernoulli Chernoff bounds give failure probability at most
\(\delta/4\) per cell \citep{Hoeffding1963}. A union bound therefore gives \(\Prob({\boldsymbol p}\in\mathcal P_m)\geq1-\delta\).
Boundary values of \(\binaryKL\) are defined by limits, so zero and full
counts are allowed.
The cost \(\KLlog_w(\theta)=\sum_xw_x\log\cosh(\theta^\mathsf Tx/2)\)
is linear in the covariate masses \(w\), so greedy allocation to
cells ordered by \(|v^\mathsf Tx|\) maximizes it.
The cell ordering is unchanged along each ray for \(t\geq0\).
For each nonzero direction \(v\), confidence calibration returns the largest
\(t\in[0,64]\) satisfying
\(\max_{w\in\mathcal P_m}\KLlog_w(tv)\leq\eta\).
Empirical and reference calibration replace this worst-case KL by the cost
under the observed cell frequencies and the true cell masses, respectively.
Coverage of the true cell masses guarantees KL feasibility even for directions
selected from the same counts, so sample splitting is unnecessary.

All searches use cap \(t=64\) and 50 bisections.
Relative feasibility tolerance is \(10^{-7}\).
The primary gain target is 90\%, with sensitivity analyses at 75\% and 95\%.

For confidence calibration in \cref{tab:logistic-oracle-controls}, joint
success rises from 1.1\% [0.7,1.6] to 98.9\% [98.4,99.3]
with population directions as \(m/\DL\) rises from 16 to 256.
\Cref{tab:logistic-secondary-settings} gives the corresponding
pseudoinverse results.
For empirical pseudoinverse calibration at ratios \(16,64,256,1024\),
feasibility is \(43.2,45.8,47.8,49.2\%\), while the 95th-percentile
KL ratios decrease to \(1.411,1.173,1.081,1.039\), respectively.
Smaller violations therefore need not imply a high rate of strict feasibility.
At \(K=1,\eta=0.05/\DL\), the gradient ray reaches at most 88.3\%
of global optimal gain, making the 90\% target unattainable on that ray.

\begin{table}[p]
\centering
\caption{Direction and calibration controls at $K=1$, $\DL=65{,}536$ and $\eta=0.05/\DL$. The population method uses the exact direction, while the pseudoinverse method estimates it from the counts. Reference calibration uses the true law. Confidence calibration uses the simultaneous confidence set for cell masses from those same counts. All metrics are percentages over 2,048 paired repetitions per row. \emph{Feas.} is true KL feasibility. \emph{Gain} is mean improvement divided by the global constrained query optimum, with infeasible gains set to zero. $S_{90}$ requires both feasibility and at least 90\% of that optimum. Brackets are pointwise 95\% Wilson intervals for $S_{90}$.}
\label{tab:logistic-oracle-controls}
\begin{tabularx}{\linewidth}{rllrrR}
\toprule
$m/\DL$ & Direction & Calibration & Feas. & Gain & $S_{90}$ [95\% CI] \\
\midrule
16 & Population & Reference & 100.0 & 100.0 & 100.0 [99.8, 100.0] \\
16 & Population & Confidence & 99.9 & 78.0 & 1.1 [0.7, 1.6] \\
16 & Pseudoinverse & Reference & 100.0 & 99.5 & 100.0 [99.8, 100.0] \\
16 & Pseudoinverse & Confidence & 99.9 & 76.9 & 0.5 [0.3, 1.0] \\
\midrule
64 & Population & Reference & 100.0 & 100.0 & 100.0 [99.8, 100.0] \\
64 & Population & Confidence & 100.0 & 88.0 & 23.6 [21.8, 25.5] \\
64 & Pseudoinverse & Reference & 100.0 & 99.9 & 100.0 [99.8, 100.0] \\
64 & Pseudoinverse & Confidence & 100.0 & 87.6 & 20.8 [19.1, 22.7] \\
\midrule
256 & Population & Reference & 100.0 & 100.0 & 100.0 [99.8, 100.0] \\
256 & Population & Confidence & 100.0 & 93.6 & 98.9 [98.4, 99.3] \\
256 & Pseudoinverse & Reference & 100.0 & 100.0 & 100.0 [99.8, 100.0] \\
256 & Pseudoinverse & Confidence & 100.0 & 93.5 & 98.7 [98.1, 99.1] \\
\bottomrule
\end{tabularx}
\end{table}

\begin{table}[p]
\caption{Pseudoinverse directions with confidence calibration at
$\eta=0.05/\DL$, over 2,048 repetitions per row.
$A_{\rm dir}$ is the percentage with direction error at most $1/32$.
Feasibility is evaluated under the true law. Normalized feasible gain and joint success
$S_{90}$ use the global query optimum as in \cref{tab:logistic-oracle-controls}.
Brackets are pointwise 95\% Wilson intervals for $S_{90}$.}
\label{tab:logistic-secondary-settings}
\begin{center}
\begin{tabularx}{\linewidth}{rrrrrrR}
\toprule
$K$ & $\DL$ & $m/\DL$ & $A_{\rm dir}$ & Feas. & Gain & $S_{90}$ [95\% CI] \\
\midrule
1 & 4,096 & 16 & 33.7 & 100.0 & 77.0 & 0.8 [0.5, 1.3] \\
1 & 4,096 & 64 & 61.2 & 100.0 & 87.5 & 19.2 [17.6, 21.0] \\
1 & 4,096 & 256 & 91.8 & 100.0 & 93.5 & 98.5 [97.9, 98.9] \\
\midrule
1 & 65,536 & 16 & 33.2 & 99.9 & 76.9 & 0.5 [0.3, 1.0] \\
1 & 65,536 & 64 & 60.8 & 100.0 & 87.6 & 20.8 [19.1, 22.7] \\
1 & 65,536 & 256 & 91.6 & 100.0 & 93.5 & 98.7 [98.1, 99.1] \\
\midrule
2 & 4,096 & 16 & 39.7 & 100.0 & 79.9 & 1.1 [0.7, 1.7] \\
2 & 4,096 & 64 & 68.0 & 100.0 & 88.9 & 32.4 [30.4, 34.5] \\
2 & 4,096 & 256 & 96.2 & 100.0 & 94.2 & 100.0 [99.8, 100.0] \\
\midrule
2 & 65,536 & 16 & 39.3 & 100.0 & 80.0 & 1.5 [1.0, 2.1] \\
2 & 65,536 & 64 & 68.2 & 100.0 & 89.0 & 32.1 [30.1, 34.1] \\
2 & 65,536 & 256 & 95.1 & 100.0 & 94.2 & 99.9 [99.6, 100.0] \\
\bottomrule
\end{tabularx}
\end{center}
\end{table}

\subsection{Sample dependence on the damping scale}
\label{app:logistic-damping-experiment}

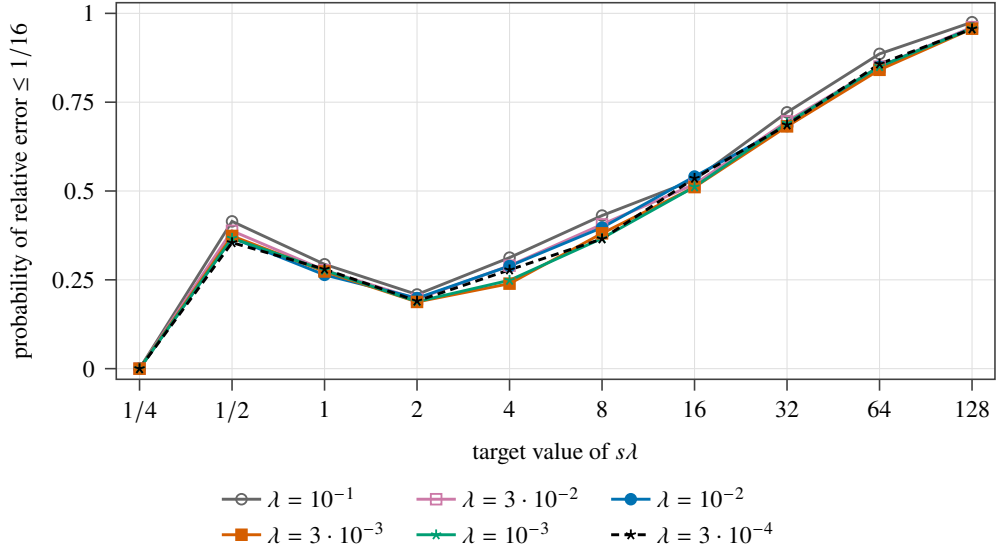
\begin{figure}[htbp]
 \centering
 \begin{tikzpicture}
\begin{axis}[
 paper wide,
 xmode=log, log basis x=2, xmin=0.21, xmax=152,
 xtick={0.25,0.5,1,2,4,8,16,32,64,128},
 xticklabels={$1/4$,$1/2$,$1$,$2$,$4$,$8$,$16$,$32$,$64$,$128$},
 xlabel={target value of $s\lambda$},
 ylabel={probability of relative error $\leq 1/16$},
 ymin=-0.03, ymax=1.03, ytick={0,0.25,0.5,0.75,1},
 legend to name=logisticdampinglegend,
 legend columns=3,
]
 \addplot[color=black!60, mark=o]
   table[x=b,y=lambda0]{\logisticdamping};
 \addlegendentry{$\lambda=10^{-1}$}
 \addplot[color=cbpurple, mark=square]
   table[x=b,y=lambda1]{\logisticdamping};
 \addlegendentry{$\lambda=3\cdot10^{-2}$}
 \addplot[color=cbblue, mark=*]
   table[x=b,y=lambda2]{\logisticdamping};
 \addlegendentry{$\lambda=10^{-2}$}
 \addplot[color=cborange, mark=square*]
   table[x=b,y=lambda3]{\logisticdamping};
 \addlegendentry{$\lambda=3\cdot10^{-3}$}
 \addplot[color=cbgreen, mark=star]
   table[x=b,y=lambda4]{\logisticdamping};
 \addlegendentry{$\lambda=10^{-3}$}
 \addplot[color=black, densely dashed, mark=star]
   table[x=b,y=lambda5]{\logisticdamping};
 \addlegendentry{$\lambda=3\cdot10^{-4}$}
\end{axis}
\end{tikzpicture}
\plotlegend{logisticdampinglegend}
 \caption{Estimates of the reciprocal of damped scalar Fisher information from paired observations at
 \(p=\lambda\), with 4,096 repetitions per point and \(2s\) observations. At small sample sizes, the binomial counts meeting the error threshold change discretely,
 so success need not increase at every grid point.}
 \label{fig:logistic-damping}
\end{figure}

We set \(p=\lambda\) so that event information and damping have comparable
scales, as in the lower-bound construction. For \(s\) independent pairs of Bernoulli observations,
let \({N_{\mathrm{disc}}}\) count pairs whose outcomes differ. Each pair has discordance
probability \(2p(1-p)\). Thus
\begin{equation}
 {N_{\mathrm{disc}}}\sim\operatorname{Binomial}(s,2p(1-p)),\qquad
 \widehat F={N_{\mathrm{disc}}}/(2s),\qquad F=p(1-p).
 \label{eq:scalar-damping-binomial-law}
\end{equation}

We use \(\lambda\in\{0.1,0.03,0.01,0.003,0.001,0.0003\}\) and
\(s\lambda\in\{1/4,1/2,1,2,4,8,16,32,64,128\}\), rounding \(s\)
to the nearest integer. There are 4,096 exact binomial repetitions per
setting, each representing \(2s\) observations.
Success requires relative error in the estimated reciprocal of the damped scalar Fisher information
\(\lvert(F+\lambda)/(\widehat F+\lambda)-1\rvert\leq1/16\).
At \(s\lambda=128\), success spans 95.6--97.5\% over the six scales,
using 2,560--853,334 observations, compared with 84.1--88.6\% at 64.

\subsection{Direction recovery in higher dimensions}
\label{app:higher-dimension}

The two-dimensional obstruction persists in higher dimensions, with one active query in each of several blocks.
Let \(d=2{n_{\mathrm{blk}}}\) for an integer \({n_{\mathrm{blk}}}\geq1\). A draw selects one block uniformly and places the
two-feature covariate there, with zeros elsewhere. All blocks share the
same covariate law, and every block has a positive-label query:
\begin{equation}
 F_d(0)=\frac{I_{n_{\mathrm{blk}}}\otimes
 \left(\begin{smallmatrix}K+1&1\\1&K+1\end{smallmatrix}\right)}{4{n_{\mathrm{blk}}}\DL},
 \qquad
 \ell_d(\theta)=\frac1{n_{\mathrm{blk}}}\sum_j\log\sigm(\theta_1^{(j)}).
 \label{eq:higher-dimension-model}
\end{equation}

The Fisher information matrix has full rank and condition number \(1+2/K\leq3\).
The constraint is \({n_{\mathrm{blk}}}^{-1}\sum_jB_K(\theta^{(j)})\leq\bar\eta\),
equivalent to conditional KL at most \(\bar\eta/\DL\).
For \(\bar\theta={n_{\mathrm{blk}}}^{-1}\sum_j\theta^{(j)}\), convexity of \(B_K\)
and concavity of \(\log\sigm\) give
\begin{equation}
 B_K(\bar\theta)\leq {n_{\mathrm{blk}}}^{-1}\sum_jB_K(\theta^{(j)}),\qquad
 \log\sigm(\bar\theta_1)\geq {n_{\mathrm{blk}}}^{-1}\sum_j\log\sigm(\theta_1^{(j)}).
 \label{eq:block-averaging-jensen}
\end{equation}
Repeating the average block therefore preserves feasibility and does not
decrease the objective. The global gain optimum is exactly \(J_K^*(\bar\eta)\).

\noindent\textbf{Complexity of the block construction.}
The shared block law also preserves the direction-recovery obstruction. For the
computational construction, let \(\DL=2^{n+5}\). When \({n_{\mathrm{blk}}}\) is a
power of two, selecting a block requires \(\log_2{n_{\mathrm{blk}}}\) additional
fair bits, so exact sampling stays polynomial-time. A known projection then
recovers the two-dimensional SAT test from any accurate high-dimensional output.
Writing \(u_K=(K+1,-1)^\mathsf T/\sqrt{(K+1)^2+1}\), the unit
direction is \(\nu_d=\mathbf1_{n_{\mathrm{blk}}}\otimes u_K/\sqrt {n_{\mathrm{blk}}}\).
The unit vector \(w=\mathbf1_{n_{\mathrm{blk}}}\otimes e_2/\sqrt {n_{\mathrm{blk}}}\) gives
\(w^\mathsf T\nu_d=(u_K)_2\), so error \(1/32\) preserves the SAT threshold.
For rational output, with \(S=\sum_j\widehat\nu_{2j}\), the test
\(w^\mathsf T\widehat\nu<-3/8\) is exactly
\(S<0,\ 64S^2>9{n_{\mathrm{blk}}}\). The test identifies the unsatisfiable case on every valid output, which gives
NP-hardness for fixed or polynomially bounded power-of-two block counts.

For \(K=1,2\), the \(2{n_{\mathrm{blk}}}\) nonzero coordinate-atom masses
each change by \(1/({n_{\mathrm{blk}}}\DL)\) and the zero mass by \(2/\DL\).
Thus the total variation distance for one draw is \(2/\DL\), also with the fair label.
\Cref{lem:padding} and the same two-law test give \(m\geq\DL/6\).

\noindent\textbf{Sampling experiment.}
We compare separate estimation in each block with pooling observations
under the known common block law.
The grid fixes \(\DL=65536,\bar\eta=0.05\), uses \(K=1,2\),
\(d\in\{2,8,32,128,256\}\), and
\(m/({n_{\mathrm{blk}}}\DL)\in\{1/4,1,4,16,64,256,1024\}\), adding fixed-total
\(m=16\DL,256\DL\) comparisons. Each setting has 1,024 paired multinomial repetitions.
We compare the population natural-gradient direction, separate and pooled
pseudoinverse directions, and the gradient direction. Two further methods use
damped directions with \(\lambda=\tau_\lambda/(4{n_{\mathrm{blk}}}\DL)\)
for \(\tau_\lambda=0.1,1\).
Block magnitudes are retained before global normalization, and zero directions
give zero updates. Calibration always uses the reference law.

\begin{table}[htbp]
 \caption{(a) Separate-block pseudoinverse at \(m/({n_{\mathrm{blk}}}\DL)=16\).
 Direction successes are out of 1,024 and gains are percentages of the
 global optimum with 95\% normal half-widths.
 (b) Gain at fixed total \(m=16\DL,K=1\).
 All methods use reference-KL calibration.}
 \label{tab:higher-dimension}
\begin{center}
 \begin{tabular*}{\linewidth}{@{\extracolsep{\fill}}rrrrr}
\toprule
\multicolumn{5}{c}{\textbf{(a) Matched expected count per block:} \(m/({n_{\mathrm{blk}}}\DL)=16\)} \\
\midrule
& \multicolumn{2}{c}{\(K=1\)} & \multicolumn{2}{c}{\(K=2\)} \\
\cmidrule(lr){2-3}\cmidrule(lr){4-5}
\(d\) & Successes & Gain (\%) & Successes & Gain (\%) \\
\midrule
2 & 361 & \(99.53\pm0.04\) & 410 & \(99.75\pm0.02\) \\
8 & 0 & \(98.15\pm0.08\) & 1 & \(98.96\pm0.04\) \\
32 & 0 & \(97.71\pm0.05\) & 0 & \(98.68\pm0.03\) \\
128 & 0 & \(97.62\pm0.03\) & 0 & \(98.64\pm0.01\) \\
256 & 0 & \(97.60\pm0.02\) & 0 & \(98.63\pm0.01\) \\
\midrule
\end{tabular*}
\par
\begin{tabular*}{\linewidth}{@{\extracolsep{\fill}}rrrrrr}
\multicolumn{6}{c}{\textbf{(b) Fixed total samples:} \(m=16\DL\), \(K=1\)} \\
\midrule
\(d\) & Separate & \(\tau_\lambda=0.1\) & \(\tau_\lambda=1\) & Gradient & Pooled \\
\midrule
2 & 99.53 & 99.54 & 98.19 & 88.28 & 99.53 \\
8 & 91.18 & 92.47 & 95.29 & 88.28 & 99.53 \\
32 & 73.01 & 50.29 & 83.26 & 88.28 & 99.59 \\
128 & 43.11 & 73.64 & 79.15 & 88.28 & 99.54 \\
256 & 29.59 & 80.64 & 82.23 & 88.28 & 99.55 \\
\bottomrule
\end{tabular*}

\end{center}
\end{table}

\Cref{tab:higher-dimension} compares increasing dimension at matched
per-block sample sizes and at a fixed total budget.
At \(d=256,m/({n_{\mathrm{blk}}}\DL)=16\), both values of \(K\) have zero direction successes
(95\% Wilson upper endpoint 0.37\%) despite gains of 97.60\% and 98.63\%.
At fixed \(m=16\DL,K=1\), pooling instead gives 99.55\% gain versus
29.59\% for separate estimation. The gradient gains remain 88.28\%
and 94.87\% for \(K=1,2\), independent of dimension.

\section{Calibration in learned classifier heads}
\label{app:learned-updates}

These experiments test whether stopping at a sampled KL boundary is reliable in
learned models without constructed rarity. \Cref{app:learned-setup} fixes the
models and queries, \cref{app:learned-access} the directions, calibration
rules and outcome measures, \cref{app:learned-original,app:learned-followup}
report the full-budget and margin experiments, and
\cref{app:learned-verification} describes sampling costs and uncertainty.

\subsection{Models and fixed interventions}
\label{app:learned-setup}

To study calibration with learned representations, each intervention
updates the affine softmax head for one query while keeping the representation
fixed. We restore the original head before
each query and repetition.
Class-stratified round-robin sampling selects 32 initially misclassified
queries per model before evaluation.
\Cref{tab:learned-models} lists the reference-set sizes, accuracies on the splits used to select queries,
and numbers of free parameters in the heads.

\begin{table}[htbp]
\caption{Fixed models. Features include an intercept, and the last class's
logits are fixed at zero to remove redundant head parameters.}
\label{tab:learned-models}
\begin{center}
\begin{tabularx}{\linewidth}{lXX}
\toprule
 & CIFAR-10 / ResNet-18 & SST-2 / DistilBERT \\
\midrule
Reference inputs & 50,000 training inputs & 67,349 training inputs \\
Query split (accuracy) & Test (95.34\%) & Validation (91.055\%) \\
Selected errors & 32 of 466 & 32 of 78 \\
Feature dimension / free head parameters & 513 / 4,617 & 769 / 769 \\
\bottomrule
\end{tabularx}
\end{center}
\end{table}

ResNet-18 \citep{HeEtAl2016} uses the final checkpoint of
a prescribed 200-epoch CIFAR-10 run \citep{Krizhevsky2009}, without
test-based checkpoint selection.
Its stem is \(3\times3\), stride one, without max pooling.
Training uses SGD (learning rate \(0.1\), momentum \(0.9\),
weight decay \(5\times10^{-4}\)) and a batch size of 128.
We use cosine learning-rate decay, random crops with four-pixel padding and horizontal flips. Pixels divided by 255 are normalized
by means \((0.4914,0.4822,0.4465)\) and standard deviations
\((0.2470,0.2435,0.2616)\). Features are the 512 final average-pooled values.

DistilBERT \citep{SanhEtAl2019} uses a fixed GLUE SST-2 checkpoint
\citep{WangEtAl2019,SocherEtAl2013}, sequence length 128 and the
768-dimensional representation after the pre-classifier and ReLU.
Feature extraction uses evaluation mode without augmentation.
The reference law \(Q\) is uniform over training inputs.
Accuracies on the reference populations are 100\% and
98.855\%.

\subsection{Directions, step selection and outcomes}
\label{app:learned-access}

We compare reference and sampled damped natural-gradient directions with the
gradient direction. This isolates how information about the background input law affects the
update. All three use the exact query gradient.
For free head parameters \(\theta\), put
\(b=\nabla_\theta\log\pi_\theta(y_*\mid x_*)\) and
\(s_\theta(y,x)=\nabla_\theta\log\pi_\theta(y\mid x)\).

The model Fisher information matrix and
reference direction are
\begin{align}
 F_Q&=\E_{x\sim Q}\sum_y\pi_\theta(y\mid x)
                s_\theta(y,x)s_\theta(y,x)^{\mathsf T},
 \label{eq:learned-model-fisher}\\
 v_Q&=(F_Q+\lambda I)^{-1}b/
                 \|(F_Q+\lambda I)^{-1}b\|_2.
 \label{eq:learned-reference-direction}
\end{align}
For the estimated direction, we replace the input expectation by a sample
average. The gradient direction is \(b/\|b\|_2\).
Labels of the reference inputs are used only for the prediction diagnostics,
not for the direction or KL. These experiments use damped directions
without a conditioning restriction. The theoretical construction instead
uses an undamped natural gradient and a Fisher matrix with condition number at most $3$.
For a unit direction \(v\), define
\begin{align}
 G_v(t)&=\log\pi_{\theta+tv}(y_*\mid x_*)-\log\pi_\theta(y_*\mid x_*),\\
 {\mathcal K}_v(t)&=\E_Q\KL(\pi_\theta(\cdot\mid x)\Vert
                         \pi_{\theta+tv}(\cdot\mid x)).
 \label{eq:learned-ray-kl}
\end{align}
Calibration maximizes the known query gain under sampled or reference KL.
Reference directions and calibration require extra population information,
whereas empirical calibration uses sampled KL.
For each fixed direction, the reference-budget benchmark and success indicator are
\begin{equation}
 G_v^*(\eta)=\sup_{t\geq0:{\mathcal K}_v(t)\leq\eta}G_v(t),\qquad
 S_\alpha(t)=\one\{{\mathcal K}_v(t)\leq\eta,\ G_v(t)\geq\alpha G_v^*(\eta)\}.
 \label{eq:learned-joint-success}
\end{equation}
The supremum is finite because the initial query probability is positive.
If a ray leaves every reference conditional distribution unchanged, its optimal query gain
may be approached only as \(t\to\infty\). If at least one reference input
has nonconstant class-logit increments, positive baseline softmax probabilities
give \({\mathcal K}_v(t)\to\infty\). The feasible interval is then compact and the
supremum is attained.
We use \(\alpha=0.75\), with \(0.90\) as a secondary threshold.
The benchmark always uses the full budget, including for the \(0.9\eta\) rule.
We also report ratios of population KL to the budget, violations above \(1.1\eta\), corrected queries
and changes in predictions on the reference population.

\noindent\textbf{Why stopping at an estimated boundary violates the budget about half the time.}
Strict feasibility is sensitive to selection at an estimated boundary. Fix a ray
independently of its calibration data. Suppose \({\mathcal K}_v\) and its
empirical estimate \(\widehat {\mathcal K}_{v,m}\) are continuous and strictly increasing
on a fixed neighborhood of a positive root \(t^*\) of \({\mathcal K}_v(t)=\eta\).
Calibration selects a root \(\widehat t_m\) of \(\widehat {\mathcal K}_{v,m}(t)=\eta\),
and this root converges in probability to the population boundary.
On this neighborhood, the boundary comparison gives
\begin{equation}
 {\mathcal K}_v(\widehat t_m)>\eta
 \quad\Longleftrightarrow\quad
 \widehat {\mathcal K}_{v,m}(t^*)<\eta.
\end{equation}
If
\(\sqrt m[\widehat {\mathcal K}_{v,m}(t^*)-\eta]\) has a centered, nondegenerate normal
limit, the strict violation probability tends to \(1/2\), even though
the violation magnitude vanishes in probability. This concerns exact
feasibility before numerical tolerances. Inactive constraints, caps or rays
estimated from the calibration data need not satisfy these conditions.
The margin rule constrains the empirical finite-step KL, so it guards against
sampling error rather than against the error of a quadratic KL approximation.

\subsection{Full-budget experiment}
\label{app:learned-original}

We cross three directions and two calibration rules to distinguish
direction-estimation effects from calibration effects. The grid is fixed before evaluation.
Sample sizes are \(m\in\{256,1024,4096\}\), with budgets
\(\eta\in\{0.01,0.001\}\) and default damping \(\lambda=0.001\).
The values \(\lambda=0.0001,0.01\) are added at
\(m=1024,\eta=0.01\). Fisher matrix estimation and calibration use independent batches sampled with
replacement. Corresponding methods use the same sampled inputs, and smaller
batches use the first draws of the larger batches.
Each condition has 20 repetitions for each of 32 queries.
The primary comparison subtracts the raw gain under empirical calibration from
that under reference calibration, using the reference direction at \(m=1024,\eta=0.01,\lambda=0.001\).

\begin{table}[htbp]
\caption{Paired raw-gain contrasts at \(m=1024,\lambda=0.001\),
in \(10^{-3}\) nats, with 95\% query-bootstrap intervals.
Direction contrasts use reference calibration.}
\label{tab:learned-selected-effects}
\begin{center}
\begin{tabularx}{\linewidth}{>{\raggedright\arraybackslash}XRR}
\toprule
Contrast (budget) & CIFAR-10 & SST-2 \\
\midrule
Reference minus empirical calibration (\(0.01\)) &
\(-1.59\) [\(-5.09,1.75\)] & \(-0.112\) [\(-1.214,0.845\)] \\
Reference direction minus estimated direction (\(0.01\)) &
2.32 [\(-0.36,5.13\)] & 0.13 [\(-5.69,4.90\)] \\
Estimated direction minus gradient direction (\(0.01\)) &
156.97 [71.08,249.32] & 1717.06 [1401.41,2021.48] \\
Reference direction minus estimated direction (\(0.001\)) &
4.76 [1.10,8.64] & 76.58 [52.35,101.32] \\
\bottomrule
\end{tabularx}
\end{center}
\end{table}

Raw gains include infeasible steps and can exceed the reference-budget
optimum. The primary calibration intervals include zero.
At the primary setting, empirical calibration along the reference direction produces KL above
\(1.1\eta\) in 27.34\% and 37.19\% of interventions.
The direction comparisons in \cref{tab:learned-selected-effects}
depend on budget and have unequal sampling costs.
Joint success is a post hoc measure in this experiment.
At the primary setting, every empirical update reaches 90\% of its own ray benchmark.
Joint success along the reference direction is 47.50\% [42.97,51.72] on CIFAR and
45.63\% [41.88,49.38] on SST. These failures are KL violations.

\subsection{Margin experiment and outcomes}
\label{app:learned-followup}

We test whether reserving 10\% of the empirical KL budget reduces boundary
crossings while retaining useful gain. We fix the grid before evaluation and reuse the
checkpoints, queries, reference law, extracted features and reference directions
at \(\lambda=0.001\). For each query, we run
20 fresh repetitions at \(m\in\{256,1024,4096,16384\}\), using the same sampled
inputs to compare calibration rules. The four rules combine empirical or
reference KL with the budget \(\eta\) or \(0.9\eta\).
We apply all four rules at each KL budget.
The following lemma bounds the gain lost by reserving budget under the
exact law. This gives a benchmark for the empirical margin comparison.

\begin{lemma}[Exact-budget margin for affine-head log-likelihood]
\label{lem:learned-margin-control}
Let \(\pi_{\theta+tv}(\cdot\mid x)\) be a softmax classifier whose logits
are affine in \(t\geq0\), with fixed parameters \(\theta\) and unit direction
\(v\). Fix a finite reference law \(Q\) and a labeled query \((x_*,y_*)\).
Define the query log-likelihood gain, population KL cost and ray benchmark by
\begin{equation}
 \begin{aligned}
 G_v(t)&=\log\pi_{\theta+tv}(y_*\mid x_*)-\log\pi_\theta(y_*\mid x_*),\\
 \mathcal K_v(t)&=\E_Q\KL(\pi_\theta(\cdot\mid X)\Vert\pi_{\theta+tv}(\cdot\mid X)),\\
 G_v^*(\eta)&=\sup_{t\geq0:\,\mathcal K_v(t)\leq\eta}G_v(t).
 \end{aligned}
\end{equation}
Then \(G_v^*({\chi_{\mathrm{bud}}}\eta)\geq {\chi_{\mathrm{bud}}}G_v^*(\eta)\)
for every \(\eta>0\) and \(0<{\chi_{\mathrm{bud}}}<1\), whether or not either
supremum is attained.
\end{lemma}
\begin{proof}[Proof of \cref{lem:learned-margin-control}]
The log-partition function is convex, so \({\mathcal K}_v\) is convex and \(G_v\)
concave, with \({\mathcal K}_v(0)=G_v(0)=0\).

For every \(t\geq0\) with
\({\mathcal K}_v(t)\leq\eta\),
\begin{equation}
 \begin{aligned}
 {\mathcal K}_v({\chi_{\mathrm{bud}}}t)&\leq {\chi_{\mathrm{bud}}}{\mathcal K}_v(t)\leq {\chi_{\mathrm{bud}}}\eta,\\
 G_v({\chi_{\mathrm{bud}}}t)&\geq {\chi_{\mathrm{bud}}}G_v(t).
 \end{aligned}
 \label{eq:learned-margin-proof}
\end{equation}
Thus \(G_v^*({\chi_{\mathrm{bud}}}\eta)\geq {\chi_{\mathrm{bud}}}G_v(t)\) for every such \(t\).
Taking the supremum over these candidates proves the claim without
assuming that either optimum is attained.
\end{proof}
\begin{table}[htb]
\begingroup
\captionof{table}{Results for empirical calibration in \cref{fig:learned-updates-recovery} at
$m=16{,}384$ and $\eta=0.01$. Joint success requires KL feasibility and
at least 75\% of the full-budget ray-optimal gain. Brackets give pointwise
95\% query-bootstrap intervals. Margin gain averages each intervention's raw gain as a percentage of the
full-budget ray optimum. This mean includes infeasible updates.}
\label{tab:learned-calibration-main}
\begin{center}
\begin{tabular*}{\linewidth}{@{\extracolsep{\fill}}lrrr}
\toprule
 & \multicolumn{2}{c}{Joint success (\%) [95\% CI]} & \\
\cmidrule(lr){2-3}
Dataset & Full budget $\eta$ & Margin $0.9\eta$ & Margin gain (\%) \\
\midrule
CIFAR-10 & 48.59 $[44.69,52.50]$ & 99.53 $[98.91,100.00]$ & 99.63 \\
SST-2 & 50.94 $[47.34,54.69]$ & 93.75 $[91.09,96.25]$ & 99.79 \\
\bottomrule
\end{tabular*}
\end{center}
\endgroup

\end{table}

Both reference rules attain 100\% joint success at both gain thresholds, and
the smallest reference-margin gain fraction is 96.23\%.
For empirical calibration at $0.9\eta$, with $\eta=0.01$ and
$m=16{,}384$, every update meets the 90\% gain target.
The remaining joint-success failures are KL violations: 3/640 on CIFAR
and 40/640 on SST. Of these, zero CIFAR violations and four SST violations
exceed \(1.1\eta\).
At this setting, the mean decline in accuracy on the training reference set stays below
0.15 percentage points under both empirical rules. Compared with full-budget reference calibration, reference-KL margins
correct two fewer CIFAR queries and one fewer SST query among 32 targets.
Joint success under the empirical margin rule increases at every tested \(m\)
for both budgets and models.

At \(\eta=0.001\), empirical margin success reaches 98.13\% on CIFAR and 91.56\% on SST at the largest sample size.
The SST interval is [88.13,94.84]. At the 75\% gain threshold, all 16
paired differences in joint success between margin and full-budget calibration
have positive lower interval endpoints. Coverage is pointwise for each comparison.
The 90\% target lowers success at some smaller sample sizes.

\subsection{Sampling cost and uncertainty}
\label{app:learned-verification}

Intervals use query means over 20 repetitions and 5,000 bootstrap resamples
of the 32 queries. They quantify variation across the selected queries, conditional on the
checkpoints and query selection.
An empirical update along the estimated direction uses independent batches
of \(m\) observations for Fisher matrix estimation and calibration.
Gradient and supplied-ray rules use only the calibration batch.

\end{document}